\documentclass{article}

\usepackage[utf8]{inputenc}
\usepackage{graphicx}
\usepackage{float}
\usepackage{amssymb}
\usepackage[margin=1in]{geometry}
\usepackage{ifthen}
\usepackage[table]{xcolor}
\usepackage{minitoc}
\usepackage{array}
\usepackage{makecell}
\usepackage{pdfpages}
\usepackage{mathtools}
\usepackage{tikz}
\usepackage{multirow}
\usepackage{booktabs}
\usepackage{natbib}
\usepackage{arydshln} 
\usepackage{xcolor}
\usepackage{enumerate}
\usepackage{enumitem}
\definecolor{DarkGreen}{rgb}{0.1,0.5,0.1} 
\definecolor{DarkRed}{rgb}{0.5,0.1,0.1}
\definecolor{DarkBlue}{rgb}{0.1,0.1,0.5}
\definecolor{DarkYellow}{rgb}{.79,.79,0}
\definecolor{unitednationsblue}{rgb}{0.36, 0.57, 0.9}
\definecolor{blue_ppt}{rgb}{0,0.6,0.93}
\definecolor{darkblue_ppt}{rgb}{0.05,0.4,0.8}

\definecolor{orange_ppt}{rgb}{0.82,0.5,0}

\usepackage{color} 
\definecolor{yc}{RGB}{200,0,0} 
\definecolor{yxc}{RGB}{255,0,0}
\definecolor{dacong}{RGB}{10,103,68}
\definecolor{ytw}{RGB}{255,192,203}

\newcommand{\alg}{FedNAC}
\newcommand{\critic}{Critic}
\newcommand{\sampler}{Q-Sampler}
\newcommand{\staterror}{\varepsilon_{\text{stat}}}
\newcommand{\approxerror}{\varepsilon_{\text{approx}}}
\newcommand{\meanstaterror}{\bar \varepsilon_{\text{stat}}}
\newcommand{\meanapproxerror}{\bar \varepsilon_{\text{approx}}}

\usepackage[unicode=true,
 bookmarks=false,
 breaklinks=false,pdfborder={0 0 1},colorlinks=false]
 {hyperref}
\hypersetup{
 colorlinks,citecolor=blue,filecolor=blue,linkcolor=blue,urlcolor=blue}
 
\usepackage{cleveref}
\usepackage{tablefootnote} 

\newcolumntype{?}{!{\vrule width 1pt}}
\usepackage{amsmath}
\usepackage{amsthm}
\usepackage{algorithm,algorithmic}

\usepackage{graphicx} 
\usepackage{subfigure}
\usepackage{bbm}

\usepackage{extarrows}
\usepackage{bbm}
\theoremstyle{plain}
\newtheorem{lm}{Lemma} 
\newtheorem{definition}{Definition}
\newtheorem{thm}{Theorem}

\newtheorem{asmp}{Assumption}

\newtheorem{crl}{Corollary}
\newtheorem{rmk}{Remark}

\allowdisplaybreaks

\usepackage{amsmath,amsfonts,bm}
\usepackage{algorithm,algorithmic}

\def\1{\bm{1}}

\def\vzero{{\bm{0}}}
\def\vone{{\bm{1}}}
\def\vxi{{\bm{\xi}}}

\def\vpi{{\bm{\pi}}}

\def\va{{\bm{a}}}
\def\vb{{\bm{b}}}
\def\vc{{\bm{c}}}
\def\vd{{\bm{d}}}
\def\ve{{\bm{e}}}

\def\vh{{\bm{h}}}

\def\vp{{\bm{p}}}
\def\vq{{\bm{q}}}

\def\vs{{\bm{s}}}

\def\vv{{\bm{v}}}
\def\vw{{\bm{w}}}
\def\vx{{\bm{x}}}

\def\vz{{\bm{z}}}

\def\mI{{\bm{I}}}
\def\mA{{\bm{A}}}
\def\mB{{\bm{B}}}
\def\mC{{\bm{C}}}

\def\mM{{\bm{M}}}

\def\mA{{\bm{A}}}
\def\mB{{\bm{B}}}
\def\mC{{\bm{C}}}

\def\mI{{\bm{I}}}

\def\mM{{\bm{M}}}

\def\mP{{\bm{P}}}
\def\mQ{{\bm{Q}}}

\def\mT{{\bm{T}}}

\def\mW{{\bm{W}}}

\DeclareMathAlphabet{\mathsfit}{\encodingdefault}{\sfdefault}{m}{sl}
\SetMathAlphabet{\mathsfit}{bold}{\encodingdefault}{\sfdefault}{bx}{n}

\newcommand{\E}{\mathbb{E}}

\newcommand{\R}{\mathbb{R}}

\newcommand{\softmax}{\mathrm{softmax}}

\DeclareMathOperator{\diag}{diag}

\usepackage{amssymb}

\usepackage{xcolor}
\definecolor{mydarkblue}{rgb}{0,0.08,0.45}
\definecolor{mygreen}{rgb}{0.032, 0.6392, 0.2039}
\definecolor{mypurple}{HTML}{B266FF}

\def\RR{{\mathbb R}}
\def\NN{{\mathbb N}}

\def\vv{{\boldsymbol v}}

\def\S{{\mathcal S}}
\def\A{{\mathcal A}}
\def\F{{\mathrm F}}

\def\RR{{\mathbb R}}

\def\NN{{\mathbb N}}

\def\vv{{\boldsymbol v}}

\def\Bar{\overline}

\def\le{{\leqslant}}

\newcommand{\prnbig}[1]{\big({#1}\big)} 

\newcommand{\norm}[1]{\left\| #1 \right\|}
\newcommand{\normbig}[1]{\big\| #1 \big\|}

\newcommand{\ex}[2]{\mathbb{E}_{#1}\left[#2\right]}
\newcommand{\exlim}[2]{\mathop\mathbb{E}\limits_{#1}\left[#2\right]}
\newcommand{\KL}[2]{\mathsf{KL}\prnbig{{#1}\,\|\,{#2}}}

\title{Federated Natural Policy Gradient and Actor Critic Methods\\
for Multi-task Reinforcement Learning}

\author{Tong Yang\thanks{Department of Electrical and Computer Engineering, Carnegie Mellon University; email: \texttt{tongyang@andrew.cmu.edu}. } \\
CMU    \\
	\and
Shicong Cen\thanks{Department of Electrical and Computer Engineering, Carnegie Mellon University; email: \texttt{shicongc@andrew.cmu.edu}. } \\
CMU 
 	\and
	Yuting Wei\thanks{Department of Statistics and Data Science, Wharton School, University of Pennsylvania; email: \texttt{ytwei@wharton.upenn.edu}.}\\
	UPenn\\
	\and
	Yuxin Chen\thanks{Department of Statistics and Data Science, Wharton School, University of Pennsylvania; email: \texttt{yuxinc@wharton.upenn.edu}. } \\
 UPenn  \\
	\and
	Yuejie Chi\thanks{Department of Electrical and Computer Engineering, Carnegie Mellon University; email: \texttt{yuejiechi@cmu.edu}. }\\
	CMU\\
	}
	
\date{October 2023; revised August 2024}

\begin{document}

\maketitle

\begin{abstract}

Federated reinforcement learning (RL) enables collaborative decision making of multiple distributed agents without sharing local data trajectories. In this work, we consider a multi-task setting, in which each agent has its own private reward function corresponding to different tasks, while sharing the same transition kernel of the environment. Focusing on infinite-horizon Markov decision processes, the goal is to learn a globally optimal policy that maximizes the sum of the discounted total rewards of all the agents in a decentralized manner, where each agent only communicates with its neighbors over some prescribed graph topology.

We develop federated vanilla and entropy-regularized natural policy gradient (NPG) methods in the tabular setting under softmax parameterization, where gradient tracking is applied to estimate the global Q-function to mitigate the impact of imperfect information sharing. We establish non-asymptotic global convergence guarantees under exact policy evaluation, where the rates are nearly independent of the size of the state-action space and illuminate the impacts of network size and connectivity. To the best of our knowledge, this is the first time that near dimension-free global convergence is established for federated multi-task RL using policy optimization. 
We further go beyond the tabular setting by proposing a federated natural actor critic (NAC) method for multi-task RL with function approximation, and establish its finite-time sample complexity taking the errors of function approximation into account.


    
    
\end{abstract}


	\tableofcontents

\section{Introduction}\label{sec:intro}

Federated reinforcement learning (FRL) is an emerging paradigm that combines the advantages of federated learning (FL) and reinforcement learning (RL)~\citep{qi2021federated,zhuo2019federated}, allowing multiple agents to learn a shared policy from local experiences, without exposing their private data to a central server nor other agents. FRL is poised to enable collaborative and efficient decision making in scenarios where data is distributed, heterogeneous, and sensitive, which arise frequently in applications such as edge computing, smart cities, and healthcare~\citep{wang2023federated,wang2020optimizing,zhuo2019federated}, to name just a few. As has been observed \citep{lian2017can}, decentralized training can lead to performance improvements in FL by avoiding communication congestions at busy nodes such as the server, especially under high-latency scenarios. This motivates us to design algorithms for the fully decentralized setting, a scenario where the agents can only communicate with their local neighbors over a prescribed network topology.\footnote{Our work seamlessly handles the server-client setting as a special case, by assuming the network topology as a fully connected network.}

In this work, we study the problem of 
\textit{federated multi-task reinforcement learning}~\citep{anwar2021multi,qi2021federated,yu2020learning}, where each agent collects its own reward --- possibly unknown to other agents --- corresponding to the local task at hand, while having access to the same dynamics (i.e., transition kernel) of the environment. The collective goal is to learn a shared policy that maximizes the total rewards accumulated from all the agents; in other words, one seeks a policy that performs well in terms of overall benefits, rather than biasing towards any individual task, achieving the Pareto frontier in a multi-objective context. There is no shortage of application scenarios where federated multi-task RL becomes highly relevant. For instance, in healthcare \citep{zerka2020systematic}, different hospitals may be interested in finding an optimal treatment for all patients without disclosing private data, where the effectiveness of the treatment can vary across different hospitals due to demographical differences. As another potential application, to enhance ChatGPT's performance across different tasks or domains~\citep{m2022exploring,rahman2023chatgpt}, one might consult domain experts to chat and rate ChatGPT's outputs for solving different tasks, and train ChatGPT in a federated manner without exposing private data or feedback of each expert. 


Nonetheless, despite the promise, provably efficient algorithms for federated multi-task RL remain substantially under-explored, especially in the fully decentralized setting. The heterogeneity of local tasks leads to a higher degree of disagreements between the global value function and  local value functions of individual agents. Due to the lack of global information sharing, care needs to be taken to judiciously balance the use of neighboring information (to facilitate consensus) and local data (to facilitate learning) when updating the policy. To the best of our knowledge, very limited  algorithms are currently available to find the global optimal policy with non-asymptotic convergence guarantees  even for tabular infinite-horizon Markov decision processes.

Motivated by the connection with decentralized optimization, it is tempting to take a policy optimization perspective to tackle this challenge. Policy gradient (PG) methods, which seek to learn the policy of interest via first-order optimization methods, play an eminent role in RL due to their simplicity and scalability. In particular, natural policy gradient (NPG) methods \citep{amari1998natural,kakade2001natural}  are among the most popular variants of PG methods, underpinning default methods used in practice such as trust region policy optimization (TRPO) \citep{schulman2015trust} and proximal policy optimization (PPO)  \citep{schulman2017proximal}. On the theoretical side, it has also been established recently that the NPG method enjoys fast global convergence to the optimal policy in an almost dimension-free manner \citep{agarwal2021theory,cen2021fast}, where the iteration complexity is nearly independent of the size of the state-action space. These benefits can be translated to their sample-based counterparts such as the natural actor critic (NAC) method \citep{xu2020improving}, where the value functions are learned via temporal difference learning. Inspired by the efficacy of NPG methods, it is natural to ask:
\begin{center}
\emph{Can we develop  \textbf{federated} variants of NPG and NAC methods in the fully decentralized setting, that come with \textbf{non-asymptotic and finite-sample global convergence} guarantees for multi-task RL?}
\end{center}


\subsection{Our contributions}

Focusing on infinite-horizon Markov decision processes (MDPs),
we provide an affirmative answer to the above question, by developing federated NPG (FedNPG) methods for solving both the vanilla and entropy-regularized multi-task RL problems with finite-time global convergence guarantees. While entropy regularization is often incorporated as an effective strategy to encourage exploration during policy learning, solving the entropy-regularized RL problem is of interest in its own right, as the optimal regularized policy possesses desirable robust properties with respect to reward perturbations \citep{eysenbach2021maximum,mckelvey1995quantal}.

Due to the multiplicative update nature of NPG methods under softmax parameterization, 
it is more convenient to work with the logarithms of local policies in the decentralized setting. In each iteration of the proposed FedNPG method, the logarithms of local policies are updated by a weighted linear combination of two terms (up to normalization): a gossip mixing \citep{nedic2009distributed} of the logarithms of neighboring local policies, and a local estimate of the global Q-function tracked via the technique of dynamic average consensus \citep{zhu2010discrete}, a prevalent idea in decentralized optimization that allows for the use of large constant learning rates \citep{di2016next,nedic2017achieving,qu2017harnessing} to accelerate convergence. We further develop sample-efficient federated NAC (FedNAC) methods that allow for both stochastic updates and function approximation. Our contributions are as follows.
\begin{itemize}
    \item We propose FedNPG methods for both the vanilla and entropy-regularized multi-task RL problems, where each agent only communicates with its neighbors and performs local computation using its own reward or task information.

    \item Assuming access to exact policy evaluation, we establish that the average iterate of vanilla FedNPG converges globally at a rate of $\mathcal{O}(1/T^{2/3})$ in terms of the sub-optimality gap for the multi-task RL problem, and that the last iterate of entropy-regularized FedNPG converges globally at a linear rate to the regularized optimal policy. Our convergence theory highlights the impacts of all salient problem parameters (see Table~\ref{tb:iteration_complexity} for details), such as the size and connectivity of the communication network. In particular, the iteration complexities of FedNPG are again almost independent  of the size of the state-action space, 
	    which recover prior results on the centralized NPG methods when the network is fully connected.
     
    \item We further demonstrate the stability of the proposed FedNPG methods when policy evaluations are only available in an inexact manner. To be specific, we prove that their convergence rates remain unchanged as long as the approximation errors are sufficiently small in the $\ell_\infty$ sense.

    \item We go beyond the tabular setting by proposing \alg --- a federated actor critic method for multi-task RL with function approximation --- and establish a finite-sample sample complexity on the order of $\mathcal{O}(1/\varepsilon^{7/2})$ for each agent in terms of the expected sub-optimality gap.

\end{itemize}

To the best of our knowledge, the proposed federated NPG and NAC methods are the first policy optimization methods for multi-task RL that achieve near dimension-free global convergence guarantees in terms of iteration and sample complexities, allowing for fully decentralized communication without any need to share local reward/task information. 

\begin{table*}[!tb]
\centering
 \resizebox{0.99\textwidth}{!}{
\begin{tabular}{cccc} 
\toprule
setting & algorithms & iteration complexity & optimality criteria \\
\midrule
\multirow{2}{*}
{
\makecell[c]
{~\\unregularized}
}&
\makecell[c]{NPG \citep{agarwal2021theory}} & $\mathcal{O}\left(\frac{1}{(1-\gamma)^2\varepsilon}+\frac{\log|\A|}{\eta\varepsilon}\right)$ $\vphantom{\frac{1^{7^{7^{7^{7^7}}}}}{1^{7^{7^{7^{7^{7^7}}}}}}}$ & $V^\star-V^{\pi^{(t)}}\leq\varepsilon$    \\ 
\cline{2-4}
&
\makecell[c]{ FedNPG (ours)} & \makecell[c]
{$\mathcal{O}\left(\frac{\sigma \sqrt{ N}\log|\A|}{(1-\gamma)^{\frac{9}{2}}(1-\sigma)\varepsilon^{\frac{3}{2}}} + \frac{1}{ (1-\gamma)^2\varepsilon} \right)$ $\vphantom{\frac{1^{7^{7^{7^{7^7}}}}}{1^{7^{7^{7^{7^{7^7}}}}}}}$ }
& $\frac{1}{T}\sum_{t=0}^{T-1}\big(V^{\star}-V^{\Bar\pi^{(t)}}\big)\leq\varepsilon$  \\ 
\hline
\multirow{2}{*}
{
\makecell[c]
{~\\regularized}
}&
\makecell[c]{NPG \citep{cen2021fast}} & $\mathcal{O}\left(\frac{1}{\tau\eta}\log \left(\frac{1}{\varepsilon}\right)\right)$ $\vphantom{\frac{1^{7^{7^{7^{7^7}}}}}{1^{7^{7^{7^{7^{7^7}}}}}}}$ & $V_\tau^\star-V_\tau^{\pi^{(t)}}\leq\varepsilon$   \\ 
\cline{2-4}
&
\makecell[c]{FedNPG (ours)} & $\mathcal O\left(\max\left\{\frac{1}{\tau\eta},\frac{1}{1-\sigma}\right\}\log\left(\frac{1}{\varepsilon}\right)\right)$ $\vphantom{\frac{1^{7^{7^{7^{7^7}}}}}{1^{7^{7^{7^{7^{7^7}}}}}}}$ & $V_\tau^\star-V_\tau^{\Bar\pi^{(t)}}\leq\varepsilon$   \\ 
\bottomrule
\end{tabular}
}
\caption{Iteration complexities of NPG and FedNPG (ours) methods to reach $\varepsilon$-accuracy of the vanilla and entropy-regularized problems, where we assume exact gradient evaluation, and only keep the dominant terms w.r.t. $\varepsilon$. The policy estimates in the $t$-iteration are $\pi^{(t)}$ and $\bar{\pi}^{(t)}$ for NPG and FedNPG, respectively, where $T$ is the number of iterations. Here, $N$ is the number of agents, $\tau\leq 1$ is the regularization parameter, $\sigma\in[0,1]$ is the spectral radius of the network, $\gamma \in [0,1)$ is the discount factor, $|\A|$ is the size of the action space, and $\eta>0$ is the learning rate. 
The iteration complexities of FedNPG reduce to their centralized counterparts when $\sigma=0$.  For vanilla FedNPG, the learning rate is set as $\eta =  \eta_1=\mathcal O \left(\frac{(1-\gamma)^9(1-\sigma)^2 \log |\A|}{  TN\sigma}\right)^{1/3}$; for entropy-regularized FedNPG, the learning rate satisfies $0<\eta< \eta_0 =\mathcal O\left(\frac{(1-\gamma)^7(1-\sigma)^2\tau}{\sigma N}\right)$.  }\label{tb:iteration_complexity}

\end{table*}

\subsection{Related work}\label{sec_app:related_work}

\paragraph{Global convergence of NPG methods for tabular MDPs.} 
\cite{agarwal2021theory} first establishes a $\mathcal O(1/T)$ last-iterate convergence rate of the  NPG method under softmax parameterization with constant step size, 
assuming access to exact policy evaluation. When entropy regularization is in place, 
\cite{cen2021fast} establishes a global linear convergence to the optimal regularized policy for the entire range of admissible constant learning rates using softmax parameterization and exact policy evaluation, 
which is further shown to be stable in the presence of $\ell_\infty$ policy evaluation errors. 
The iteration complexity of NPG methods is nearly independent with the size of the state-action space, 
which is in sharp contrast to softmax policy gradient methods that may take exponential time to converge \citep{li2021softmax,mei2020global}. 
\cite{lan2021policy} proposed a more general framework through the lens of mirror descent for regularized RL with global linear convergence guarantees, 
which is further generalized in \cite{zhan2021policy,lan2023block}. Earlier analysis of regularized MDPs can be found in \cite{shani2019adaptive}. 
Besides, \cite{xiao2022convergence} proves that vanilla NPG also achieves linear convergence when geometrically increasing learning rates are used; 
see also \cite{khodadadian2021linear,bhandari2020note}. \cite{zhou2022anchor} developed an anchor-changing NPG method for multi-task RL under various optimality criteria in the centralized setting.

\paragraph{Convergence and sample complexity bounds of NAC.}  
The convergence and sample complexity of a variety of natural actor-critic methods (NACs) are extensively studied in the literature~\citep{bhatnagar2009natural,wang2019neural,khodadadian2022finite,agarwal2021theory,yuan2022linear}.  
More pertinent to our work, \citet{agarwal2021theory} introduced Q-NPG --- a sample version of the NPG method with function approximation under softmax parameterization --- and obtained a  convergence rate of $\mathcal{O}(1/\sqrt{T})$. 
\citet{yuan2022linear} weakens some of its assumptions and improves the convergence rate to $\mathcal{O}(1/T)$ and gives the $\widetilde{\mathcal{O}}(1/\varepsilon^3)$ sample complexity using a constant actor learning rate. 
The FedNAC method we propose in this paper can be seen as a decentralized version of Q-NPG, and in the server-client setting where the network is fully connected, 
our convergence rate and sample complexity match those in \citet{yuan2022linear}. 

\paragraph{Distributed and federated RL.} There have been a variety of settings being set forth for distributed and federated RL. 
\cite{mnih2016asynchronous,espeholt2018impala,assran2019gossip,khodadadian2022federated,woo2023blessing} focused on developing federated versions of RL algorithms to accelerate training, 
assuming all agents share the same transition kernel and reward function; in particular, \cite{khodadadian2022federated,woo2023blessing,woo2024federated} established the provable benefits of federated learning in terms of linear speedup. 
More pertinent to our work, \cite{zhao2023federated,anwar2021multi} considered the federated multi-task framework, allowing different agents having private reward functions. 
\cite{zhao2023federated} proposed an empirically probabilistic algorithm that can seek an optimal policy under the server-client setting, while \cite{anwar2021multi} developed new attack methods in the presence of adversarial agents. Recently \cite{lan2023improved} discussed how to avoid transmitting the Hessian matrix during communication in the server-client setting where all agents share the same reward function.
Different from the FRL framework, \cite{chen2021communication,chen2021multi,omidshafiei2017deep,kar2012qd,chen2022decentralized,zeng2021decentralized}
considered the distributed multi-agent RL setting where the agents interact with a dynamic environment through a multi-agent Markov decision process, where each agent can have their own state or action spaces. \cite{zeng2021decentralized} developed a decentralized policy gradient method where different agents have different MDPs, where a special case of their setting recovers ours. However, the convergence rate developed in  \cite{zeng2021decentralized} has rather pessimistic dependencies with the size of the state-action space, together with other parameters, without leveraging natural policy gradients and gradient tracking techniques.

\paragraph{Decentralized first-order optimization algorithms.} Early work of consensus-based first-order optimization algorithms for the fully decentralized setting include but are not limited to \cite{lobel2008convergence,nedic2009distributed,duchi2011dual}. Gradient tracking, which leverages the idea of dynamic average consensus \citep{zhu2010discrete} to track the gradient of the global objective function, is a popular method to improve the convergence speed \citep{qu2017harnessing,nedic2017achieving,di2016next,pu2021distributed,li2020communication}.

%

\paragraph{Notation.} Boldface small and capital letters denote vectors and matrices, respectively. Sets are denoted with curly capital letters, e.g., $\S,\A$. We let $(\RR^d,\norm{\cdot})$ denote the $d$-dimensional real coordinate space equipped with norm $\norm{\cdot}$. The $\ell^p$-norm of $\vv$ is denoted by $\norm{\vv}_p$, where $1\leq p\leq \infty$, and the spectral norm and the Frobenius norm of a matrix $\mM$ are denoted by $\norm{\mM}_2$ and $\norm{\mM}_F$, resp. 
We let $[N]$ denote $\{1, \dots, N\}$, use $\vone_N$ to represent the all-one vector of length $N$, and denote by $\vzero$ a vector or a matrix consisting of all 0's. We allow the application of functions such as $\log(\cdot)$ and $\exp(\cdot)$ to vectors or matrices, with the understanding that they are applied in an element-wise manner.


\section{Model and backgrounds}\label{sec:preliminaries}


  
\subsection{Markov decision processes}

\paragraph{Markov decision processes.} We consider an infinite-horizon discounted Markov decision process (MDP) denoted by $ \mathcal{M} = (\mathcal{S}, \mathcal{A}, P, r, \gamma)$, where $\mathcal{S}$ and
$\mathcal{A}$ denote the state space and the action space, respectively, $\gamma\in [0, 1) $
indicates the discount factor, $P: \mathcal{S}\times \mathcal{A} \rightarrow \Delta(\mathcal{S})$ is the transition kernel, and $r: \mathcal{S}\times \mathcal{A} \rightarrow [0, 1]$ stands for the reward function. To be more specific, for each state-action pair $(s, a) \in \mathcal{S}\times \mathcal{A}$ and any state $s' \in \mathcal{S}$,
we denote by $P(s'|s, a)$ the transition probability from state $s$ to state $s'$ when action $a$ is taken, and $r(s, a) $ the instantaneous reward received in state $s$ when action $a$ is taken. Furthermore, a policy $\pi : \mathcal{S} \rightarrow \Delta(\mathcal{A})$ specifies an action selection rule, where $\pi(a|s)$ specifies the probability of taking action $a$ in state $s$
for each $(s, a) \in \mathcal{S}\times \mathcal{A}$.

For any given policy $\pi$, we denote by $V^\pi : \mathcal{S} \mapsto \RR$ the corresponding
value function, which is the expected discounted cumulative reward with an initial state $s_0 = s$, given by 
\begin{equation}\label{eq:V}
    \forall s\in \mathcal{S}: \quad V^\pi (s) \coloneqq \E\left[\sum_{t=0}^\infty \gamma^t r(s_t,a_t)| s_0=s\right],
\end{equation}
where the randomness is over the trajectory generated following the policy $a_t \sim \pi(\cdot|s_t)$ and the MDP dynamic $s_{t+1} \sim P(\cdot|s_t, a_t)$.
We also overload the notation $V^\pi (\rho)$ to indicate the expected value function of policy $\pi$ when the
initial state follows a distribution $\rho$ over $\mathcal{S}$, namely,
   $ V^\pi (\rho) \coloneqq \mathbb{E}_{s\sim \rho}\left[V^\pi(s)\right]$.
Similarly, the Q-function $Q^\pi : \mathcal{S} \times \mathcal{A} \mapsto \mathbb{R}$ of policy $\pi$ is defined by
\begin{equation}\label{eq:Q}
    \forall (s,a)\in \mathcal{S}\times \mathcal{A}:\quad  Q^\pi (s,a)\coloneqq  \E \left[\sum_{t=0}^\infty \gamma^t r(s_t,a_t)| s_0=s, a_0=a\right],
\end{equation}
which measures the expected discounted cumulative
reward with an initial state $s_0 = s$ and an initial action $a_0 = a$, with expectation taken over the randomness of the trajectory. The optimal policy $\pi^{\star}$ refers to the policy that maximizes the value function $V^{\pi}(s)$ for all states $s\in\S$, which is guaranteed to exist \citep{puterman2014markov}. The corresponding optimal value function and Q-function are denoted as $V^{\star}$ and $Q^{\star}$, respectively.
 
    
\subsection{Entropy-regularized RL}

Entropy regularization \citep{williams1991function,ahmed2019understanding} is a popular technique in practice that encourages stochasticity of the policy to promote exploration, as well as robustness against reward uncertainties. Mathematically, this can be viewed as adjusting the instantaneous reward based the current policy in use as
\begin{equation}\label{eq:reward_reg}
    \forall (s,a)\in \mathcal{S}\times \mathcal{A}:\quad r_\tau(s,a)\coloneqq r(s,a)-\tau \log{\pi (a|s)}\,,
\end{equation}
where $\tau\geq 0$ denotes the regularization parameter. Typically, $\tau$ should not be too large to outweigh the actual rewards; for ease of presentation, we assume $\tau\leq \min\left\{1, \, \frac{1}{\log|  \mathcal{A}|} \right\}$ \citep{cen2022faster}. Equivalently, this amounts to the entropy-regularized (also known as ``soft'') value function, defined as
\begin{equation}\label{eq:V_reg_s} 
     \forall s\in \mathcal{S}: \quad    V_\tau^\pi (s) \coloneqq V^\pi (s) +\tau \mathcal{H}(s,\pi). 
\end{equation}
Here, we define 
\begin{align}
 \mathcal{H}(s,\pi)& \coloneqq \E\left[\sum_{t=0}^\infty -\gamma^t \log{\pi (a_t|s_t)}\big | s_0=s\right] =\frac{1}{1-\gamma}\mathbb{E}_{s'\sim d_s^\pi}\left[ - \sum_{a\in\mathcal{A}}\pi(a|s')\log \pi(a|s')\right]   , \label{eq:entropy_reg_s} 
\end{align}
where  $d^\pi_{s_0}$ is the discounted state visitation distribution of policy $\pi$ given an initial state $s_0 \in \mathcal{S}$, denoted by
\begin{equation}\label{eq:d_s0} 
    \forall s\in\mathcal{S}: \quad d^\pi_{s_0}(s)\coloneqq (1-\gamma)\sum_{t=0}^\infty \gamma^t\mathbb{P}(s_t=s|s_0)\,,
\end{equation}
with the trajectory  generated by following policy $\pi$ in the MDP $\mathcal{M}$ starting from state $s_0$.  
Analogously, the regularized (or soft) Q-function $Q_\tau^\pi$ of policy $\pi$ is related to the soft value function $V_\tau^\pi(s)$ as
\begin{subequations}
\begin{align}
    \forall (s,a)\in \mathcal{S}\times \mathcal{A}:\quad Q_\tau^\pi(s,a)&=r(s,a)+\gamma\mathbb{E}_{s'\in P(\cdot|s,a)}\left[V_\tau^\pi(s')\right]\,,\label{eq:Q_V}\\ 
    \forall s\in \mathcal{S}:\quad V_\tau^\pi (s)&=\mathbb{E}_{a\sim \pi(\cdot|s)}\left[-\tau\pi(a|s)+Q_\tau^\pi(s,a)\right]\,.\label{eq:V_Q} 
\end{align}
\end{subequations}
The optimal regularized policy, the optimal regularized value function, and the Q-function 
are denoted by $\pi^\star_\tau$, $V^\star_\tau$, and $Q^\star_\tau$, respectively.

For a distribution $\rho\in\Delta(\S)$, we define $d_\rho^\pi(s)=\E_{s_0\sim\rho}[d_{s_0}^\pi(s)]$. We also define the \textit{state-action visitation distribution} $\bar{d}^\pi_\rho$ as
\begin{equation}\label{eq:sa_vis}
    \forall (s,a)\in\S\times \A:\quad  \bar{d}^\pi_{\rho}(s,a)\coloneqq d^\pi_{\rho}(s)\pi(a|s)=(1-\gamma)\E_{s_0\sim\rho}\left[\sum_{t=0}^\infty \gamma^t\mathbb{P}(s_t=s, a_t=a|s_0)\right].
\end{equation}
Furthermore, we define the {state-action visitation distribution} induced by an initial state-action distribution $\nu\in\Delta(\S\times\A)$, i.e.,
\begin{equation}\label{eq:sa_vis_extend}
    \forall (s,a)\in\S\times \A:\quad  \tilde{d}^\pi_{\nu}(s,a)\coloneqq (1-\gamma)\E_{(s_0,a_0)\sim\nu}\left[\sum_{t=0}^\infty \gamma^t\mathbb{P}(s_t=s, a_t=a|s_0,a_0)\right].
\end{equation}

The following simple fact holds for all $(s,a)\in\S\times\A$:
\begin{equation}\label{eq:d_facts}
    d_\rho^\pi(s)\geq (1-\gamma)\rho(s),\quad
    \bar{d}^\pi_{\rho}(s,a)\geq (1-\gamma)\rho(s)\pi(a|s),\quad \tilde{d}^\pi_{\nu}(s,a)\geq (1-\gamma)\nu(s,a).
\end{equation}

\subsection{Natural policy gradient methods}



Natural policy gradient (NPG) methods lie at the heart of policy optimization, serving as the backbone of popular heuristics such as TRPO \citep{schulman2015trust} and PPO \citep{schulman2017proximal}. Instead of directly optimizing the policy over the probability simplex, one often adopts the softmax parameterization, which parameterizes the policy as 
\begin{equation}\label{eq:softmax_policy} 
    \pi_\theta \coloneqq \softmax (\theta) \quad \text{or} \quad \forall (s,a)\in \mathcal{S}\times \mathcal{A}: \quad \pi_\theta(a|s)\coloneqq \frac{\exp{\theta(s,a)}}{\sum_{a'\in \mathcal{A}}\exp{\theta(s,a')}}
\end{equation}
for any $\theta$: $\mathcal{S}\times \mathcal{A} \rightarrow \mathbb{R}$. 

%


\paragraph{Vanilla NPG method.} In the tabular setting, the update rule of vanilla NPG at the $t$-th iteration can be concisely represented as 
\begin{equation}\label{eq:update_npg_vanilla}
                \forall (s,a)\in \mathcal{S}\times \mathcal{A}:\quad \pi^{(t+1)}(a|s) \propto \pi^{(t)}(a|s) \exp{\left(\frac{\eta Q^{(t)}(s,a)}{1-\gamma}\right)}\,,
\end{equation}
where $\eta>0$ denotes the learning rate, and $Q^{(t)} = Q^{\pi^{(t)}}$ is the Q-function under policy $\pi^{(t)}$. \citet{agarwal2021theory} shows that: in order to find an $\varepsilon$-optimal policy, NPG takes at most
$\mathcal{O}\left(\frac{1}{(1-\gamma)^2\varepsilon} \right)$
iterations, assuming exact policy evaluation.

\paragraph{Entropy-regularized NPG method.} Turning to the regularized problem, we note that the update rule of entropy-regularized NPG becomes
\begin{equation}\label{eq:update_npg}
                \forall (s,a)\in \mathcal{S}\times \mathcal{A}:\quad \pi^{(t+1)}(a|s) \propto (\pi^{(t)}(a|s))^{1-\frac{\eta\tau}{1-\gamma}}\exp{\left(\frac{\eta Q_\tau^{(t)}(s,a)}{1-\gamma}\right)}\,,
\end{equation}
where $\eta\in (0,\frac{1-\gamma}{\tau}]$ is the learning rate, and $Q_{\tau}^{(t)} = Q_{\tau}^{\pi^{(t)}}$ is the soft Q-function of policy $\pi^{(t)}$. 
\citet{cen2022fast} proves that entropy-regularized NPG enjoys fast global linear convergence to the optimal regularized policy: to find an $\varepsilon$-optimal regularized policy, entropy-regularized NPG takes no more than $\mathcal{O}\left(\frac{1}{\eta\tau}\log \left(\frac{1}{\varepsilon}\right)\right)$ iterations. 

\section{Federated NPG methods for multi-task RL}\label{sec:formulation}

\subsection{Federated multi-task RL}

In this paper, we consider the federated multi-task RL setting, where a set of agents learn collaboratively a single policy that maximizes its average performance over all the tasks using only local computation and communication.

\paragraph{Multi-task RL.} Each agent $n\in [N]$ has its own private reward function $r_n(s,a)$ --- corresponding to different tasks --- while sharing the same transition kernel of the environment. The goal is to collectively learn a single policy $\pi$ that maximizes the global value function given by
\begin{equation} \label{eq:global_V}
  V^{\pi}(s) = \frac{1}{N}\sum_{n=1}^N V_n^{\pi}(s) ,
\end{equation}
where $V_n^\pi$ is the value function of agent $n\in[N]$, defined by
\begin{equation}\label{eq:V_n} 
    \forall s\in \mathcal{S}: \quad V_n^\pi (s)\coloneqq  \E \left[\sum_{t=0}^\infty \gamma^t r_n(s_t,a_t)| s_0=s\right]\,.
\end{equation}
Clearly, the global value function \eqref{eq:global_V} corresponds to using the average reward of all agents
\begin{equation} \label{eq:global_r}
r(s, a)=\frac{1}{N}\sum_{n=1}^N r_n(s,a).
\end{equation}
The global Q-function $Q^{\pi}(s,a)$ and the agent Q-functions $Q_n^{\pi}(s,a)$ can be defined in a similar manner obeying $Q^\pi(s,a)=\frac{1}{N}\sum_{n=1}^N Q_n^\pi(s,a)$.
 
In parallel, we are interested in the entropy-regularized setting, where each agent $n\in[N]$ is equipped with a regularized reward function given by
\begin{equation}\label{eq:reward_reg_n}
    r_{\tau,n}(s,a) \coloneqq r_n(s,a)-\tau \log{\pi (a|s)}\,,
\end{equation}
and we define similarly the regularized value function and the global regularized value function as
\begin{equation}\label{eq:V_reg_n}
     \forall s\in \mathcal{S}: \quad   V_{\tau,n}^\pi (s) \coloneqq \mathbb{E}\left[\sum_{t=0}^\infty \gamma^t r_{\tau,n}(s_t,a_t)| s_0=s\right]\,, \quad\mbox{and}\quad   V_\tau^\pi(s)=\frac{1}{N}\sum_{n=1}^N V_{\tau,n}^\pi(s).
\end{equation}
The soft Q-function of agent $n$ is given by
\begin{equation}\label{eq:Q_reg_n}
    Q_{\tau,n}^\pi(s,a)=r_{n}(s,a)+\gamma\mathbb{E}_{s'\in P(\cdot|s,a)}\left[V_{\tau,n}^\pi(s')\right]\,,
\end{equation}
and  the global soft Q-function is given by $Q_\tau^\pi(s,a)=\frac{1}{N}\sum_{n=1}^N Q_{\tau,n}^\pi(s,a)$.

\paragraph{Federated policy optimization in the fully decentralized setting.} We consider a federated setting with fully decentralized communication, that is, all the agents are synchronized to perform information exchange over some prescribed network topology denoted by  an undirected weighted graph $\mathcal{G}([N],E)$. Here, $E$ stands for the edge set of the graph with $N$ nodes --- each corresponding to an agent --- and two agents can communicate with each other if and only if there is an edge connecting them. The information sharing over the graph is best described by a mixing matrix \citep{nedic2009distributed}, denoted by $\mW=[w_{ij}]\in[0,1]^{N\times N}$, where $w_{ij}$ is a positive number if $(i,j)\in E$ and 0 otherwise. We also make the following standard assumptions on the mixing matrix.
 \begin{asmp}[double stochasticity]\label{asmp:mixing matrix}
     The mixing matrix $\mW=[w_{ij}]\in[0,1]^{N\times N}$ is symmetric (i.e., $\mW^\top = \mW$) and
doubly stochastic (i.e., $\mW \vone_N = \vone_N$, $\vone_N^\top \mW = \vone_N^\top$). 
 \end{asmp}

The following standard metric measures how fast information propagates over the graph.
\begin{definition}[spectral radius]\label{lm:W}
  The spectral radius of $\mW$ is defined as
  \begin{equation}\label{eq:radius}
  \sigma \coloneqq \Big\| \mW-\frac{1}{N}\vone_N\vone_N^\top\Big\|_2 \in [0,1).
  \end{equation}  
\end{definition}

The spectral radius $\sigma$ determines how fast information propagate over the network. 
For instance, in a fully-connected network, we can achieve $\sigma = 0$ by setting $\mW = \frac{1}{N}\vone_N\vone_N^\top$. 
For control of $1/(1 - \sigma)$ regarding different graphs, we refer the readers to paper \cite{nedic2018network}. In an Erd\"{o}s-R\'{e}nyi random graph, as long as the graph is connected, one has  with high probability $\sigma \asymp 1$.
Another immediate consequence is that for any $\vx\in\R^N$, letting $\overline{x}=\frac{1}{N}\vone_N^\top\vx$ be its average, we have
    \begin{equation}\label{eq:property_W}
        \norm{\mW\vx-\Bar{x}\vone_N}_2\leq\sigma \norm{\vx-\Bar{x}\vone_N}_2\,,
    \end{equation}
where the consensus error contracts by a factor of $\sigma$. 

\subsection{Proposed federated NPG algorithms}

Assuming softmax parameterization, the problem can be formulated as decentralized optimization,
\begin{align}
\textsf{(unregularized)} \qquad  \max_{\theta} \; V^{\pi_\theta}(s) = \frac{1}{N}\sum_{n=1}^N V_n^{\pi_\theta}(s) ,  \label{eq:unreg_banana}   \\
\textsf{(regularized)} \qquad \max_{\theta} \; V_{\tau}^{\pi_\theta}(s) = \frac{1}{N}\sum_{n=1}^N V_{\tau,n}^{\pi_\theta}(s) ,  \label{eq:reg_banana} 
\end{align}
where $\pi_\theta \coloneqq \softmax (\theta) $ subject to communication constraints. Motivated by the success of NPG methods, we aim to develop federated NPG methods to achieve our goal. For notational convenience, let
$\vpi^{(t)} \coloneqq \big(\pi_1^{(t)},\cdots,\pi_N^{(t)} \big)^\top $
be the collection of policy estimates at all agents in the $t$-th iteration. Let 
\begin{equation}\label{eq:average_policy}
\overline{\pi}^{(t)} \coloneqq \softmax\left(\frac{1}{N}\sum_{n=1}^N \log \pi_n^{(t)}\right),
\end{equation}
which satisfies that $\overline{\pi}^{(t)}(a|s)\propto \left(\prod_{n=1}^N \pi_n^{(t)}(a|s)\right)^{1/N}$  for each $(s,a)\in \mathcal{S}\times \mathcal{A}$. Therefore, $\overline{\pi}^{(t)}$ could be seen as the normalized geometric mean of $\{\pi_n^{(t)}\}_{n\in[N]}$. Define the collection of Q-function estimates as
\begin{align*}
\mQ^{(t)}  \coloneqq \Big( Q_{1}^{\pi_1^{(t)}},\cdots, Q_{N}^{\pi_N^{(t)}}\Big)^\top , \qquad 
\mQ_\tau^{(t)}  \coloneqq \Big( Q_{\tau,1}^{\pi_1^{(t)}},\cdots, Q_{\tau,N}^{\pi_N^{(t)}}\Big)^\top.
\end{align*}
 We shall often abuse the notation and treat $\vpi^{(t)}$, $\mQ_\tau^{(t)}$ as matrices in $\R^{N\times |\mathcal{S}||\mathcal{A}|}$, and treat $\vpi^{(t)}(a|s)$, $\mQ_\tau^{(t)}(a|s)$ as vectors in $\R^N$, for all $ (s,a)\in \mathcal{S}\times\mathcal{A}$.  

\paragraph{Vanilla federated NPG methods.} To motivate the algorithm development, observe that the NPG method (cf.~\eqref{eq:update_npg_vanilla}) applied to \eqref{eq:unreg_banana} 
adopts the update rule
\begin{align*}
 \pi^{(t+1)}(a|s) \propto \pi^{(t)}(a|s) \exp{\left(\frac{\eta Q^{\pi^{(t)}}(s,a)}{1-\gamma}\right)} = \pi^{(t)}(a|s) \exp{\left(\frac{\eta \sum_{n=1}^N Q_n^{\pi^{(t)}}(s,a)}{N(1-\gamma)}\right)} \,
\end{align*}
for all $(s,a)\in \mathcal{S}\times \mathcal{A}$.
Two challenges arise when executing this update rule: the policy estimates are maintained locally without consensus, and the global Q-function are unavailable in the decentralized setting. To address these challenges, we apply the idea of dynamic average consensus \citep{zhu2010discrete}, where each agent maintains its own estimate $T_n^{(t)}(s,a)$ of the global Q-function, which are collected as vector 
$$  \mT^{(t)} =   \big( T_{1}^{(t)},\cdots, T_{N}^{ (t)}\big)^\top.$$ At each iteration, each agent updates its policy estimates
based on its neighbors' information via gossip mixing, in addition to a correction term that tracks the difference $Q_n^{\pi_n^{(t+1)}}(s,a) - Q_n^{\pi_n^{(t)}}(s,a) $ of the local Q-functions between consecutive policy updates. Note that the mixing is applied linearly to the logarithms of local policies, which translates into a multiplicative mixing of the local policies. Algorithm~\ref{alg:DNPG_exact} summarizes the detailed procedure of the proposed algorithm written in a compact matrix form, which we dub as federated NPG (FedNPG).
Note that the agents do not need to share their reward functions with others, and agent $n\in [N]$  will only be responsible to evaluate the local policy $\pi_n^{(t)}$ using the local reward $r_n$.

\begin{algorithm}[t]
    \caption{Federated NPG (FedNPG)}
    \label{alg:DNPG_exact}
    \begin{algorithmic}[1] 
        \STATE \textbf{Input:}  learning rate $\eta>0$, iteration number $T\in\NN_+$, mixing matrix $\mW\in\R^{N\times N}$.

        \STATE \textbf{Initialize:} $\vpi^{(0)}$, $\mT^{(0)}=\mQ^{(0)}$.
        
        \FOR{$t=0, 1, \cdots T-1$} 

            \STATE Update the policy for each $(s,a)\in \mathcal{S}\times \mathcal{A}$:
            \begin{equation}\label{eq:update_marl_0}
                \log\vpi^{(t+1)}(a|s)=\mW\left(\log\vpi^{(t)}(a|s)+\frac{\eta}{1-\gamma}\mT^{(t)}(s,a)\right)- \log\vz^{(t)}(s)\,,
            \end{equation}
            where $\vz^{(t)}(s)=\sum_{a'\in\mathcal{A}}\exp\left\{\mW\left(\log\vpi^{(t)}(a'|s)+\frac{\eta}{1-\gamma}\mT^{(t)}(s,a')\right)\right\}$.
            \STATE Evaluate $\mQ^{(t+1)}$.
            \STATE Update the global Q-function estimate for each $(s,a)\in \mathcal{S}\times \mathcal{A}$:
            \begin{equation}\label{eq:Q_tracking_0}
                \mT^{(t+1)}(s,a)=\mW\Big(\mT^{(t)}(s,a)+\underbrace{\mQ^{(t+1)}(s,a)-\mQ^{(t)}(s,a)}_{\text{Q-tracking}} \Big)\,.
            \end{equation}
        \ENDFOR 
    \end{algorithmic}
\end{algorithm}


\paragraph{Entropy-regularized federated NPG methods.} Moving onto the entropy regularized case, we adopt similar algorithmic ideas to decentralize \eqref{eq:update_npg}, and propose the federated NPG (FedNPG) method with entropy regularization, summarized in Algorithm~\ref{alg:EDNPG_exact}.   Clearly, the entropy-regularized FedNPG method reduces to the vanilla FedNPG in the absence of the regularization (i.e., when $\tau = 0$).

 \begin{algorithm}[h]
    \caption{Federated NPG (FedNPG) with entropy regularization}
    \label{alg:EDNPG_exact}
    \begin{algorithmic}[1] 
        \STATE \textbf{Input:}  learning rate $\eta>0$, iteration number $T\in\NN_+$, mixing matrix $\mW\in\R^{N\times N}$, regularization coefficient $\tau>0$.

        \STATE \textbf{Initialize:} $\vpi^{(0)}$, $\mT^{(0)}=\mQ_\tau^{(0)}$.
        
        \FOR{$t=0, 1, \cdots$} 

            \STATE Update the policy for each $(s,a)\in \mathcal{S}\times \mathcal{A}$:
            \begin{equation}\label{eq:update_marl}
                \log\vpi^{(t+1)}(a|s)=\mW
                \left(\left(1-\frac{\eta \tau}{1-\gamma}\right)\log\vpi^{(t)}(a|s)+\frac{\eta}{1-\gamma}\mT^{(t)}(s,a)\right)- \log\vz^{(t)}(s)\,,
            \end{equation}
            where $\vz^{(t)}(s)=\sum_{a'\in\mathcal{A}}\exp\left\{\mW\left(\left(1-\frac{\eta \tau}{1-\gamma}\right)\log\vpi^{(t)}(a'|s)+\frac{\eta}{1-\gamma}\mT^{(t)}(s,a')\right)\right\}$.
            \STATE Evaluate $\mQ_\tau^{(t+1)}$.
            
            \STATE Update the global Q-function estimate for each $(s,a)\in \mathcal{S}\times \mathcal{A}$:
            \begin{equation}\label{eq:Q_tracking}
                \mT^{(t+1)}(s,a)=\mW \Big(\mT^{(t)}(s,a)+\underbrace{\mQ_\tau^{(t+1)}(s,a)-\mQ_\tau^{(t)}(s,a)}_{\text{Q-tracking}} \Big)\,.
            \end{equation}
        \ENDFOR 
    \end{algorithmic}
\end{algorithm}



\subsection{Global convergence of FedNPG}

\paragraph{Convergence with exact policy evaluation.} We begin with the global convergence of FedNPG (cf.~Algorithm~\ref{alg:DNPG_exact}), stated in the following theorem.
The formal statement and proof of this result can be found in Appendix~\ref{sec_app:analysis_DNPG_exact}.

\begin{thm}[Global sublinear convergence of exact FedNPG (informal)]\label{thm:DNPG_exact_rate}
 Suppose $\pi_n^{(0)},n\in[N]$ are set as the uniform distribution. Then for $ 0<\eta\leq \eta_1\coloneqq \frac{(1-\sigma)^2(1-\gamma)^3}{16\sqrt{N}\sigma}$, we have
    \begin{equation}\label{eq:convergence_rate_exact_0_informal} 
        \frac{1}{T}\sum_{t=0}^{T-1}\left(V^{\star}(\rho)-V^{\Bar\pi^{(t)}}(\rho)\right) \leq  \frac{V^\star(d_\rho^{\pi^\star})}{(1-\gamma)T}+\frac{\log |\A|}{\eta T}+\frac{32 N\sigma^2 \eta^2}{(1-\gamma)^9(1-\sigma)^2}\,.
    \end{equation}
Furthermore, the consensus error satisfies
     \begin{equation}\label{eq:consensus_error_thm}
\forall n\in[N]:\quad \norm{\log\pi_n^{(t)}-\log\bar\pi^{(t)}}_\infty\leq\frac{32N\sigma}{3(1-\gamma)^4(1-\sigma)}\eta\,.
 \end{equation}
\end{thm}
Theorem~\ref{thm:DNPG_exact_rate} characterizes the average-iterate convergence of the average policy $\Bar\pi^{(t)}$ (cf.~\eqref{eq:average_policy}) across the agents, which  depends logarithmically on the size of the action space, and independently on the size of the state space. In addition, the consensus error of the local policies $\pi_n^{(t)}$ towards the average policy $\Bar\pi^{(t)}$ is characterized in \eqref{eq:consensus_error_thm}. When $T\geq \frac{128 \sqrt{N}\log |\A|\sigma^{4}}{(1-\sigma)^4}$, by optimizing the learning rate
     $\eta=\left(\frac{(1-\gamma)^9(1-\sigma)^2 \log |\A|}{ 32 TN\sigma^2}\right)^{1/3}$ to balance the latter two terms, we arrive at
 \begin{subequations} \label{eq:convergence_rate_exact_opt_0_inexact} 
    \begin{align}   
        \frac{1}{T}\sum_{t=0}^{T-1}\left(V^{\star}(\rho)-V^{\Bar\pi^{(t)}}(\rho)\right) & \lesssim  
        \frac{V^\star(d_\rho^{\pi^\star})}{(1-\gamma)T}+\frac{ N^{1/3}\sigma^{2/3}}{(1-\gamma)^3(1-\sigma)^{2/3}}\left(\frac{\log |\A|}{T}\right)^{2/3}\,. \\
        \norm{\log\pi_n^{(t)}-\log\bar\pi^{(t)}}_\infty & \lesssim \frac{ N^{2/3}\sigma^{1/3}}{(1-\gamma)(1-\sigma)^{1/3}}\left(\frac{\log |\A|}{T}\right)^{1/3}.
    \end{align}
\end{subequations}
    A few comments are in order.
\begin{itemize}
\item {\em Server-client setting.} When the network is fully connected, i.e., $\sigma =0$, the convergence rate of FedNPG recovers the $\mathcal{O}(1/T)$ rate, matching that of the centralized NPG established in \citet{agarwal2021theory}. 
\item {\em Well-connected networks.} When the network is relatively well-connected in the sense of $\frac{\sigma^2}{(1-\sigma)^2}\lesssim \frac{1-\gamma}{N^{1/2}}$, FedNPG first converges at the rate of $\mathcal{O}(1/T)$, and then at the slower $\mathcal{O}(1/T^{2/3})$ rate after $T \gtrsim  \frac{(1-\gamma)^3(1-\sigma)^2}{N\sigma^2}$. 

\item {\em Poorly-connected networks.} In addition, when the network is poorly connected in the sense of $\frac{\sigma^2}{(1-\sigma)^2}\gtrsim \frac{1-\gamma}{N^{1/2}}$, we see that FedNPG converges at the slower $\mathcal{O}(1/T^{2/3})$ rate.
\end{itemize}

We state the iteration complexity in Corollary~\ref{crl:DNPG_iteration_complexity}.

\begin{crl}[Iteration complexity of exact FedNPG]\label{crl:DNPG_iteration_complexity}
To reach 
$$\frac{1}{T}\sum_{t=0}^{T-1}\big(V^{\star}(\rho)-V^{\Bar\pi^{(t)}}(\rho)\big)\leq\varepsilon,$$ the iteration complexity of FedNPG is at most 
        $ \mathcal{O}\left(\left(\frac{\sigma}{(1-\gamma)^{9/2}(1-\sigma)\varepsilon^{3/2}}+\frac{\sigma^{2}}{(1-\sigma)^4}\right)\sqrt{N}\log|\A|+\frac{1}{\varepsilon (1-\gamma)^2}\right)$. 
\end{crl}

\paragraph{Convergence with inexact policy evaluation.} In practice, the policies need to be evaluated using samples collected by the agents, where the Q-functions are only estimated approximately. We are interested in gauging how the approximation error impacts the performance of FedNPG, as demonstrated in the following theorem.

\begin{thm}[Global sublinear convergence of inexact FedNPG (informal)]\label{thm:DNPG_inexact_convergence_informal}
    Suppose that an estimate $q_{n}^{\pi_n^{(t)}}$ are used in replace of $Q_n^{\pi_n^{(t)}}$ in Algorithm~\ref{alg:DNPG_exact}. Under the assumptions of Theorem~\ref{thm:DNPG_exact_rate}, we have
    \begin{equation}\label{eq:convergence_rate_inexact_0_informal} 
        \frac{1}{T}\sum_{t=0}^{T-1}\left(V^{\star}(\rho)-V^{\overline\pi^{(t)}}(\rho)\right)\leq  \frac{V^\star(d_\rho^{\pi^\star})}{(1-\gamma)T}+\frac{\log |\A|}{\eta T}+\frac{32 N\sigma^2 \eta^2}{(1-\gamma)^9(1-\sigma)^2}
        +C_3\max_{n\in[N], t\in [T]}\norm{Q_{n}^{\pi_n^{(t)}}-q_{n}^{\pi_n^{(t)}}}_\infty\,,
    \end{equation}
    where $C_3\coloneqq\frac{32\sqrt{N}\sigma\eta}{(1-\gamma)^5(1-\sigma)^2}\left(\frac{  \eta\sqrt{N}}{(1-\gamma)^3}+1\right)+\frac{2}{(1-\gamma)^2}$.

\end{thm}
\noindent The formal statement and proof of this result is given in Appendix~\ref{sec_app:analysis_DNPG_inexact}.

As long as 
   $ \max_{n\in[N], t\in[T]}\big\| Q_{n}^{\pi_n^{(t)}}-q_{n}^{\pi_n^{(t)}}\big\|_\infty\leq \frac{\varepsilon}{C_3}$,
 inexact FedNPG reaches $\frac{1}{T}\sum_{t=0}^{T-1}\big(V^{\star}(\rho)-V^{\Bar\pi^{(t)}}(\rho)\big)\leq 2\varepsilon$ at the same iteration complexity as predicted in 
Corollary~\ref{crl:DNPG_iteration_complexity}. Equipped with existing sample complexity bounds on policy evaluation, e.g. using a simulator as in \citet{li2020breaking} and \cite{li2023q}, this immediate leads to a sample complexity bound for a federated actor-critic type algorithm for multi-task RL. We detail this in the following remark.  
 
\begin{rmk}[sample complexity bound of inexact FedNPG]\label{rmk:sample_complexity_FedNPG}
Recall that
\citet{li2020breaking} shows that for any fixed policy $\pi$, model-based policy evaluation achieves $\norm{q_\tau^\pi-Q_\tau^\pi}_\infty\leq \varepsilon_{\mathsf{eval}}$
 with high probability if the number of samples per state-action pair exceeds the order of
$\widetilde{\mathcal{O}}\left(\frac{1}{(1-\gamma)^3  \varepsilon_{\mathsf{eval}}^2}\right)$.
When $T\gtrsim  \frac{  \sqrt{N}\log |\A|\sigma^{4}}{(1-\sigma)^4}$ and $\eta=\left(\frac{(1-\gamma)^9(1-\sigma)^2 \log |\A|}{32 TN\sigma^2}\right)^{1/3}$, we have $C_3 \asymp  1/(1-\gamma)^2 $. By employing fresh samples for the policy evaluation of each agent at every iteration, we can set $\varepsilon_{\mathsf{eval}}: = \max_{n\in[N], t\in [T]}\norm{Q_{n}^{\pi_n^{(t)}}-q_{n}^{\pi_n^{(t)}}}_\infty \asymp  \frac{\varepsilon}{C_3} \asymp (1-\gamma)^2\varepsilon $,  
and invoke the union bound over all iterations 
to give a loose upper bound of sample complexity of FedNPG per state-action pair at each agent as follows:
\begin{align*}
& \underbrace{ \widetilde{\mathcal{O}}\left(\left(\frac{\sigma}{(1-\gamma)^{9/2}(1-\sigma)\varepsilon^{3/2}}+\frac{\sigma^{2}}{(1-\sigma)^4}\right)\sqrt{N}+\frac{1}{\varepsilon (1-\gamma)^2}\right) }_{\mathsf{iteration~complexity}} \cdot \hspace{-0.1in}\underbrace{ \widetilde{\mathcal{O}}\left(\frac{1}{(1-\gamma)^7  \varepsilon^2}\right)}_{\mathsf{sample~complexity~per~iteration}}  \\
& =\widetilde{\mathcal{O}}\left(\frac{1}{(1-\gamma)^{7}\varepsilon^2}\cdot\left[\left(\frac{\sigma}{(1-\gamma)^{9/2}(1-\sigma)\varepsilon^{3/2}}+\frac{\sigma^{2}}{(1-\sigma)^4}\right)\sqrt{N}+\frac{1}{\varepsilon (1-\gamma)^2}\right]\right)\,.
\end{align*}
Hence, the total sample complexity scales linearly with respect to the size of the state-action space up to logarithmic factors. When $\sigma$ is close to 1, which corresponds to the case where the network exhibits a high degree of locality, the above sample complexity becomes
$$\widetilde{\mathcal{O}}\left(\frac{\sqrt{N}}{(1-\gamma)^{7}\varepsilon^2}\cdot\left[\left(\frac{1}{(1-\gamma)^{9/2}(1-\sigma)\varepsilon^{3/2}}+\frac{1}{(1-\sigma)^4}\right)\right]\right)\,,$$
which further simplifies to 
\begin{equation} \label{eq:sample_complexity_FedNPG}
    \widetilde{\mathcal{O}}\left(\frac{\sqrt{N}}{(1-\gamma)^{11.5} (1-\sigma) \varepsilon^{3.5}} \right)
\end{equation}
for sufficiently small $\varepsilon$.
\end{rmk}


\subsection{Global convergence of FedNPG with entropy regularization}

\paragraph{Convergence with exact policy evaluation.}  Next, we present our global convergence guarantee of entropy-regularized FedNPG with exact policy evaluation (cf.~Algorithm~\ref{alg:EDNPG_exact}). 

\begin{thm}[Global linear convergence of exact entropy-regularized FedNPG 
 (informal)]\label{thm:EDNPG_exact_rate}
    For any $\gamma\in(0,1)$ and $0<\tau\leq 1$, there exists $\eta_0  = \min \left\{\frac{1-\gamma}{\tau}, \mathcal O\left(\frac{(1-\gamma)^7(1-\sigma)^2\tau}{\sigma^2 N}\right)\right\} $, such that if $0<\eta\leq\eta_0$, then we have
\begin{equation}
\begin{split}
    \big\| \Bar{Q}_\tau^{(t)}-Q_\tau^\star \big\|_\infty\leq 2\gamma C_1\rho(\eta)^t\,,\quad \big\| \log\pi_\tau^\star-\log\Bar\pi^{(t)}\big\|_\infty\leq \frac{2C_1}{\tau}\rho(\eta)^t\,,
    \label{eq:EDNPG_linear_convergence} 
\end{split}
\end{equation}
where $\Bar{Q}_\tau^{(t)}:=Q_\tau^{\Bar\pi^{(t)}}$, $\rho(\eta)\leq\max\{1-\frac{\tau\eta}{2}, \frac{3+\sigma}{4}\}< 1$, and $C_1$ is some problem-dependent constant. Furthermore, the consensus error satisfies
\begin{equation}\label{eq:consensus_error_entropy_thm}
    \forall n\in[N]:\quad \normbig{\log\pi_n^{(t)}-\log\overline\pi^{(t)}}_\infty \leq 2 C_1 \rho(\eta)^t  .
\end{equation}
\end{thm}
The exact expressions of $C_1$ and $\eta_0$ are specified in Appendix~\ref{sec_app:analysis_EDNPG_exact}. Theorem~\ref{thm:EDNPG_exact_rate} confirms that entropy-regularized FedNPG converges at a linear rate to the optimal regularized policy, which is almost independent of the size of the state-action space, highlighting the positive role of entropy regularization in federated policy optimization. When the network is fully connected, i.e. $\sigma =0$, the iteration complexity of entropy-regularized FedNPG reduces to $\mathcal{O}\Big( \frac{1}{\eta\tau}\log\frac{1}{\varepsilon}\Big)$, matching that of the centralized entropy-regularized NPG established in \citet{cen2021fast}. When the network is less connected, one needs to be more conservative in the choice of learning rates, leading to a higher iteration complexity, as described in the following corollary.

\begin{crl}[Iteration complexity of exact entropy-regularized FedNPG]\label{crl:EDNPFG_iteration_complexity}
  To  reach $\norm{\log\pi_\tau^\star-\log\Bar\pi^{(t)}}_\infty\leq\varepsilon$,  the iteration complexity of entropy-regularized FedNPG is at most 
    \begin{equation}\label{eq:EDNPG_iteration_complexity_eta}
      \widetilde{\mathcal{O}} \left( \max\left\{\frac{2}{\tau\eta},\frac{4}{1-\sigma}\right\}\log\frac{1}{ \varepsilon} \right) 
    \end{equation}
up to logarithmic factors. 
    Especially, when $\eta=\eta_0$, the best iteration complexity becomes 
    $$\widetilde{\mathcal{O}}\left(\left( \frac{N \sigma^2}{(1-\gamma)^7(1-\sigma)^2\tau^2} +\frac{1}{1-\gamma}\right) \log\frac{1}{\tau\varepsilon}\right) .$$
    %
%
\end{crl}



\paragraph{Convergence with inexact policy evaluation.} 
Last but not the least, we present the informal convergence results of entropy-regularized FedNPG with inexact policy evaluation, whose formal version can be found in Appendix~\ref{sec_app:analysis_EDNPG_inexact}.

 \begin{thm}[Global linear convergence of inexact entropy-regularized FedNPG (informal)]\label{thm:EDNPG_inexact_convergence_informal}
     Suppose that an estimate $q_{\tau,n}^{\pi_n^{(t)}}$ are used in replace of $Q_{\tau,n}^{\pi_n^{(t)}}$ in Algorithm~\ref{alg:EDNPG_exact}. Under the assumptions of Theorem~\ref{thm:EDNPG_exact_rate}, we have
\begin{equation}\label{eq:EDNPG_linear_convergence_inexact} 
\begin{split}
\big\| \Bar{Q}_\tau^{(t)}-Q_\tau^\star\big\|_\infty &\leq 2\gamma \Big(C_1\rho(\eta)^t+C_2\max_{n\in[N], t\in[T]} \big\|Q_{\tau,n}^{\pi_n^{(t)}}-q_{\tau,n}^{\pi_n^{(t)}} \big\|_\infty \Big)\,,\\
    \big\| \log\pi_\tau^\star-\log\Bar\pi^{(t)} \big\|_\infty& \leq \frac{2}{\tau}\Big(C_1\rho(\eta)^t+C_2\max_{n\in[N],t\in[T]} \big\| Q_{\tau,n}^{\pi_n^{(t)}}-q_{\tau,n}^{\pi_n^{(t)}} \big\|_\infty \Big)\,,
\end{split}
\end{equation}
where  $\Bar{Q}_\tau^{(t)}:=Q_\tau^{\Bar\pi^{(t)}}$, $\rho(\eta)\leq\max\{1-\frac{\tau\eta}{2}, \frac{3+\sigma}{4}\}< 1$, and $C_1$, $C_2$ are problem-dependent constants.
\end{thm}

\section{Federated NAC methods for multi-task RL}
\label{sec:nac_ext}

In this section, motivated by the design and analysis of FedNPG, we go beyond the tabular setting and exact policy evaluation, by proposing a federated natural actor-critic (\alg) method  with function approximation and stochastic policy evaluation. Specifically,   we consider the policy with function approximation under softmax parameterization of the following form:
\begin{equation}\label{eq:param_fa}
    f_\xi(a|s)=\frac{\exp(\phi^\top(s,a)\vxi)}{\sum_{a'\in\A}\exp(\phi^\top(s,a')\vxi)},
\end{equation}
for all $(s,a)\in\S\times\A$ and $\vxi\in\R^p$, where
$\phi: \S\times\A\rightarrow \R^p$ is a known feature map. We assume $\phi$ is bounded over $\S\times\A$, i.e., there exists $C_\phi>0$ such that 
$$\norm{\phi(s,a)}_2\leq C_\phi$$ holds for all $(s,a)\in\S\times\A$. Following~~\citet{agarwal2021theory,yuan2022linear}, given any $\vw\in\R^p$, $Q:\S\times\A\rightarrow \R$ and  probability distribution $\zeta\in\Delta(\S\times\A)$ over the state-action space, we define the \textit{function approximation error} $\ell(\vw,Q,\zeta)$ as follows:
\begin{equation}\label{eq:l}
    \ell(\vw,Q,\zeta)\coloneqq  \E_{(s,a)\sim \zeta}\left[\left(\vw^\top\phi(s,a)-Q(s,a)\right)^2\right].
\end{equation}
By searching for $\vw$ that minimizes $\ell(\vw,Q,\zeta)$, it approximates $Q(s,a)$ using the feature map $\phi$ with respect to the distribution $\zeta$.

\begin{algorithm}[t]
    \caption{Federated Natural Actor-Critic (\alg)}
    \label{alg:actor_critic}
    \begin{algorithmic}[1] 
        \STATE \textbf{Input:}  
        number of actor iterations $T$, number of critic iterations $K$,
        actor learning rate $\alpha$, 
        critic learning rate $\beta$, discounted factor $\gamma\in [0,1)$
        \STATE \textbf{Initialization:} initial state-action distribution $\nu$, actor parameter $\vxi^{(0)}=(\vxi^{(0)\top}_{1},\cdots,\vxi^{(0)\top}_{N})^\top\in\R^{N\times p}$, $\vh^{(-1)}=\vw^{(-1)}=\vzero\in\R^{N\times p}$
        
        \FOR{$t=0, \cdots, T-1$} 
            \STATE Critic update: $\vw^{(t)}_n=$ \, \critic ($K$, $\nu$, $\vxi^{(t)}_n$, $\gamma$, $\beta$, $r_n$), $n\in[N]$ (Algorithm \ref{alg:critic})
            
            \STATE Update the critic parameter for estimating the global Q-function:
            \begin{equation}\label{eq:GT}
                \vh^{(t)}=\mW\left(\vh^{(t-1)}+\vw^{(t)}-\vw^{(t-1)}\right)
            \end{equation}
            
            \STATE Actor update:
            \begin{equation}\label{eq:actor_update}
                \vxi^{(t+1)}=\mW\left(\vxi^{(t)}+\alpha \vh^{(t)}\right)
            \end{equation}
        \ENDFOR 
    \end{algorithmic}
\end{algorithm}

\subsection{Algorithm design}

Our proposed federated NAC method {\alg} could be seen as a decentralized version of Q-NPG method~\citep{agarwal2021theory,yuan2022linear}, which we briefly review as follows.

\paragraph{Q-NPG method.} Q-NPG is a sample version of NPG with function approximation which is suitable for the case where $\S$ or $\A$ is large or infinite. 
We consider the policy with function approximation under softmax parameterization \eqref{eq:param_fa}.

Given an approximate solution $\vw^{(t)}$ for minimizing the function approximation error $\ell(\vw,Q^{f_{\xi^{(t)}}},\tilde d_\nu^{f_{\xi^{(t)}}})$ (see \eqref{eq:l}), the Q-NPG update rule $\vxi^{(t+1)}=\vxi^{(t)}+\alpha \vw^{(t)}$, when plugged in parameterization~\eqref{eq:param_fa}, results in the following policy update rule when we set $\alpha=\eta/(1-\gamma)$:
\begin{equation}\label{eq:Q_NPG_policy_update}
    f^{(t+1)}(a|s)\propto f^{(t)}(a|s)\exp\left(\frac{\eta\phi^\top(s,a)\vw^{(t)}}{1-\gamma}\right)\,,
\end{equation}
which could be seen as the function approximation version of the update rule~\eqref{eq:update_npg_vanilla} of vanilla NPG method.

\paragraph{Federated NAC method.}
Let us now discuss the high-level design of FedNAC, which is presented in Algorithm~\ref{alg:actor_critic}. At the $t$-th iteration ($t =0,\ldots,T-1$), denote the actor (concerning the policies) parameters of all agents as $\vxi^{(t)} =( \vxi_1^{(t)}, \ldots, \vxi_N^{(t)} )^{\top}\in \mathbb{R}^{N\times p}$, and the critic parameters of all agents as $\vw^{(t)} =( \vw_1^{(t)}, \ldots, \vw_N^{(t)} )^{\top}\in \mathbb{R}^{N\times p}$ (concerning the local Q-values) and $\vh^{(t)} =( \vh_1^{(t)}, \ldots, \vh_N^{(t)} )^{\top}\in \mathbb{R}^{N\times p}$ (concerning the global Q-values).  
\begin{itemize}
\item First, the critic parameter $\vw_n^{(t)}$ is locally updated at each agent by aiming to minimize $ \ell(\vw,Q_n^{(t)}, \tilde{d}_{n}^{(t)})$ (cf.~\eqref{eq:l}) with gradient descent, where $Q_n^{(t)}$ is the local Q-function of the local policy $f_{\xi_n^{(t)}}$, and $\tilde{d}_n^{(t)}$ is the state-action visitation distribution induced by the local policy $f_{\xi_n^{(t)}}$ and an initial state-action distribution $\nu$ (determined from the data sampling mechanism, cf.~\eqref{eq:sa_vis_extend}). However, since $Q_n^{(t)}$ is not directly available, it needs to be estimated from samples. Therefore, the critic update takes $K$ steps of stochastic gradient descent with critic learning rate $\beta$, given by
$$\widetilde{\vw}_{k+1}=\widetilde{\vw}_k-\beta \big(\widetilde{\vw}_k^\top \phi(s_k,a_k)-\widehat{Q}_{\xi}(s_k,a_k)\big)\phi(s_k,a_k), $$
for $k=0,\ldots, K-1$, where $(s_k, a_k)$ is sampled on the local policy $f_{\xi_n^{(t)}}$, and $\widehat{Q}_{\xi}(s_k,a_k)$ is a careful estimate of the Q-value using a trajectory with expected length $1/(1-\gamma)$ (see Algorithm~\ref{alg:sampler}  adopted from \citet[Lemma~4]{yuan2022linear}), and $\widetilde{\vw}_{0} = \bf{0}$ for simplicity. As a consequence, in line~4 of Algorithm~\ref{alg:critic}, we have 
\begin{equation}\label{eq:unbias_gd}
    \E \left[\widehat{\nabla}_w \ell(\widetilde{\vw}_k,\widehat{Q}^\pi,\tilde{d}^{f_\xi})\right]=\nabla_w \ell(\widetilde{\vw}_k,\widehat{Q}^\pi,\tilde{d}^{f_\xi})\,.
\end{equation}
The final critic is updated as $\vw_n^{(t)} = \frac{1}{K}\sum_{k=1}^K \widetilde{\vw}_k$. The total sample complexity of the critic update per iteration is then on the order of $K/(1-\gamma)$.

\item Next, the critic parameter $\vh_n^{(t)}$ for estimating the global Q-function can then be estimated by averaging with the neighbors with the Q-tracking term, given by
$                \vh^{(t)}=\mW\left(\vh^{(t-1)}+\vw^{(t)}-\vw^{(t-1)}\right). $

\item Finally, the actor parameter $\vxi_n^{(t)}$ can be updated via averaging with the neighbors along with the policy gradient informed by $\vh_n^{(t)}$, given by
$    \vxi^{(t+1)}=\mW\left(\vxi^{(t)}+\alpha \vh^{(t)}\right), $
where $\alpha$ is the learning rate of the actor. 
\end{itemize}
Note that the sample complexity of FedNAC is on the order of $KT/(1-\gamma)$.  An important aspect of the FedNAC method is that the policy is updated using trajectory data collected via executing the learned policy, which is closer to practice and more challenging to learn than using  the generative model. 

\begin{algorithm}[ht]
    \caption{\critic $(K, \nu, \vxi, \gamma, \beta, r)$: sample-based regression solver to minimize $\ell(\vw,Q_n^{(t)},\tilde d_n^{(t)})$ }
    \label{alg:critic}
    \begin{algorithmic}[1] 
        \STATE \textbf{Initialize:} critic parameter $w_0\in\R^p$
        \FOR{$k=0,\cdots, K-1$}
            \STATE Sampling: $(s_k, a_k), \widehat{Q}^\pi(s_k,a_k)=$\sampler($\nu, f_\xi, \gamma$, $r$) (Algorithm~\ref{alg:sampler}) 
            \STATE Compute the stochastic gradient estimator of $L_Q$: 
            \begin{equation}\label{eq:sg_LQ}
                \widehat{\nabla}_w \ell(\widetilde{\vw}_k,\widehat{Q}^\pi,\tilde{d}^{f_\xi})=2\left(\widetilde{\vw}_k^\top \phi(s_k,a_k)-\widehat{Q}^\pi(s_k,a_k)\right)\phi(s_k,a_k)
            \end{equation}
            \STATE Critic Update: $\widetilde{\vw}_{k+1}=\widetilde{\vw}_k-\beta \widehat{\nabla}_w \ell(\widetilde{\vw}_k,\widehat{Q}^\pi,\tilde{d}^{f_\xi})$
        \ENDFOR
        \STATE \textbf{Output:} $\vw_{\text{out}}=\frac{1}{K}\sum_{k=1}^K \widetilde{\vw}_k$
    \end{algorithmic}
\end{algorithm}

\begin{algorithm}[ht]
    \caption{\sampler $(\nu, \pi, \gamma, r)$}
    \label{alg:sampler}
    \begin{algorithmic}[1] 
        \STATE \textbf{Initialize:} $(s_0,a_0)\sim \nu$, time step $h,t=0$, variable $X\sim$ Bernoulli($\gamma$)
        \WHILE{$X=1$}
            \STATE Sample $s_{h+1}\sim P(\cdot|s_h,a_h)$
            \STATE Sample $a_{h+1}\sim \pi(\cdot|s_{h+1})$
            \STATE $h\leftarrow h+1$
            \STATE $X\sim$ Bernoulli($\gamma$)
        \ENDWHILE
        \STATE Set $\widehat{Q}^\pi(s_h,a_h)=r(s_h,a_h)$, $X\sim$ Bernoulli($\gamma$), $t=h$
        \WHILE{$X=1$}
            \STATE Sample $s_{t+1}\sim P(\cdot|s_t,a_t)$
            \STATE Sample $a_{t+1}\sim \pi(\cdot|s_{t+1})$
            \STATE $\widehat{Q}^\pi(s_h,a_h)\leftarrow \widehat{Q}^\pi(s_h,a_h)+r(s_{t+1},a_{t+1})$
            \STATE $t\leftarrow t+1$
            \STATE $X\sim$ Bernoulli($\gamma$)
        \ENDWHILE
        \STATE \textbf{Output:} $(s_h,a_h)$ and $\widehat{Q}^\pi(s_h,a_h)$
    \end{algorithmic}
\end{algorithm}

\subsection{Theoretical guarantees}
\label{sec:results_FedNAC}
We first state the assumptions that are needed to guarantee the convergence of Algorithm~\ref{alg:actor_critic}, which are all commonly used in the literature, e.g., \citet{yuan2022linear,agarwal2021theory}. To begin, we require the covariance matrix of the feature map induced by the initial state-action distribution $\nu$ satisfies the following assumption to guarantee the convergence of the critic.
\begin{asmp}[PSD of the covariance matrix of the feature map]\label{asmp:psd}
There exists $\mu>0$ such that
   \begin{equation}\label{eq:qsd}
       \E_{(s,a)\sim\nu}\left[\phi(s,a)\phi^\top(s,a)\right]=\Sigma_\nu\geq\mu\mI.
    \end{equation}
\end{asmp}

We also need to ensure that the Q-values can be well approximated by the linear function approximation using feature map $\phi(s,a)$, which is captured next.
\begin{asmp}[Bounded approximation error]\label{asmp:eps_approx}
    For each $n\in[N]$, there exists $\approxerror^n\geq 0$ such that for all $t\in\NN$, it holds that
     $   \E\left[\ell\left(\vw_{\star,n}^{(t)}, Q_n^{(t)},\tilde{d}_n^{(t)}\right)\right]\leq \approxerror^n$,
where $ \vw_{\star,n}^{(t)}\coloneqq\arg\min_{\vw} \ell\left(\vw_{\star,n}^{(t)}, Q_n^{(t)},\tilde{d}_n^{(t)}\right)$.
\end{asmp}
We denote the average approximation error as $\bar\varepsilon_{\text{approx}}=\frac{1}{N}\sum_{n=1}^N\varepsilon_{\text{approx}}^n$.
Similar as \citet{yuan2022linear}, we need the following assumption that bounds the transfer errors  due to distribution shifts. 
 

\begin{asmp}[Bounded transfer error]\label{asmp:bound_transfer} There exists $C_\nu>0$ such that for all $n\in[N]$ and $t\in\NN$, it holds that
   $     \E_{(s,a)\sim \tilde d_n^{(t)}}\left[\left(\frac{h^{\pi}(s,a)}{\tilde d_n^{(t)}(s,a)}\right)^2\right]\leq C_\nu$,
  where $h^{\pi}(s,a)$ is the state-action visitation distribution induced by any policy $\pi$ from initial state distribution $\rho$.
\end{asmp}
Note that if we choose $\nu(s,a)>0$ for all $(s,a)\in\S\times\A$, then Assumption~\ref{asmp:bound_transfer} is guaranteed to hold true (see Lemma~\ref{lm:ub_C_nu} in Appendix~\ref{sec_app:extension}).
We are now ready to state the convergence guarantee, whose formal version and proof could be found in Appendix~\ref{sec_app:extension}.

\begin{thm}[Convergence rate of Algorithm~\ref{alg:actor_critic} (informal)]\label{thm:FAC_convergence_informal}
Let $\vxi_1^{(0)}=\cdots=\vxi_N^{(0)}$ in \alg. Denoting  $\bar\vxi^{(t)}\coloneqq \frac{1}{N}\sum_{n=1}^N \vxi_n^{(t)}$, and $\bar f^{(t)} \coloneqq f_{\bar\xi^{(t)}}$ as the average policy. 
Then under Assumption~\ref{asmp:mixing matrix}, \ref{asmp:psd}, \ref{asmp:eps_approx} and \ref{asmp:bound_transfer}, with appropriately chosen learning rates $\alpha$ and $\beta$, as long as the number of actor iterations satisfies
\small
\begin{align*}
    T\gtrsim\max\bigg\{\frac{\sigma}{\varepsilon^{3/2}(1-\gamma)^{17/4}(1-\sigma)^{3/2}},\frac{1}{\varepsilon(1-\gamma)},
    \frac{\sigma^{1/4}}{\varepsilon^{3/4}(1-\sigma)^{3/8}(1-\gamma)^{7/8}N^{3/8}},\frac{\sigma^4}{(1-\gamma)^2(1-\sigma)^6}\bigg\}
\end{align*}
\normalsize
and the number of critic iterations satisfies
$K=\mathcal{O}\left(\frac{1}{(1-\gamma)^6\varepsilon^2}\right)$,
it holds that \vspace{-0.1in}
\begin{equation}\label{eq:fednac_opt}
V^{\star}(\rho)-\frac{1}{T}\sum_{t=0}^{T-1} V^{\bar f^{(t)}}(\rho)\lesssim \varepsilon+\frac{\bar\varepsilon_{approx}}{1-\gamma}. \vspace{-0.1in}
\end{equation}
\end{thm}



In the server-client setting when $\sigma=0$, to reach \eqref{eq:fednac_opt}, it suffices to choose $T=\mathcal{O}\left(\frac{1}{(1-\gamma)\varepsilon}\right)$ and $K=\mathcal{O}\left(\frac{1}{(1-\gamma)^6\varepsilon^2}\right)$, leading to a total sample complexity of $KT/(1-\gamma) = \mathcal{O}\left(\frac{1}{(1-\gamma)^8\varepsilon^3}\right)$ per agent, and $T=\mathcal{O}\left(\frac{1}{(1-\gamma)\varepsilon}\right)$ rounds of communication. The sample complexity matches that of (centralized) Q-NPG established in~\citet{yuan2022linear} with a single agent.
%
On the other end, in the fully decentralized setting when $\sigma$ is not close to 0, \alg~requires $\mathcal{O}\left(\frac{1}{(1-\gamma)^{45/4}\varepsilon^{7/2}(1-\sigma)^{3/2}}\right)$ samples for each agent and $\mathcal{O}\left(\frac{1}{\varepsilon^{3/2}(1-\gamma)^{17/4}(1-\sigma)^{3/2}}\right)$ rounds of communication to  reach \eqref{eq:fednac_opt}, for sufficiently small $\varepsilon$. Encouragingly, the dependency on the accuracy level $\varepsilon$ --- the dominating factor --- in the sample complexity matches that of FedNPG given in \eqref{eq:sample_complexity_FedNPG} when assuming access to the generative model, which allows query of arbitrary state-action pairs. In contrast, FedNAC only collects on-policy samples, and therefore is much more challenging to guarantee its convergence.

\section{Conclusions}\label{sec:conclusion}

This work proposes the first provably efficient federated NPG (FedNPG) methods for solving vanilla and entropy-regularized multi-task RL problems in the fully decentralized setting. The established finite-time global convergence guarantees are almost independent of the size of the state-action space up to some logarithmic factor, and illuminate the impacts of the size and connectivity of the network. Furthermore, the proposed FedNPG methods are provably robust vis-a-vis inexactness of local policy evaluations. 
Last but not least, we also propose~\alg, which can be viewed as an extension of FedNPG with function approximation and stochastic policy evaluation, and establish its finite-time sample complexity. Future directions include generalizing the framework of federated policy optimization to allow personalized policy learning in a shared environment. 


\section*{Acknowledgments}
	 
The work of T. Yang, S.~Cen and Y. Chi are supported in part by the grants ONR N00014-19-1-2404, NSF CCF-1901199, CCF-2106778,  AFRL FA8750-20-2-0504, and a CMU Cylab seed grant. 
The work of Y.~Wei is supported in part by the the NSF grants DMS-2147546/2015447, CAREER award DMS-2143215, CCF-2106778, and the Google Research Scholar Award. The work of Y.~Chen is supported in part by the Alfred P.~Sloan Research Fellowship, the Google Research Scholar Award, the AFOSR grant FA9550-22-1-0198, 
the ONR grant N00014-22-1-2354,  and the NSF grants CCF-2221009 and CCF-1907661. 
S.~Cen is also gratefully supported by Wei Shen and Xuehong Zhang Presidential Fellowship, Boeing Scholarship, and JP Morgan Chase PhD Fellowship.	 

{
\bibliographystyle{apalike}
\bibliography{bibfileRL,biblio}

\begin{thebibliography}{}

\bibitem[Agarwal et~al., 2021]{agarwal2021theory}
Agarwal, A., Kakade, S.~M., Lee, J.~D., and Mahajan, G. (2021).
\newblock On the theory of policy gradient methods: Optimality, approximation, and distribution shift.
\newblock {\em The Journal of Machine Learning Research}, 22(1):4431--4506.

\bibitem[Ahmed et~al., 2019]{ahmed2019understanding}
Ahmed, Z., Le~Roux, N., Norouzi, M., and Schuurmans, D. (2019).
\newblock Understanding the impact of entropy on policy optimization.
\newblock In {\em International Conference on Machine Learning}, pages 151--160.

\bibitem[Amari, 1998]{amari1998natural}
Amari, S.-I. (1998).
\newblock Natural gradient works efficiently in learning.
\newblock {\em Neural computation}, 10(2):251--276.

\bibitem[Anwar and Raychowdhury, 2021]{anwar2021multi}
Anwar, A. and Raychowdhury, A. (2021).
\newblock Multi-task federated reinforcement learning with adversaries.
\newblock {\em arXiv preprint arXiv:2103.06473}.

\bibitem[Assran et~al., 2019]{assran2019gossip}
Assran, M., Romoff, J., Ballas, N., Pineau, J., and Rabbat, M. (2019).
\newblock Gossip-based actor-learner architectures for deep reinforcement learning.
\newblock {\em Advances in Neural Information Processing Systems}, 32.

\bibitem[Bach and Moulines, 2013]{bach2013non}
Bach, F. and Moulines, E. (2013).
\newblock Non-strongly-convex smooth stochastic approximation with convergence rate o (1/n).
\newblock {\em Advances in neural information processing systems}, 26.

\bibitem[Bhandari and Russo, 2021]{bhandari2020note}
Bhandari, J. and Russo, D. (2021).
\newblock On the linear convergence of policy gradient methods for finite {MDP}s.
\newblock In {\em International Conference on Artificial Intelligence and Statistics}, pages 2386--2394. PMLR.

\bibitem[Bhatnagar et~al., 2009]{bhatnagar2009natural}
Bhatnagar, S., Sutton, R.~S., Ghavamzadeh, M., and Lee, M. (2009).
\newblock Natural actor-critic algorithms.
\newblock {\em Automatica}, 45(11):2471--2482.

\bibitem[Cen et~al., 2022a]{cen2022fast}
Cen, S., Cheng, C., Chen, Y., Wei, Y., and Chi, Y. (2022a).
\newblock Fast global convergence of natural policy gradient methods with entropy regularization.
\newblock {\em Operations Research}, 70(4):2563--2578.

\bibitem[Cen et~al., 2022b]{cen2022faster}
Cen, S., Chi, Y., Du, S.~S., and Xiao, L. (2022b).
\newblock Faster last-iterate convergence of policy optimization in zero-sum {M}arkov games.
\newblock In {\em The Eleventh International Conference on Learning Representations}.

\bibitem[Cen et~al., 2021]{cen2021fast}
Cen, S., Wei, Y., and Chi, Y. (2021).
\newblock Fast policy extragradient methods for competitive games with entropy regularization.
\newblock {\em Advances in Neural Information Processing Systems}, 34:27952--27964.

\bibitem[Chen et~al., 2022]{chen2022decentralized}
Chen, J., Feng, J., Gao, W., and Wei, K. (2022).
\newblock Decentralized natural policy gradient with variance reduction for collaborative multi-agent reinforcement learning.
\newblock {\em arXiv preprint arXiv:2209.02179}.

\bibitem[Chen et~al., 2021a]{chen2021communication}
Chen, T., Zhang, K., Giannakis, G.~B., and Ba{\c{s}}ar, T. (2021a).
\newblock Communication-efficient policy gradient methods for distributed reinforcement learning.
\newblock {\em IEEE Transactions on Control of Network Systems}, 9(2):917--929.

\bibitem[Chen et~al., 2021b]{chen2021multi}
Chen, Z., Zhou, Y., and Chen, R. (2021b).
\newblock Multi-agent off-policy {TDC} with near-optimal sample and communication complexity.
\newblock In {\em 2021 55th Asilomar Conference on Signals, Systems, and Computers}, pages 504--508. IEEE.

\bibitem[Di~Lorenzo and Scutari, 2016]{di2016next}
Di~Lorenzo, P. and Scutari, G. (2016).
\newblock Next: In-network nonconvex optimization.
\newblock {\em IEEE Transactions on Signal and Information Processing over Networks}, 2(2):120--136.

\bibitem[Duchi et~al., 2011]{duchi2011dual}
Duchi, J.~C., Agarwal, A., and Wainwright, M.~J. (2011).
\newblock Dual averaging for distributed optimization: Convergence analysis and network scaling.
\newblock {\em IEEE Transactions on Automatic control}, 57(3):592--606.

\bibitem[Espeholt et~al., 2018]{espeholt2018impala}
Espeholt, L., Soyer, H., Munos, R., Simonyan, K., Mnih, V., Ward, T., Doron, Y., Firoiu, V., Harley, T., Dunning, I., et~al. (2018).
\newblock Impala: Scalable distributed deep-rl with importance weighted actor-learner architectures.
\newblock In {\em International conference on machine learning}, pages 1407--1416. PMLR.

\bibitem[Eysenbach and Levine, 2021]{eysenbach2021maximum}
Eysenbach, B. and Levine, S. (2021).
\newblock Maximum entropy {RL} (provably) solves some robust {RL} problems.
\newblock In {\em International Conference on Learning Representations}.

\bibitem[Horn and Johnson, 2012]{horn2012matrix}
Horn, R.~A. and Johnson, C.~R. (2012).
\newblock {\em Matrix analysis}.
\newblock Cambridge university press.

\bibitem[Kakade, 2001]{kakade2001natural}
Kakade, S.~M. (2001).
\newblock A natural policy gradient.
\newblock {\em Advances in neural information processing systems}, 14.

\bibitem[Kar et~al., 2012]{kar2012qd}
Kar, S., Moura, J.~M., and Poor, H.~V. (2012).
\newblock Qd-learning: A collaborative distributed strategy for multi-agent reinforcement learning through consensus.
\newblock {\em arXiv preprint arXiv:1205.0047}.

\bibitem[Khodadadian et~al., 2022a]{khodadadian2022finite}
Khodadadian, S., Doan, T.~T., Romberg, J., and Maguluri, S.~T. (2022a).
\newblock Finite sample analysis of two-time-scale natural actor-critic algorithm.
\newblock {\em IEEE Transactions on Automatic Control}.

\bibitem[Khodadadian et~al., 2021]{khodadadian2021linear}
Khodadadian, S., Jhunjhunwala, P.~R., Varma, S.~M., and Maguluri, S.~T. (2021).
\newblock On the linear convergence of natural policy gradient algorithm.
\newblock In {\em 2021 60th IEEE Conference on Decision and Control (CDC)}, pages 3794--3799. IEEE.

\bibitem[Khodadadian et~al., 2022b]{khodadadian2022federated}
Khodadadian, S., Sharma, P., Joshi, G., and Maguluri, S.~T. (2022b).
\newblock Federated reinforcement learning: Linear speedup under {M}arkovian sampling.
\newblock In {\em International Conference on Machine Learning}, pages 10997--11057. PMLR.

\bibitem[Lan, 2023]{lan2021policy}
Lan, G. (2023).
\newblock Policy mirror descent for reinforcement learning: Linear convergence, new sampling complexity, and generalized problem classes.
\newblock {\em Mathematical programming}, 198(1):1059--1106.

\bibitem[Lan et~al., 2023a]{lan2023block}
Lan, G., Li, Y., and Zhao, T. (2023a).
\newblock Block policy mirror descent.
\newblock {\em SIAM Journal on Optimization}, 33(3):2341--2378.

\bibitem[Lan et~al., 2023b]{lan2023improved}
Lan, G., Wang, H., Anderson, J., Brinton, C., and Aggarwal, V. (2023b).
\newblock Improved communication efficiency in federated natural policy gradient via admm-based gradient updates.
\newblock {\em arXiv preprint arXiv:2310.19807}.

\bibitem[Li et~al., 2020]{li2020communication}
Li, B., Cen, S., Chen, Y., and Chi, Y. (2020).
\newblock Communication-efficient distributed optimization in networks with gradient tracking and variance reduction.
\newblock {\em The Journal of Machine Learning Research}, 21(1):7331--7381.

\bibitem[Li et~al., 2023a]{li2023q}
Li, G., Cai, C., Chen, Y., Wei, Y., and Chi, Y. (2023a).
\newblock Is q-learning minimax optimal? a tight sample complexity analysis.
\newblock {\em Operations Research}.

\bibitem[Li et~al., 2023b]{li2020breaking}
Li, G., Wei, Y., Chi, Y., and Chen, Y. (2023b).
\newblock Breaking the sample size barrier in model-based reinforcement learning with a generative model.
\newblock {\em Operations Research}.

\bibitem[Li et~al., 2023c]{li2021softmax}
Li, G., Wei, Y., Chi, Y., and Chen, Y. (2023c).
\newblock Softmax policy gradient methods can take exponential time to converge.
\newblock {\em Mathematical Programming}, pages 1--96.

\bibitem[Lian et~al., 2017]{lian2017can}
Lian, X., Zhang, C., Zhang, H., Hsieh, C.-J., Zhang, W., and Liu, J. (2017).
\newblock Can decentralized algorithms outperform centralized algorithms? a case study for decentralized parallel stochastic gradient descent.
\newblock {\em Advances in neural information processing systems}, 30.

\bibitem[Lobel and Ozdaglar, 2008]{lobel2008convergence}
Lobel, I. and Ozdaglar, A. (2008).
\newblock Convergence analysis of distributed subgradient methods over random networks.
\newblock In {\em 2008 46th Annual Allerton Conference on Communication, Control, and Computing}, pages 353--360. IEEE.

\bibitem[M~Alshater, 2022]{m2022exploring}
M~Alshater, M. (2022).
\newblock Exploring the role of artificial intelligence in enhancing academic performance: A case study of chatgpt.
\newblock {\em Available at SSRN}.

\bibitem[McKelvey and Palfrey, 1995]{mckelvey1995quantal}
McKelvey, R.~D. and Palfrey, T.~R. (1995).
\newblock Quantal response equilibria for normal form games.
\newblock {\em Games and economic behavior}, 10(1):6--38.

\bibitem[Mei et~al., 2020]{mei2020global}
Mei, J., Xiao, C., Szepesvari, C., and Schuurmans, D. (2020).
\newblock On the global convergence rates of softmax policy gradient methods.
\newblock In {\em International Conference on Machine Learning}, pages 6820--6829. PMLR.

\bibitem[Mnih et~al., 2016]{mnih2016asynchronous}
Mnih, V., Badia, A.~P., Mirza, M., Graves, A., Lillicrap, T., Harley, T., Silver, D., and Kavukcuoglu, K. (2016).
\newblock Asynchronous methods for deep reinforcement learning.
\newblock In {\em International conference on machine learning}, pages 1928--1937.

\bibitem[Nachum et~al., 2017]{nachum2017bridging}
Nachum, O., Norouzi, M., Xu, K., and Schuurmans, D. (2017).
\newblock Bridging the gap between value and policy based reinforcement learning.
\newblock In {\em Advances in Neural Information Processing Systems}, pages 2775--2785.

\bibitem[Nedi{\'c} et~al., 2018]{nedic2018network}
Nedi{\'c}, A., Olshevsky, A., and Rabbat, M.~G. (2018).
\newblock Network topology and communication-computation tradeoffs in decentralized optimization.
\newblock {\em Proceedings of the IEEE}, 106(5):953--976.

\bibitem[Nedic et~al., 2017]{nedic2017achieving}
Nedic, A., Olshevsky, A., and Shi, W. (2017).
\newblock Achieving geometric convergence for distributed optimization over time-varying graphs.
\newblock {\em SIAM Journal on Optimization}, 27(4):2597--2633.

\bibitem[Nedic and Ozdaglar, 2009]{nedic2009distributed}
Nedic, A. and Ozdaglar, A. (2009).
\newblock Distributed subgradient methods for multi-agent optimization.
\newblock {\em IEEE Transactions on Automatic Control}, 54(1):48--61.

\bibitem[Omidshafiei et~al., 2017]{omidshafiei2017deep}
Omidshafiei, S., Pazis, J., Amato, C., How, J.~P., and Vian, J. (2017).
\newblock Deep decentralized multi-task multi-agent reinforcement learning under partial observability.
\newblock In {\em International Conference on Machine Learning}, pages 2681--2690. PMLR.

\bibitem[Petersen and Pedersen, 2008]{petersen2008matrix}
Petersen, K.~B. and Pedersen, M.~S. (2008).
\newblock The matrix cookbook.
\newblock {\em Technical University of Denmark}, 7(15):510.

\bibitem[Pu and Nedi{\'c}, 2021]{pu2021distributed}
Pu, S. and Nedi{\'c}, A. (2021).
\newblock Distributed stochastic gradient tracking methods.
\newblock {\em Mathematical Programming}, 187:409--457.

\bibitem[Puterman, 2014]{puterman2014markov}
Puterman, M.~L. (2014).
\newblock {\em Markov decision processes: discrete stochastic dynamic programming}.
\newblock John Wiley \& Sons.

\bibitem[Qi et~al., 2021]{qi2021federated}
Qi, J., Zhou, Q., Lei, L., and Zheng, K. (2021).
\newblock Federated reinforcement learning: Techniques, applications, and open challenges.
\newblock {\em arXiv preprint arXiv:2108.11887}.

\bibitem[Qu and Li, 2017]{qu2017harnessing}
Qu, G. and Li, N. (2017).
\newblock Harnessing smoothness to accelerate distributed optimization.
\newblock {\em IEEE Transactions on Control of Network Systems}, 5(3):1245--1260.

\bibitem[Rahman et~al., 2023]{rahman2023chatgpt}
Rahman, M.~M., Terano, H.~J., Rahman, M.~N., Salamzadeh, A., and Rahaman, M.~S. (2023).
\newblock Chatgpt and academic research: a review and recommendations based on practical examples.
\newblock {\em Rahman, M., Terano, HJR, Rahman, N., Salamzadeh, A., Rahaman, S.(2023). ChatGPT and Academic Research: A Review and Recommendations Based on Practical Examples. Journal of Education, Management and Development Studies}, 3(1):1--12.

\bibitem[Schulman et~al., 2015]{schulman2015trust}
Schulman, J., Levine, S., Abbeel, P., Jordan, M., and Moritz, P. (2015).
\newblock Trust region policy optimization.
\newblock In {\em International conference on machine learning}, pages 1889--1897.

\bibitem[Schulman et~al., 2017]{schulman2017proximal}
Schulman, J., Wolski, F., Dhariwal, P., Radford, A., and Klimov, O. (2017).
\newblock Proximal policy optimization algorithms.
\newblock {\em arXiv preprint arXiv:1707.06347}.

\bibitem[Shani et~al., 2020]{shani2019adaptive}
Shani, L., Efroni, Y., and Mannor, S. (2020).
\newblock Adaptive trust region policy optimization: Global convergence and faster rates for regularized {MDP}s.
\newblock In {\em Proceedings of the AAAI Conference on Artificial Intelligence}, volume~34, pages 5668--5675.

\bibitem[Wang et~al., 2020]{wang2020optimizing}
Wang, H., Kaplan, Z., Niu, D., and Li, B. (2020).
\newblock Optimizing federated learning on non-iid data with reinforcement learning.
\newblock In {\em IEEE INFOCOM 2020-IEEE Conference on Computer Communications}, pages 1698--1707. IEEE.

\bibitem[Wang et~al., 2023]{wang2023federated}
Wang, J., Hu, J., Mills, J., Min, G., Xia, M., and Georgalas, N. (2023).
\newblock Federated ensemble model-based reinforcement learning in edge computing.
\newblock {\em IEEE Transactions on Parallel and Distributed Systems}.

\bibitem[Wang et~al., 2019]{wang2019neural}
Wang, L., Cai, Q., Yang, Z., and Wang, Z. (2019).
\newblock Neural policy gradient methods: Global optimality and rates of convergence.
\newblock {\em arXiv preprint arXiv:1909.01150}.

\bibitem[Williams and Peng, 1991]{williams1991function}
Williams, R.~J. and Peng, J. (1991).
\newblock Function optimization using connectionist reinforcement learning algorithms.
\newblock {\em Connection Science}, 3(3):241--268.

\bibitem[Woo et~al., 2023]{woo2023blessing}
Woo, J., Joshi, G., and Chi, Y. (2023).
\newblock The blessing of heterogeneity in federated q-learning: Linear speedup and beyond.
\newblock {\em arXiv preprint arXiv:2305.10697}.

\bibitem[Woo et~al., 2024]{woo2024federated}
Woo, J., Shi, L., Joshi, G., and Chi, Y. (2024).
\newblock Federated offline reinforcement learning: Collaborative single-policy coverage suffices.
\newblock In {\em Forty-first International Conference on Machine Learning}.

\bibitem[Xiao, 2022]{xiao2022convergence}
Xiao, L. (2022).
\newblock On the convergence rates of policy gradient methods.
\newblock {\em The Journal of Machine Learning Research}, 23(1):12887--12922.

\bibitem[Xu et~al., 2020]{xu2020improving}
Xu, T., Wang, Z., and Liang, Y. (2020).
\newblock Improving sample complexity bounds for actor-critic algorithms.
\newblock {\em arXiv preprint arXiv:2004.12956}.

\bibitem[Yu et~al., 2020]{yu2020learning}
Yu, T., Li, T., Sun, Y., Nanda, S., Smith, V., Sekar, V., and Seshan, S. (2020).
\newblock Learning context-aware policies from multiple smart homes via federated multi-task learning.
\newblock In {\em 2020 IEEE/ACM Fifth International Conference on Internet-of-Things Design and Implementation (IoTDI)}, pages 104--115. IEEE.

\bibitem[Yuan et~al., 2022]{yuan2022linear}
Yuan, R., Du, S.~S., Gower, R.~M., Lazaric, A., and Xiao, L. (2022).
\newblock Linear convergence of natural policy gradient methods with log-linear policies.
\newblock {\em arXiv preprint arXiv:2210.01400}.

\bibitem[Zeng et~al., 2021]{zeng2021decentralized}
Zeng, S., Anwar, M.~A., Doan, T.~T., Raychowdhury, A., and Romberg, J. (2021).
\newblock A decentralized policy gradient approach to multi-task reinforcement learning.
\newblock In {\em Uncertainty in Artificial Intelligence}, pages 1002--1012. PMLR.

\bibitem[Zerka et~al., 2020]{zerka2020systematic}
Zerka, F., Barakat, S., Walsh, S., Bogowicz, M., Leijenaar, R.~T., Jochems, A., Miraglio, B., Townend, D., and Lambin, P. (2020).
\newblock Systematic review of privacy-preserving distributed machine learning from federated databases in health care.
\newblock {\em JCO clinical cancer informatics}, 4:184--200.

\bibitem[Zhan et~al., 2023]{zhan2021policy}
Zhan, W., Cen, S., Huang, B., Chen, Y., Lee, J.~D., and Chi, Y. (2023).
\newblock Policy mirror descent for regularized reinforcement learning: A generalized framework with linear convergence.
\newblock {\em SIAM Journal on Optimization}, 33(2):1061--1091.

\bibitem[Zhao et~al., 2023]{zhao2023federated}
Zhao, F., Ren, X., Yang, S., Zhao, P., Zhang, R., and Xu, X. (2023).
\newblock Federated multi-objective reinforcement learning.
\newblock {\em Information Sciences}, 624:811--832.

\bibitem[Zhou et~al., 2022]{zhou2022anchor}
Zhou, R., Liu, T., Kalathil, D., Kumar, P., and Tian, C. (2022).
\newblock Anchor-changing regularized natural policy gradient for multi-objective reinforcement learning.
\newblock {\em Advances in Neural Information Processing Systems}, 35:13584--13596.

\bibitem[Zhu and Mart{\'\i}nez, 2010]{zhu2010discrete}
Zhu, M. and Mart{\'\i}nez, S. (2010).
\newblock Discrete-time dynamic average consensus.
\newblock {\em Automatica}, 46(2):322--329.

\bibitem[Zhuo et~al., 2019]{zhuo2019federated}
Zhuo, H.~H., Feng, W., Lin, Y., Xu, Q., and Yang, Q. (2019).
\newblock Federated deep reinforcement learning.
\newblock {\em arXiv preprint arXiv:1901.08277}.

\end{thebibliography}
}


\appendix
 


\section{Convergence analysis of FedNPG}\label{sec:analysis}


For technical convenience, we present first the analysis for entropy-regularized FedNPG and then for vanilla FedNPG.

 
\subsection{Analysis of entropy-regularized FedNPG with exact policy evaluation}
\label{sec_app:analysis_EDNPG_exact}

To facilitate analysis, we introduce several notation below. For all $t\geq 0$, we recall $\overline{\pi}^{(t)}$ as the normalized geometric mean of $\{\pi_n^{(t)}\}_{n\in[N]}$:
\begin{equation}\label{eq:overline_pi}
    \overline{\pi}^{(t)}\coloneqq \softmax\left(\frac{1}{N}\sum_{n=1}^N \log \pi_n^{(t)}\right)\,,
\end{equation}
from which we can easily see that for each $(s,a)\in \mathcal{S}\times \mathcal{A}$, $\overline{\pi}^{(t)}(a|s)\propto \left(\prod_{n=1}^N \pi_n^{(t)}(a|s)\right)^{\frac{1}{N}}$. We denote the soft $Q$-functions of $\overline{\pi}^{(t)}$ by $\overline{\mQ}_\tau^{(t)}$:
\begin{equation}\label{eq:overline_Q_tau_t_vec}
    \overline{\mQ}_\tau^{(t)}\coloneqq 
    \begin{pmatrix} Q_{\tau,1}^{\overline{\pi}^{(t)}}\\
    \vdots\\
    Q_{\tau,N}^{\overline{\pi}^{(t)}}\end{pmatrix}\,.
\end{equation}

In addition, we define $\widehat{Q}_\tau^{(t)}$, $\overline{Q}_\tau^{(t)}\in \R^{|\mathcal{S}||\mathcal{A}|}$ and $\overline{V}_\tau^{(t)}\in \R^{|\mathcal{S}|}$ as follows
\begin{subequations} \label{eq:QV_tautau}
\begin{align}
    \widehat{Q}_\tau^{(t)}&\coloneqq\frac{1}{N}\sum_{n=1}^N Q_{\tau,n}^{\pi_n^{(t)}}\,, \label{eq:hat_Q_tau_t}\\
    \overline{Q}_\tau^{(t)}&\coloneqq Q_\tau^{\overline{\pi}^{(t)}}=\frac{1}{N}\sum_{n=1}^N Q_{\tau,n}^{\overline{\pi}^{(t)}}\,. \label{eq:overline_Q_tau_t}\\
    \overline{V}_\tau^{(t)}&\coloneqq V_\tau^{\overline{\pi}^{(t)}}=\frac{1}{N}\sum_{n=1}^N V_{\tau,n}^{\overline{\pi}^{(t)}}\,.\label{eq:overline_V_tau_t}
\end{align}
\end{subequations}
For notational convenience, we also denote
\begin{equation}\label{eq:alpha}
    \alpha\coloneqq 1-\frac{\eta\tau}{1-\gamma}\,.
\end{equation}

Following \citet{cen2022fast}, we introduce the following auxiliary sequence 
$\{\vxi^{(t)}=(\xi_1^{(t)},\cdots,\xi_N^{(t)})^\top\in\R^{N\times|\mathcal{S}||\mathcal{A}|}\}_{t=0,1,\cdots}$, 
each recursively defined as
\begin{subequations} \label{eq:def_vxi}
\begin{align}
    \forall (s,a)\in \mathcal{S}\times \mathcal{A}:\quad\vxi^{(0)}(s,a)&\coloneqq \frac{\norm{\exp \left(Q_\tau^\star(s,\cdot)/\tau\right)}_1}{\norm{\exp \left(\frac{1}{N}\sum_{n=1}^N\log\pi_n^{(0)}(\cdot|s)\right)}_1}\cdot \vpi^{(0)}(a|s)\,, \label{eq:xi_0}\\
    \log\vxi^{(t+1)}(s,a)&=\mW\left(\alpha\log\vxi^{(t)}(s,a)+(1-\alpha)\mT^{(t)}(s,a)/\tau\right)\,,\label{eq:update_auxiliary}
\end{align}
\end{subequations}
where $\mT^{(t)}(s,a)$ is updated via \eqref{eq:Q_tracking_0}. Similarly, we introduce an averaged auxiliary sequence $\{\overline{\xi}^{(t)}\in\R^{|\mathcal{S}||\mathcal{A}|}\}$ given by
\begin{subequations}\label{eq:def_avg_xi}
\begin{align}
    \forall (s,a)\in \mathcal{S}\times \mathcal{A}:\quad \overline{\xi}^{(0)}(s,a)&\coloneqq \norm{\exp \left(Q_\tau^\star(s,\cdot)/\tau\right)}_1\cdot \overline{\pi}^{(0)}(a|s)\,,\\
    \log\overline{\xi}^{(t+1)}(s,a)&=
    \alpha\log\overline{\xi}^{(t)}(s,a)+(1-\alpha)\widehat{Q}_\tau^{(t)}(s,a)/\tau.  \label{eq:update_auxiliary_avg}
\end{align}
\end{subequations}

We introduces four error metrics defined as
\begin{subequations}
\begin{align}
    \Omega_1^{(t)}&\coloneqq\normbig{u^{(t)}}_\infty\,,\label{eq:Omega_1}\\
    \Omega_2^{(t)}&\coloneqq\normbig{v^{(t)}}_\infty\,,\label{eq:Omega_2}\\
    \Omega_3^{(t)}&\coloneqq\normbig{Q_\tau^\star -\tau \log \overline{\xi}^{(t)}}_\infty\,,\label{eq:Omega_3}\\
    \Omega_4^{(t)}&\coloneqq\max\left\{0, -\min_{s,a}\left(\overline Q_\tau^{(t)}(s,a)-\tau\log\overline\xi^{(t)}(s,a)\right)\right\},\label{eq:Omega_4}
\end{align}
\end{subequations}
where $u^{(t)},v^{(t)}\in\mathbb{R}^{|\S||\A|}$ are defined as
\begin{align}
    u^{(t)}(s,a)&\coloneqq \normbig{\log\vxi^{(t)}(s,a)-\log\overline\xi^{(t)}(s,a)\vone_N}_2\,,\label{eq:u}\\
    v^{(t)}(s,a)&\coloneqq \normbig{\mT^{(t)}(s,a)-\widehat{Q}_\tau^{(t)}(s,a)\vone_N}_2\,.\label{eq:v}
\end{align}
We collect the error metrics above in a vector $\mathbf{\Omega}^{(t)}\in\mathbb{R}^4$:
\begin{equation}\label{eq:Omega}
    \mathbf{\Omega}^{(t)}\coloneqq\left(\Omega_1^{(t)},\Omega_2^{(t)}, \Omega_3^{(t)}, \Omega_4^{(t)}\right)^\top\,.
\end{equation}

With the above preparation, we are ready to state the convergence guarantee of Algorithm~\ref{alg:EDNPG_exact} in Theorem~\ref{thm:distributed_exact_convergence} below, which is the formal version of Theorem~\ref{thm:EDNPG_exact_rate}.

\begin{thm}\label{thm:distributed_exact_convergence}
    For any $N\in\NN_+, \tau>0,\gamma\in(0,1)$, there exists $\eta_0>0$ which depends only on $N,\gamma,\tau,\sigma,|\A|$, such that if $0<\eta\leq\eta_0$ and $1-\sigma>0$, then the updates of Algorithm~\ref{alg:EDNPG_exact} satisfy
\begin{align}
    \normbig{\overline{Q}_\tau^{(t)}-Q_\tau^\star}_\infty & \leq 2\gamma\rho(\eta)^t\normbig{\mathbf{\Omega}^{(0)}}_2\,,\label{eq:linear_convergence_Q}\\
    \normbig{\log\pi_\tau^\star-\log\overline\pi^{(t)}}_\infty & \leq \frac{2}{\tau}\rho(\eta)^t\normbig{\mathbf{\Omega}^{(0)}}_2\,,\label{eq:linear_convergence_policy}
\end{align}
where 
$$\rho(\eta)\leq\max\Big\{1-\frac{\tau\eta}{2}, \frac{3+\sigma}{4}\Big\}< 1\,.$$
Moreover, the consensus errors satisfy:
\begin{equation}\label{eq:consensus_error_entropy}
    \forall n\in[N]:\quad \normbig{\log\pi_n^{(t)}-\log\overline\pi^{(t)}}_\infty \leq 2\rho(\eta)^t\normbig{\mathbf{\Omega}^{(0)}}_2\,.
\end{equation}
\end{thm}
The dependency of $\eta_0$ on $N,\gamma,\tau,\sigma,|\A|$ is made clear in Lemma~\ref{lm:bound_rho} that will be presented momentarily in this section.
The rest of this section is dedicated to the proof of Theorem \ref{thm:distributed_exact_convergence}. We first state a key lemma that tracks the error recursion of Algorithm~\ref{alg:EDNPG_exact}.


\begin{lm}
\label{lm:linear_sys_ent}
The following linear system holds for all $t \geq 0$:
\begin{equation}\label{eq:matrix}
    \mathbf{\Omega}^{(t+1)}
    \leq\underbrace{\begin{pmatrix}
        \sigma\alpha & \frac{\eta\sigma}{1-\gamma} & 0 & 0\\
         S\sigma & \left(1+\frac{\eta M\sqrt{N}}{1-\gamma}\sigma\right)\sigma & \frac{(2+\gamma)\eta MN}{1-\gamma}\sigma & \frac{\gamma\eta MN}{1-\gamma}\sigma\\
         (1-\alpha)M & 0 & (1-\alpha)\gamma+\alpha & (1-\alpha)\gamma\\
        \frac{2\gamma+\eta\tau}{1-\gamma}M & 0 & 0 & \alpha
    \end{pmatrix}}_{=:\mA(\eta)}
    \mathbf{\Omega}^{(t)}\,,
\end{equation}
where we let
\begin{equation}\label{eq:S}
    S\coloneqq M\sqrt{N}\left(2\alpha+(1-\alpha)\cdot \sqrt{2N}+\frac{1-\alpha}{\tau}\cdot\sqrt{N}M\right)\,,
\end{equation}
and
$$M\coloneqq \frac{1+\gamma+2\tau(1-\gamma)\log|\A|}{(1-\gamma)^2}\cdot\gamma\,.$$
In addition, it holds for all $t \geq 0$ that
\begin{align}
    \norm{\overline{Q}_\tau^{(t)}-Q_\tau^\star}_\infty &\leq \gamma\Omega_3^{(t)} + \gamma\Omega_4^{(t)}\,, \label{eq:ineq_omega1} \\
    \normbig{\log\overline\pi^{(t)}-\log\pi_\tau^\star}_\infty &\leq \frac{2}{\tau} \Omega_3^{(t)}\,.
\end{align}
\end{lm}

\begin{proof}
    See Appendix~\ref{sec:pf_lm_lin_sys_ent}.
\end{proof}


Let $\rho(\eta)$ denote the spectral norm of $\mA(\eta)$. As $\mathbf{\Omega}^{(t)} \geq 0$, it is immediate from \eqref{eq:matrix} that
\begin{equation*}
    \normbig{\mathbf{\Omega}^{(t)}}_2\leq \rho(\eta)^t\normbig{\mathbf{\Omega}^{(0)}}_2\,,
\end{equation*}
and therefore we have
\begin{equation*}
    \norm{\overline{Q}_\tau^{(t)}-Q_\tau^\star}_\infty\leq 2\gamma\normbig{\mathbf{\Omega}^{(t)}}_\infty \leq 2\gamma \rho(\eta)^t\normbig{\mathbf{\Omega}^{(0)}}_2\,,
\end{equation*}
and 
\begin{equation*}
    \normbig{\log\overline\pi^{(t)}-\log\pi_\tau^\star}_\infty\leq \frac{2}{\tau}\normbig{\mathbf{\Omega}^{(t)}}_\infty \leq \frac{2}{\tau} \rho(\eta)^t\normbig{\mathbf{\Omega}^{(0)}}_2\,.
\end{equation*}
It remains to bound the spectral radius $\rho(\eta)$, which is achieved by the following lemma.
\begin{lm}[Bounding the spectral norm of $\mA(\eta)$]\label{lm:bound_rho}
    Let 
    \begin{equation}\label{eq: zeta}
        \zeta\coloneqq \frac{(1-\gamma)(1-\sigma)^2\tau}{8\left(\tau S_0\sigma^2+10Mc\sigma^2/(1-\gamma)+(1-\sigma)^2\tau^2/16\right)}\,,
    \end{equation}
    where $S_0\coloneqq M\sqrt{N}\left(2+\sqrt{2N}+\frac{M\sqrt{N}}{\tau}\right)$, $c\coloneqq MN/(1-\gamma)$.
    For any $N\in\NN_+, \tau>0,\gamma\in(0,1)$, if 
    \begin{equation}\label{eq:eta_bound}
        0<\eta\leq \eta_0 \coloneqq {\min\Big\{\frac{1-\gamma}{\tau}, \zeta\Big\}}\,,
    \end{equation}
    then we have 
    \begin{equation}\label{eq:rho_bound}
        \rho(\eta)\leq\max\Big\{\frac{3+\sigma}{4}, \frac{1+(1-\alpha)\gamma+\alpha}{2}\Big\}< 1\,.
    \end{equation}
\end{lm}
\begin{proof}
    See Appendix~\ref{sec:pf_lm_bound_rho}.
\end{proof}


\subsection{Analysis of entropy-regularized FedNPG with inexact policy evaluation}
\label{sec_app:analysis_EDNPG_inexact}

We define the collection of {\em inexact} Q-function estimates as
\begin{align*}
\vq_\tau^{(t)}  \coloneqq \Big( q_{\tau,1}^{\pi_1^{(t)}},\cdots, q_{\tau,N}^{\pi_N^{(t)}}\Big)^\top,
\end{align*}
and then the update rule \eqref{eq:Q_tracking} should be understood as
    \begin{equation}\label{eq:Q_tracking_inexact}
                \mT^{(t+1)}(s,a)=\mW \left(\mT^{(t)}(s,a)+\vq_\tau^{(t+1)}(s,a)-\vq_\tau^{(t)}(s,a) \right)
            \end{equation}
in the inexact setting. For notational simplicity, we define
$e_n\in\R$ as
\begin{equation}\label{eq:delta_n}
    e_n:= \max_{t\in[T]}\norm{Q_{\tau,n}^{\pi_n^{(t)}}-q_{\tau,n}^{\pi_n^{(t)}}}_\infty\,,\quad n\in[N]\,,
\end{equation}
and let $\ve=(e_1,\cdots,e_n)^\top$.
Define $\widehat{q}_\tau^{(t)}$, the approximation of $\widehat{Q}_\tau^{(t)}$ as
\begin{align}
    \widehat{q}_\tau^{(t)}&\coloneqq\frac{1}{N}\sum_{n=1}^N q_{\tau,n}^{\pi_n^{(t)}}\,. \label{eq:hat_q_tau_t}
\end{align}

With slight abuse of notation, we adapt the auxiliary sequence $\{\overline{\xi}^{(t)}\}_{t=0,\cdots}$ to the inexact updates as
\begin{subequations}
\begin{align}
    \overline{\xi}^{(0)}(s,a)&\coloneqq \norm{\exp \left(Q_\tau^\star(s,\cdot)/\tau\right)}_1\cdot \overline{\pi}^{(0)}(a|s)\,, \label{eq:avg_xi_0_e}\\
    \overline{\xi}^{(t+1)}(s,a)&\coloneqq \left[\overline{\xi}^{(t)}(s,a)\right]^\alpha \exp\left((1-\alpha)\frac{\widehat{q}_\tau^{(t)}(s,a)}{\tau}\right)\,,\quad \forall (s,a)\in \mathcal{S}\times \mathcal{A}, \,\,t\geq 0\,.\label{eq:avg_xi_t+1_e}
\end{align}
\end{subequations}
In addition, we define
\begin{subequations}
\begin{align}
    \Omega_{1}^{(t)}&\coloneqq\norm{u^{(t)}}_\infty\,,\label{eq:Omega_1_inexact}\\
    \Omega_{2}^{(t)}&\coloneqq\norm{v^{(t)}}_\infty\,,\label{eq:Omega_2_inexact}\\
    \Omega_{3}^{(t)}&\coloneqq\norm{Q_\tau^\star -\tau \log \overline{\xi}^{(t)}}_\infty\,,\label{eq:Omega_3_inexact}\\
    \Omega_{4}^{(t)}&\coloneqq\max\left\{0, -\min_{s,a}\left(\overline q_\tau^{(t)}(s,a)-\tau\log\overline\xi^{(t)}(s,a)\right)\right\}\,,\label{eq:Omega_4_inexact}
\end{align}
\end{subequations}
where  
\begin{align}
    u^{(t)}(s,a)&\coloneqq \norm{\log\vxi^{(t)}(s,a)-\log\overline\xi^{(t)}(s,a)\vone_N}_2\,,\label{eq:u_inexact}\\
    v^{(t)}(s,a)&\coloneqq \norm{\mT^{(t)}(s,a)-\widehat{q}_\tau^{(t)}(s,a)\vone_N}_2\,.\label{eq:v_inexact}
\end{align}
We let $\mathbf{\Omega}^{(t)}$ be
\begin{equation}\label{eq:Omega_inexact}
    \mathbf{\Omega}^{(t)}\coloneqq\left(\Omega_{1}^{(t)},\Omega_{2}^{(t)}, \Omega_{3}^{(t)}, \Omega_{4}^{(t)}\right)^\top\,.
\end{equation}

With the above preparation, we are ready to state the inexact convergence guarantee of Algorithm~\ref{alg:EDNPG_exact} in Theorem~\ref{thm:EDNPG_inexact_convergence} below, which is
 the formal version of Theorem~\ref{thm:EDNPG_inexact_convergence_informal}.

\begin{thm}\label{thm:EDNPG_inexact_convergence}
    Suppose that $q_{\tau,n}^{\pi_n^{(t)}}$ are used in replace of $Q_{\tau,n}^{\pi_n^{(t)}}$ in Algorithm~\ref{alg:EDNPG_exact}. For any $N\in\NN_+, \tau>0,\gamma\in(0,1)$, there exists $\eta_0>0$ which depends only on $N,\gamma,\tau,\sigma,|\A|$, such that if $0<\eta\leq\eta_0$ and $1-\sigma>0$,  we have
\begin{align}
    \norm{\overline{Q}_\tau^{(t)}-Q_\tau^\star}_\infty & \leq 2\gamma\left(\rho(\eta)^t\norm{\mathbf{\Omega}^{(0)}}_2+ C_2\max_{n\in[N],t\in[T]}\norm{Q_{\tau,n}^{\pi_n^{(t)}}-q_{\tau,n}^{\pi_n^{(t)}}}_\infty\right)\,,\label{eq:linear_convergence_Q_inexact}\\
    \norm{\log\pi_\tau^\star-\log\overline\pi^{(t)}}_\infty &\leq \frac{2}{\tau}\left(\rho(\eta)^t\norm{\mathbf{\Omega}^{(0)}}_2+ C_2\max_{n\in[N],t\in[T]}\norm{Q_{\tau,n}^{\pi_n^{(t)}}-q_{\tau,n}^{\pi_n^{(t)}}}_\infty\right)\,.\label{eq:linear_convergence_policy_inexact}
\end{align} 
Moreover, the consensus errors satisfy:
\begin{equation}\label{eq:consensus_error_entropy_inexact}
    \forall n\in[N]:\quad \normbig{\log\pi_n^{(t)}-\log\overline\pi^{(t)}}_\infty \leq 2\left(\rho(\eta)^t\norm{\mathbf{\Omega}^{(0)}}_2+ C_2\max_{n\in[N],t\in[T]}\norm{Q_{\tau,n}^{\pi_n^{(t)}}-q_{\tau,n}^{\pi_n^{(t)}}}_\infty\right)\,,
\end{equation}
where $\rho(\eta)\leq\max\{1-\frac{\tau\eta}{2}, \frac{3+\sigma}{4}\}< 1$ is the same as in Theorem~\ref{thm:distributed_exact_convergence}, and $C_2\coloneqq\frac{\sigma\sqrt{N}(2(1-\gamma)+M\sqrt{N}\eta)+2\gamma^2+\eta\tau}{(1-\gamma)(1-\rho(\eta))}$.
\end{thm}

From Theorem~\ref{thm:EDNPG_inexact_convergence}, we can conclude that if 
\begin{equation}\label{eq:error_bound_E-DNPG}
    \max_{n\in[N],t\in[T]}\norm{Q_{\tau,n}^{\pi_n^{(t)}}-q_{\tau,n}^{\pi_n^{(t)}}}_\infty\leq \frac{(1-\gamma)(1-\rho(\eta))\varepsilon}{2\gamma\left(\sigma\sqrt{N}(2(1-\gamma)+M\sqrt{N}\eta)+2\gamma^2+\eta\tau\right)}\,,
\end{equation}
then inexact entropy-regularized FedNPG could still achieve 2$\varepsilon$-accuracy (i.e. $\norm{\overline{Q}_\tau^{(t)}-Q_\tau^\star}_\infty\leq 2\varepsilon$) within $\max\left\{\frac{2}{\tau\eta},\frac{4}{1-\sigma}\right\}\log\frac{2\gamma \norm{\mathbf{\Omega}^{(0)}}_2}{\varepsilon}$ iterations.

\begin{rmk}
When $\eta=\eta_0$ (cf.~\eqref{eq:eta_bound} and \eqref{eq: zeta}) and $\tau\leq 1$, the RHS of \eqref{eq:error_bound_E-DNPG} is of the order
$$\mathcal{O}\left(\frac{(1-\gamma)\tau\eta_0\varepsilon}{\gamma(\gamma^2+\sigma\sqrt{N}(1-\gamma))}\right)=\mathcal{O}\left(\frac{(1-\gamma)^8\tau^2(1-\sigma)^2\varepsilon}{\gamma(\gamma^2+\sigma\sqrt{N}(1-\gamma))(\gamma^2 N\sigma^2+(1-\sigma)^2\tau^2(1-\gamma)^6)}\right)\,,$$
which can be translated into a crude sample complexity bound when using fresh samples to estimate the soft Q-functions in each iteration. 
\end{rmk}

The rest of this section outlines the proof of Theorem~\ref{thm:EDNPG_inexact_convergence}.  We first state a key lemma that tracks the error recursion of Algorithm~\ref{alg:EDNPG_exact} with inexact policy evaluation, which is a modified version of Lemma~\ref{lm:linear_sys_ent}. 
\begin{lm}
\label{lm:linear_sys_ent_approx}
The following linear system holds for all $t\geq 0$:
\begin{equation}\label{eq:matrix_inexact}
    \mathbf{\Omega}^{(t+1)}
    \leq \mA(\eta)
    \mathbf{\Omega}^{(t)}
    +\underbrace{\begin{pmatrix}
        0\\
        \sigma\sqrt{N}\left(2+\frac{M\sqrt{N}\eta}{1-\gamma}\right)\\
        \frac{\eta\tau}{1-\gamma}\\
        \frac{2\gamma^2}{1-\gamma}
    \end{pmatrix}\norm{\ve}_\infty}_{=: \vb(\eta)}\,,
\end{equation}
where $\mA(\eta)$ is provided in Lemma~\ref{lm:linear_sys_ent}.
In addition, it holds for all $t \geq 0$ that
\begin{align} 
    \norm{\overline{Q}_\tau^{(t)}-Q_\tau^\star}_\infty & \leq \gamma\Omega_3^{(t)} + \gamma\Omega_4^{(t)}\,, \label{eq:watermelon}\\
    \normbig{\log\overline\pi^{(t)}-\log\pi_\tau^\star}_\infty & \leq \frac{2}{\tau} \Omega_3^{(t)}\,. \label{eq:firemelon}
\end{align}
\end{lm}
\begin{proof}
    See Appendix~\ref{sec_app:lm_linear_sys_ent_approx}.
\end{proof}
By \eqref{eq:matrix_inexact}, we have
\begin{equation*}
    \forall t\in N_+:\quad \mathbf{\Omega}^{(t)}\leq\mA(\eta)^t \mathbf{\Omega}^{(0)}+\sum_{s=1}^t\mA(\eta)^{t-s}\vb(\eta)\,,
\end{equation*}
which gives
\begin{align}
    \norm{\mathbf{\Omega}^{(t)}}_2 &\leq  \rho(\eta)^t\norm{\mathbf{\Omega}^{(0)}}_2+\sum_{s=1}^t\rho(\eta)^{t-s}\norm{\vb(\eta)}_2 \norm{\ve}_\infty \notag\\
   & \leq  \rho(\eta)^t\norm{\mathbf{\Omega}^{(0)}}_2+ \frac{\sigma\sqrt{N}(2(1-\gamma)+M\sqrt{N}\eta)+2\gamma^2+\eta\tau}{(1-\gamma)(1-\rho(\eta))}\norm{\ve}_\infty\,.\label{eq:estimate_t_inexact}
\end{align}
Here,  \eqref{eq:estimate_t_inexact} follows from $\norm{\vb(\eta)}_2\leq \norm{\vb(\eta)}_1=\frac{\sigma\sqrt{N}(2(1-\gamma)+M\sqrt{N}\eta)+2\gamma^2+\eta\tau}{1-\gamma}\norm{\ve}_\infty$ and $\sum_{s=1}^t\rho(\eta)^{t-s}\leq 1/(1-\rho(\eta))$. 
Recall that the bound on $\rho(\eta)$ has already been established in Lemma~\ref{lm:bound_rho}. Therefore we complete the proof of Theorem~\ref{thm:EDNPG_inexact_convergence} by combining the above inequality with \eqref{eq:watermelon} and \eqref{eq:firemelon} in a similar fashion as before. We omit further details for conciseness.

\subsection{Analysis of FedNPG with exact policy evaluation}
\label{sec_app:analysis_DNPG_exact}

We state the formal version of Theorem~\ref{thm:DNPG_exact_rate} below.

\begin{thm}\label{thm:distributed_exact_convergence_0}
    Suppose all $\pi_n^{(0)}$ in Algorithm~\ref{alg:DNPG_exact} are initialized as uniform distribution. When $$0<\eta\leq\eta_1\coloneqq \frac{(1-\sigma)^2(1-\gamma)^3}{8(1+\gamma)\gamma\sqrt{N}\sigma^2}\,,$$
   we have
    \begin{equation}\label{eq:convergence_rate_exact_0}
    \begin{split}
        \frac{1}{T}\sum_{t=0}^{T-1}\left(V^{\star}(\rho)-V^{\overline\pi^{(t)}}(\rho)\right) &\leq  \frac{V^\star(d_\rho^{\pi^\star})}{(1-\gamma)T}+\frac{\log |\A|}{\eta T}+\frac{8(1+\gamma)^2\gamma^2 N\sigma^2}{(1-\gamma)^9(1-\sigma)^2}\eta^2
    \end{split}
    \end{equation}
 for any fixed state distribution $\rho$.
Furthermore, we have
     \begin{equation}\label{eq:consensus_error_0}
\forall n\in[N]:\quad \norm{\log\pi_n^{(t)}-\log\overline\pi^{(t)}}_\infty\leq\frac{32N\sigma}{3(1-\gamma)^4(1-\sigma)}\eta\,.
 \end{equation}
\end{thm}

The rest of this section is dedicated to prove Theorem~\ref{thm:distributed_exact_convergence_0}.
Similar to \eqref{eq:overline_Q_tau_t_vec}, we denote the $Q$-functions of $\overline{\pi}^{(t)}$ by $\overline{\mQ}^{(t)}$:
\begin{equation}\label{eq:overline_Q_t_vec}
    \overline{\mQ}^{(t)}\coloneqq 
    \begin{pmatrix} Q_{1}^{\overline{\pi}^{(t)}}\\
    \vdots\\
    Q_{N}^{\overline{\pi}^{(t)}}\end{pmatrix}\,.
\end{equation}

In addition, similar to \eqref{eq:QV_tautau}, we define $\widehat{Q}^{(t)}$, $\overline{Q}^{(t)}\in \R^{|\mathcal{S}||\mathcal{A}|}$ and $\overline{V}^{(t)}\in \R^{|\mathcal{S}|}$ as follows
\begin{subequations} \label{eq:QV}
\begin{align}
    \widehat{Q}^{(t)}&\coloneqq\frac{1}{N}\sum_{n=1}^N Q_{n}^{\pi_n^{(t)}}\,, \label{eq:hat_Q_t}\\
    \overline{Q}^{(t)}&\coloneqq Q^{\overline{\pi}^{(t)}}=\frac{1}{N}\sum_{n=1}^N Q_{n}^{\overline{\pi}^{(t)}}\,. \label{eq:overline_Q_t}\\
    \overline{V}^{(t)}&\coloneqq V^{\overline{\pi}^{(t)}}=\frac{1}{N}\sum_{n=1}^N V_{n}^{\overline{\pi}^{(t)}}\,.\label{eq:overline_V_t}
\end{align}
\end{subequations}

Following the same strategy in the analysis of entropy-regularized FedNPG, we introduce the auxiliary sequence $\{\vxi^{(t)}=(\xi_1^{(t)},\cdots,\xi_N^{(t)})^\top\in\R^{N\times|\mathcal{S}||\mathcal{A}|}\}$ recursively:
\begin{subequations}
\begin{align}
    \vxi^{(0)}(s,a)&\coloneqq \frac{1}{\norm{\exp \left(\frac{1}{N}\sum_{n=1}^N\log\pi_n^{(0)}(\cdot|s)\right)}_1}\cdot \vpi^{(0)}(a|s)\,, \label{eq:zeta_0}\\
    \log\vxi^{(t+1)}(s,a)&=\mW\left(\log\vxi^{(t)}(s,a)+\frac{\eta}{1-\gamma}\mT^{(t)}(s,a)\right),\label{eq:zeta_t+1}
\end{align}
\end{subequations}
as well as the averaged auxiliary sequence $\{\overline{\xi}^{(t)}\in\R^{|\mathcal{S}||\mathcal{A}|}\}$:
\begin{subequations}
\begin{align}
   \overline{\xi}^{(0)}(s,a)&\coloneqq \overline{\pi}^{(0)}(a|s)\,, \label{eq:avg_xi_0}\\
    \log\overline{\xi}^{(t+1)}(s,a)&\coloneqq \log\overline{\xi}^{(t)}(s,a) + \frac{\eta}{1-\gamma}\widehat{Q}^{(t)}(s,a)\,,\quad \forall (s,a)\in \mathcal{S}\times \mathcal{A}, \,\,t\geq 0\,.\label{eq:avg_zeta_t+1}
    \end{align}
    \end{subequations}
As usual, we collect the consensus errors in a vector $\mathbf{\Omega}^{(t)} = (\normbig{u^{(t)}}_\infty,\normbig{v^{(t)}}_\infty)^\top$, where $u^{(t)},v^{(t)}\in \R^{|\S||\A|}$ are defined as:
\begin{align}
    u^{(t)}(s,a)& \coloneqq  \normbig{\log \vxi^{(t)}(s,a)-\log\overline\xi^{(t)}(s,a)\vone_N}_2\,,\label{eq:u_0}\\
    v^{(t)}(s,a)& \coloneqq  \normbig{\mT^{(t)}(s,a)-\widehat{Q}^{(t)}(s,a)\vone_N}_2\,.\label{eq:v_0}
\end{align}

\paragraph{Step 1: establishing the error recursion.} The next key lemma establishes the error recursion of Algorithm~\ref{alg:DNPG_exact}.
\begin{lm}
\label{lm:linear_sys}
The updates of FedNPG satisfy
\begin{equation}\label{eq:matrix_0}
    \mathbf{\Omega}^{(t+1)}
    \leq
    \underbrace{\begin{pmatrix}
        \sigma & \frac{\eta}{1-\gamma}\sigma\\ 
        J\sigma & \sigma\left(1+\frac{(1+\gamma)\gamma\sqrt{N}\eta}{(1-\gamma)^3}\sigma\right)
    \end{pmatrix}}_{=: \mB(\eta)}
    \mathbf{\Omega}^{(t)}
    +
    \underbrace{\begin{pmatrix}
        0\\
        \frac{(1+\gamma)\gamma N\sigma}{(1-\gamma)^4}\eta
    \end{pmatrix}}_{=:\vd(\eta)}
\end{equation}
for all $t\geq 0$, where
\begin{equation}\label{eq:J}
    J\coloneqq \frac{2(1+\gamma)\gamma}{(1-\gamma)^2}\sqrt{N}\,.
\end{equation}
In addition, we have
\begin{equation} \label{eq:induction_phi}
    \phi^{(t+1)}(\eta)\leq \phi^{(t)}(\eta)+\frac{2(1+\gamma)\gamma}{(1-\gamma)^4}\eta\normbig{u^{(t)}}_\infty-\eta\left({V}^{\star}(\rho)-\overline{V}^{(t)}(\rho)\right)\,,
\end{equation}
where
\begin{equation} \label{eq:def_phit}
    \phi^{(t)}(\eta)\coloneqq \ex{s\sim d_\rho^{\pi^\star}}{\KL{\pi^\star(\cdot|s)}{\overline\pi^{(t)}(\cdot|s)}}-\frac{\eta}{1-\gamma}\overline{V}^{(t)}(d_\rho^{\pi^\star})\,,\quad \forall t\geq 0\,.
\end{equation}
Moreover, when $\eta\leq\eta_1$, we have
\begin{equation}\label{eq:consensus_error_vanilla}
\forall n\in[N]:\quad \norm{\log\pi_n^{(t)}-\log\overline\pi^{(t)}}_\infty\leq2\left(\frac{3}{8}\sigma+\frac{5}{8}\right)^t\norm{\mathbf{\Omega}^{(0)}}_2+\frac{32N\sigma}{3(1-\gamma)^4(1-\sigma)}\eta\,.
\end{equation}
\end{lm}

\begin{proof}
    See Appendix~\ref{sec:pf_lm_linear_sys}.
\end{proof}

Note that when all $\pi_n^{(0)}$ in Algorithm~\ref{alg:DNPG_exact} are initialized as uniform distribution, $\mathbf{\Omega}^{(0)}=\vzero$ and \eqref{eq:consensus_error_vanilla} indicates \eqref{eq:consensus_error_0} in Theorem~\ref{thm:distributed_exact_convergence_0}.

\paragraph{Step 2: bounding the value functions.} 
Let $\vp\in\R^2$ be defined as:
\begin{equation}\label{eq:p}
    \vp(\eta)=
    \begin{pmatrix}
        p_1(\eta)\\
        p_2(\eta)
    \end{pmatrix}
    \coloneqq
    \frac{2(1+\gamma)\gamma}{(1-\gamma)^4}
    \begin{pmatrix}
    \frac{\sigma(1-\gamma)\left(1-\sigma-(1+\gamma)\gamma\sqrt{N}\sigma\eta/(1-\gamma)^3\right)\eta}{(1-\gamma)\left(1-\sigma-(1+\gamma)\gamma\sqrt{N}\sigma^2\eta/(1-\gamma)^3\right)(1-\sigma)-J\sigma^2\eta}\\
    \frac{\sigma\eta^2}{(1-\gamma)\left(1-\sigma-(1+\gamma)\gamma\sqrt{N}\sigma^2\eta/(1-\gamma)^3\right)(1-\sigma)-J\sigma^2\eta} 
    \end{pmatrix};
\end{equation}
the rationale for this choice will be made clear momentarily. We define the following Lyapunov function 
\begin{equation}\label{eq:Lyapunov}
    \Phi^{(t)}(\eta)=\phi^{(t)}(\eta)+\vp(\eta)^\top \mathbf{\Omega}^{(t)}\,,\quad \forall t\geq 0\,,
\end{equation}
which satisfies
\begin{align}
    \Phi^{(t+1)}(\eta)&=\phi^{(t+1)}(\eta)+\vp(\eta)^\top \mathbf{\Omega}^{(t+1)}\notag\\
    &\leq \phi^{(t)}(\eta)+\frac{2(1+\gamma)\gamma}{(1-\gamma)^4}\eta\normbig{u^{(t)}}_\infty-\eta\left(V^{\star}(\rho)-\overline{V}^{(t)}(\rho)\right) + \vp(\eta)^\top\left(\mB(\eta)\mathbf{\Omega}^{(t)}+\vd(\eta)\right) \notag \\
    &= \Phi^{(t)}(\eta)+\left[\vp(\eta)^\top\left(\mB(\eta)-\mI \right)+\left(\frac{2(1+\gamma)\gamma}{(1-\gamma)^4}\eta, 0\right)\right]\mathbf{\Omega}^{(t)}-\eta\left(V^{\star}(\rho)-\overline{V}^{(t)}(\rho)\right)\notag\\
    &\qquad +p_2(\eta)\frac{(1+\gamma)\gamma N\sigma}{(1-\gamma)^4}\eta\,.\label{eq:ineq_lyap}
\end{align}
Here, the second inequality follows from \eqref{eq:induction_phi}.
One can verify that the second term vanishes due to the choice of $\vp(\eta)$:
\begin{equation} \label{eq:magic}
\vp(\eta)^\top\left(\mB(\eta)-\mI \right)+\left(\frac{2(1+\gamma)\gamma}{(1-\gamma)^4}\eta, 0\right)=(0,0)\,.
\end{equation}
Therefore, we conclude that
$$V^{\star}(\rho)-\overline{V}^{(t)}(\rho)\leq\frac{\Phi^{(t)}(\eta)-\Phi^{(t+1)}(\eta)}{\eta}+p_2(\eta)\frac{(1+\gamma)\gamma N\sigma}{(1-\gamma)^4}\,.$$
Averaging over $t = 0,\cdots, T-1$,
\begin{align}
    &\frac{1}{T}\sum_{t=0}^{T-1}\left(V^{\star}(\rho)-\overline{V}^{(t)}(\rho)\right)\notag\\
    &\leq \frac{\Phi^{(0)}(\eta)-\Phi^{(T)}(\eta)}{\eta T}+\frac{2(1+\gamma)^2\gamma^2}{(1-\gamma)^8}\cdot\frac{N\sigma^2\eta^2 }{(1-\gamma)(1-\sigma-(1+\gamma)\gamma\sqrt{N}\sigma^2\eta/(1-\gamma)^3)(1-\sigma)-\sigma ^2 J\eta}\,.\label{eq:ub_to_be_refined}
\end{align}


\paragraph{Step 3: simplifying the expression.}
We first upper bound the first term in the RHS of \eqref{eq:ub_to_be_refined}. 
Assuming uniform initialization for all $\pi_n^{(0)}$ in Algorithm~\ref{alg:DNPG_exact}, we have $\norm{u^{(0)}}_\infty=\norm{v^{(0)}}_\infty=0$, and 
$$\mathbb{E}_{s\sim d_\rho^{\pi^\star}}\left[\KL{\pi^\star(\cdot|s)}{\overline\pi^{(0)}(\cdot|s)}\right]\leq\log|\A|.$$
Therefore, putting together relations \eqref{eq:Lyapunov} and \eqref{eq:phi} we have
\begin{align}
    \frac{\Phi^{(0)}(\eta)-\Phi^{(T)}(\eta)}{\eta T} & \leq  \frac{\log |\A|}{T\eta}+\frac{1}{T}\left(\vp(\eta)^\top\mathbf{\Omega}^{(0)}/\eta+\frac{V^\star(d_\rho^{\pi^\star})}{1-\gamma}\right)=\frac{\log |\A|}{T\eta}+\frac{V^\star (d_\rho^{\pi^\star})}{T(1-\gamma)}\,,\label{eq:ub_term1}
\end{align}

To continue, we upper bound the second term in the RHS of \eqref{eq:ub_to_be_refined}. 
Note that
\begin{equation*}
    \eta\leq \eta_1\leq \frac{(1-\sigma)(1-\gamma)^3}{2(1+\gamma)\gamma\sqrt{N}\sigma^2}\,,
\end{equation*}
which gives
\begin{equation}\label{eq:bd_denominator_1}
    \frac{(1+\gamma)\gamma\sqrt{N}\sigma^2}{(1-\gamma)^3}\eta\leq \frac{1-\sigma}{2}.
\end{equation}

Thus we have
\begin{align}
    &(1-\gamma)(1-\sigma-(1+\gamma)\gamma\sqrt{N}\sigma^2\eta/(1-\gamma)^3)(1-\sigma)-J\sigma^2\eta\notag\\
   & \geq  (1-\gamma)(1-\sigma)^2/2-J\sigma^2\eta_1\notag\\
   & \geq  (1-\gamma)(1-\sigma)^2/4\,,\label{eq:lb_denominator}
\end{align}
where the first inequality follows from \eqref{eq:bd_denominator_1} and the second inequality follows from the definition of $\eta_1$ and $J$.
By \eqref{eq:lb_denominator}, we deduce
\begin{equation}\label{eq:ub_term2}
    \frac{2(1+\gamma)^2\gamma^2}{(1-\gamma)^8}\cdot\frac{N\sigma^2\eta^2}{(1-\gamma)(1-\sigma-(1+\gamma)\gamma\sqrt{N}\sigma^2\eta/(1-\gamma)^3)(1-\sigma)-J\sigma^2\eta}\leq \frac{8(1+\gamma)^2\gamma^2 N\sigma^2}{(1-\gamma)^9(1-\sigma)^2}\eta^2\,,
\end{equation}
and our advertised bound~\eqref{eq:convergence_rate_exact_0} thus follows from plugging \eqref{eq:ub_term1} and \eqref{eq:ub_term2} into \eqref{eq:ub_to_be_refined}.


\subsection{Analysis of FedNPG with inexact policy evaluation}
\label{sec_app:analysis_DNPG_inexact}

We state the formal version of Theorem~\ref{thm:DNPG_inexact_convergence_informal} below.
\begin{thm}\label{thm:DNPG_inexact_convergence}
    Suppose that $q_{n}^{\pi_n^{(t)}}$ are used in replace of $Q_n^{\pi_n^{(t)}}$ in Algorithm~\ref{alg:DNPG_exact}. Suppose all $\pi_n^{(0)}$ in Algorithm~\ref{alg:DNPG_exact} set to uniform distribution. Let $$0<\eta\leq\eta_1\coloneqq \frac{(1-\sigma)^2(1-\gamma)^3}{8(1+\gamma)\gamma\sqrt{N}\sigma^2}\,,$$
 we have
    \begin{equation}\label{eq:convergence_rate_inexact_0}
    \begin{split}
        &\frac{1}{T}\sum_{t=0}^{T-1}\left(V^{\star}(\rho)-V^{\overline\pi^{(t)}}(\rho)\right)\notag\\
        &\leq \frac{V^\star(d_\rho^{\pi^\star})}{(1-\gamma)T}+\frac{\log |\A|}{\eta T}+\frac{8(1+\gamma)^2\gamma^2 N\sigma^2}{(1-\gamma)^9(1-\sigma)^2}\eta^2\\
        &\qquad+\left[\frac{8(1+\gamma)\gamma}{(1-\gamma)^5(1-\sigma)^2}\sqrt{N}\sigma\eta\left(\frac{(1+\gamma)\gamma\eta\sqrt{N}}{(1-\gamma)^3}+2\right)+\frac{2}{(1-\gamma)^2}\right]\max_{n\in[N],t\in[T]}\norm{Q_{n}^{\pi_n^{(t)}}-q_{n}^{\pi_n^{(t)}}}_\infty
    \end{split}
   \end{equation}
 for any fixed state distribution $\rho$.  

Furthermore, we have
\begin{equation}\label{eq:consensus_error}
\forall n\in[N]:\quad \norm{\log\pi_n^{(t)}-\log\overline\pi^{(t)}}_\infty\leq\frac{32}{3(1-\sigma)}\left(\frac{N\sigma}{(1-\gamma)^4}\eta+\sqrt{N}\sigma\left(\frac{\eta \sqrt{N}}{(1-\gamma)^3}+1\right)\max_{n\in[N],t\in[T]}\norm{Q_{n}^{\pi_n^{(t)}}-q_{n}^{\pi_n^{(t)}}}_\infty\right)\,.
\end{equation}
\end{thm}

We next outline the proof of Theorem~\ref{thm:DNPG_inexact_convergence}. With slight abuse of notation,
we again define $e_n\in\R$ as
\begin{equation}\label{eq:delta_n_vanilla}
    e_n:= \max_{t\in[T]}\norm{Q_{n}^{\pi_n^{(t)}}-q_{n}^{\pi_n^{(t)}}}_\infty\,,\quad n\in[N]\,,
\end{equation}
and let $\ve=(e_1,\cdots,e_n)^\top$. We define the collection of {\em inexact} Q-function estimates as
\begin{align*}
\vq^{(t)}  \coloneqq \Big( q_{1}^{\pi_1^{(t)}},\cdots, q_{N}^{\pi_N^{(t)}}\Big)^\top,
\end{align*}
and then the update rule \eqref{eq:Q_tracking_0} should be understood as
    \begin{equation}\label{eq:Q_tracking_0_inexact}
                \mT^{(t+1)}(s,a)=\mW \left(\mT^{(t)}(s,a)+\vq^{(t+1)}(s,a)-\vq^{(t)}(s,a) \right)
            \end{equation}
in the inexact setting.
Define $\widehat{q}^{(t)}$, the approximation of $\widehat{Q}^{(t)}$ as
\begin{align}
    \widehat{q}^{(t)}&\coloneqq\frac{1}{N}\sum_{n=1}^N q_{n}^{\pi_n^{(t)}}\,, \label{eq:hat_q_t}
\end{align}
we adapt the averaged auxiliary sequence $\{\overline{\xi}^{(t)}\in\R^{|\mathcal{S}||\mathcal{A}|}\}$ to the inexact updates as follows:
\begin{subequations}
\begin{align}
    \overline{\xi}^{(0)}(s,a)&\coloneqq \overline{\pi}^{(0)}(a|s)\,, \label{eq:avg_zeta_0_inexact}\\
    \overline{\xi}^{(t+1)}(s,a)&\coloneqq \overline{\xi}^{(t)}(s,a)\exp\left(\frac{\eta}{1-\gamma}\widehat{q}^{(t)}(s,a)\right)\,,\quad \forall (s,a)\in \mathcal{S}\times \mathcal{A}, \,\,t\geq 0\,.\label{eq:avg_zeta_t+1_inexact}
\end{align}
\end{subequations}

As usual, we define the consensus error vector as $\mathbf{\Omega}^{(t)} = (\normbig{u^{(t)}}_\infty,\normbig{v^{(t)}}_\infty)^\top$, where $u^{(t)},v^{(t)}\in \R^{|\S||\A|}$ are given by  
\begin{align}
    u^{(t)}(s,a)\coloneqq & \norm{\log\vxi^{(t)}(s,a)-\log\overline\xi^{(t)}(s,a)\vone_N}_2\,,\label{eq:u_0_inexact}\\
    v^{(t)}(s,a)\coloneqq & \norm{\mT^{(t)}(s,a)-\widehat{q}^{(t)}(s,a)\vone_N}_2\,.\label{eq:v_0_inexact}
\end{align}
The following lemma characterizes the dynamics of the error vector $\mathbf{\Omega}^{(t)}$, perturbed by additional approximation error.
\begin{lm}
    \label{lm:linear_sys_inexact}
    The updates of inexact FedNPG satisfy
    \begin{equation}\label{eq:matrix_0_inexact}
        \mathbf{\Omega}^{(t+1)}
        \leq
        \mB(\eta)
        \mathbf{\Omega}^{(t)}
        +
        \vd(\eta)
        +
        \underbrace{\begin{pmatrix}
            0\\
            \sqrt{N}\sigma\left(\frac{(1+\gamma)\gamma \eta \sqrt{N}}{(1-\gamma)^3}+2\right)
        \end{pmatrix}\normbig{\ve}_\infty}_{=: \vc(\eta)}\,.
    \end{equation}
    In addition, we have
    \begin{equation}\label{eq:A0_inexact}
        \phi^{(t+1)}(\eta)\leq \phi^{(t)}(\eta)+\frac{2(1+\gamma)\gamma}{(1-\gamma)^4}\eta\norm{u^{(t)}}_\infty+\frac{2\eta}{(1-\gamma)^2}\norm{\ve}_\infty-\eta\left(V^{\star}(\rho)-\overline{V}^{(t)}(\rho)\right)\,,
    \end{equation}    
    where $  \phi^{(t)}(\eta)$ is defined in \eqref{eq:def_phit}.

    Moreover, when $\eta\leq\eta_1$, we have
\begin{equation}\label{eq:consensus_error_0_inexact}
\forall n\in[N]:\quad \norm{\log\pi_n^{(t)}-\log\overline\pi^{(t)}}_\infty\leq2\left(\frac{3}{8}\sigma+\frac{5}{8}\right)^t\norm{\mathbf{\Omega}^{(0)}}_2+\frac{32}{3(1-\sigma)}\left(\frac{N\sigma}{(1-\gamma)^4}\eta+\sqrt{N}\sigma\left(\frac{\eta \sqrt{N}}{(1-\gamma)^3}+1\right)\norm{\ve}_\infty\right)\,.
\end{equation}
\end{lm}
\begin{proof}
    See Appendix~\ref{sec:pf_lm_linear_sys_inexact}.
\end{proof}
Similar to \eqref{eq:ineq_lyap}, we can recursively bound $\Phi^{(t)}(\eta)$ (defined in \eqref{eq:Lyapunov}) as
\begin{align}
    \Phi^{(t+1)}(\eta)&=\phi^{(t+1)}(\eta)+\vp(\eta)^\top \mathbf{\Omega}^{(t+1)}\notag\\
 &   \overset{\eqref{eq:A0_inexact}}{\leq}  \phi^{(t)}(\eta)+\frac{2(1+\gamma)\gamma}{(1-\gamma)^4}\eta\norm{u^{(t)}}_\infty+\frac{2\eta}{(1-\gamma)^2}\norm{\ve}_\infty-\eta\left(V^{\star}(\rho)-\overline{V}^{(t)}(\rho)\right) \nonumber \\
    &\qquad\qquad   +\vp(\eta)^\top\left(\mB(\eta)\mathbf{\Omega}^{(t)}+\vd(\eta)+\vc(\eta)\right) \nonumber \\
    &= \Phi^{(t)}(\eta)+\underbrace{\left[\vp(\eta)^\top\left(\mB(\eta)-\mI \right)+\left(\frac{2(1+\gamma)\gamma}{(1-\gamma)^4}\eta, 0\right)\right]}_{ = (0,0)~\sf{via~}\eqref{eq:magic} }\mathbf{\Omega}^{(t)}-\eta\left(V^{\star}(\rho)-\overline{V}^{(t)}(\rho)\right)\notag\\
    &\qquad\qquad +p_2(\eta)\frac{(1+\gamma)\gamma N\sigma}{(1-\gamma)^4}\eta + \left[p_2(\eta)\sqrt{N}\sigma\left(\frac{(1+\gamma)\gamma \eta \sqrt{N}}{(1-\gamma)^3}+2\right)+\frac{2\eta}{(1-\gamma)^2}\right]\norm{\ve}_\infty\,. \label{eq:ineq_lyap_inexact}
\end{align}
From the above expression we know that
$$V^{\star}(\rho)-\overline{V}^{(t)}(\rho)\leq\frac{\Phi^{(t)}(\eta)-\Phi^{(t+1)}(\eta)}{\eta}+p_2(\eta)\frac{(1+\gamma)\gamma N\sigma}{(1-\gamma)^4}+\left[p_2(\eta)\sqrt{N}\sigma\left(\frac{(1+\gamma)\gamma\sqrt{N}}{(1-\gamma)^3}+\frac{2}{\eta}\right)+\frac{2}{(1-\gamma)^2}\right]\norm{\ve}_\infty\,,$$
which gives
\begin{align}
    \frac{1}{T}\sum_{t=0}^{T-1}\left(V^{\star}(\rho)-\overline{V}^{(t)}(\rho)\right)&\leq \frac{\Phi^{(0)}(\eta)-\Phi^{(T)}(\eta)}{\eta T}+p_2(\eta)\frac{(1+\gamma)\gamma N\sigma}{(1-\gamma)^4}\notag\\
    &\qquad +\left[p_2(\eta)\sqrt{N}\sigma\left(\frac{(1+\gamma)\gamma\sqrt{N}}{(1-\gamma)^3}+\frac{2}{\eta}\right)+\frac{2}{(1-\gamma)^2}\right]\norm{\ve}_\infty\,\label{eq:ub_to_be_refined_inexact}
\end{align}
via telescoping. Combining the above expression with \eqref{eq:ub_term1}, \eqref{eq:lb_denominator} and \eqref{eq:ub_term2}, we have
\begin{align}
    \frac{1}{T}\sum_{t=0}^{T-1}\left(V^{\star}(\rho)-\overline{V}^{(t)}(\rho)\right) 
    &\leq \frac{\log |\A|}{T\eta}+\frac{V^\star (d_\rho^{\pi^\star})}{T(1-\gamma)}+\frac{8(1+\gamma)^2\gamma^2 N\sigma}{(1-\gamma)^9(1-\sigma)^2}\eta^2\notag\\
    &\qquad +\left[\frac{8(1+\gamma)\gamma}{(1-\gamma)^5(1-\sigma)^2}\sqrt{N}\sigma\eta\left(\frac{(1+\gamma)\gamma\eta\sqrt{N}}{(1-\gamma)^3}+2\right)+\frac{2}{(1-\gamma)^2}\right]\norm{\ve}_\infty\,,\label{eq:ub_inexact}
\end{align}
which establishes \eqref{eq:convergence_rate_inexact_0}. 

\section{Convergence analysis of FedNAC}\label{sec_app:extension}
Let $\pi^\star$ be 
an optimal policy and does not need to belong to the log-linear policy class. Fix a state distribution $\rho\in\Delta(\S)$ and a state-action distribution $\nu$. To simplify the notation, we denote $d^{\pi^\star}_\rho$ as $d_\star$, $d^{f_{\bar\vxi^{(t)}}}$ as $d^{(t)}$,  $\tilde{d}_n^{(t)}$ as $\tilde{d}_\nu^{f_{\xi_n^{(t)}}}$, and define $d_n^{(t)}$ and $\bar d_n^{(t)}$ analogously. We also let $Q_n^{(t)}$ denote $Q_n^{\xi_n^{(t)}}$.

Define
\begin{equation}\label{eq:vartheta_rho}
    \vartheta_\rho\coloneqq\frac{1}{1-\gamma}\norm{\frac{d_\star}{\rho}}_\infty\geq \frac{1}{1-\gamma}
\end{equation}
and assume $\vartheta_\rho<\infty$.

We also introduce a weighted KL divergence given by
\begin{equation}\label{eq:weighted_KL}
    D_\star^{(t)}\coloneqq\E_{s\sim d_\star}\left[\KL{\pi^\star(\cdot|s)} {\pi^{(t)}(\cdot|s)}\right]\,,
\end{equation}
where $\KL{\cdot}{\cdot}:\R^{|\A|}\times\R^{|\A|}\rightarrow\R$ is the Kullback-Leibler (KL) divergence:
\begin{equation}\label{eq:KL}
    \forall f,g\in \R^{|\A|}:\quad \KL{f}{g}\coloneqq\sum_{a\in\A}f(a)\log\left(\frac{f(a)}{g(a)}\right)\,.
\end{equation}

Given a state distribution $\rho$ and an optimal policy $\pi^\star$, we define a state-action measure $\tilde{d}^\star$ as
\begin{equation}\label{eq:tilde_d_star}
    \tilde{d}^\star(s,a)\coloneqq d_\star(s)\cdot \text{Unif}_\A(a)=\frac{d_\star(s)}{|\A|}.
\end{equation}

The following theorem guarantees that for any fixed policy $\pi$ and state-action distribution $\nu\in\Delta(\S\times\A)$, the Q-Sampler algorithm (cf.~Algorithm~\ref{alg:sampler}) samples $(s,a)$ from $\tilde d_\nu^\pi$ and gives an unbiased estimate $\widehat Q^\pi(s,a)$ of $Q^\pi(s,a)$, whose proof can be found in~\citet[Lemma~4]{yuan2022linear}.
\begin{lm}[Lemma~4 in~\citet{yuan2022linear}]\label{thm:sampler}
Consider the output $(s_h,a_h)$ and $\widehat Q^\pi(s_h,a_h)$ of Algorithm~\ref{alg:sampler}. It follows that
\begin{align*}
    \E[h+1]&=\frac{1}{1-\gamma}\,,\\
    P(s_h=s,a_h=a)&=\tilde d^\pi_\nu(s,a)\,,\\
    \E\left[\widehat Q^\pi(s_h,a_h)|s_h,a_h\right]&=Q^\pi(s_h,a_h)\,.
\end{align*}
\end{lm}

To present the convergence results of~\alg, we further introduce the following notation, where $t\in\NN$ represents the iteration step in \alg:
\begin{subequations}
\begin{align}
    \hat{\vw}^{(t)}&\coloneqq \frac{1}{N}\sum_{n=1}^{N}\vw_n^{(t)}, \label{eq:hat_w}\\
    \bar\vxi^{(t)}&\coloneqq \frac{1}{N}\sum_{n=1}^N \vxi_n^{(t)} ,\label{eq:bar^{(t)}heta}\\
    \bar f^{(t)} &\coloneqq f_{\bar\xi^{(t)}},\label{eq:bar_pi}\\
    f_n^{(t)}&\coloneqq f_{\xi_n^{(t)}},\label{eq:qi_n}\\
    \vw_{\star,n}^{(t)}&\in\arg\min_\vw\ell \left(\vw, Q_n^{(t)},\tilde{d}_n^{(t)}\right),\label{eq:w_star_n}\\
    \hat \vw_\star^{(t)} &\coloneqq \frac{1}{N}\sum_{n=1}^N \vw_{\star,n}^{(t)}.\label{eq:bar_w_star}
\end{align}
\end{subequations}
For convenience of narration, we introduce the following bounded statistical error assumption.
\begin{asmp}[Bounded statistical error]\label{asmp:eps_stat}
    For all $n\in[N]$, there exists $\varepsilon_{\text{stat}}^n>0$ such that for all $t\in\NN$ in Algorithm~\ref{alg:actor_critic}, we have
    \begin{equation}\label{eq:eps_stat}
        \E\left[\ell \left(\vw_n^{(t)}, Q_n^{(t)}, \tilde d_n^{(t)}\right)-\ell \left(\vw^{(t)}_{\star,n}, Q_n^{(t)}, \tilde d_n^{(t)}\right)\right]\leq \staterror^n.
    \end{equation}
\end{asmp}


When solving the regression problem with sampling based approaches, we can expect $\staterror^n=\mathcal{O}(1/K)$, where $K$ is the iteration number of Algorithm~\ref{alg:critic}.
\begin{thm}[Convergence rate of  Critic (Algorithm~\ref{alg:critic})]\label{thm:rate_critic}
    For Algorithm~\ref{alg:critic}, let $\vw_0=\vzero$ and $\beta=\frac{1}{2C_\phi}$. Then under Assumption~\ref{asmp:psd}, we have
    \begin{equation}\label{eq:rate_critic}
        \E\left[\ell \left(\vw_
        {\text{out}}, Q_{\xi}, \tilde d_\xi\right)\right]-\ell \left(\vw^\star, Q_{\xi}, \tilde d_\xi\right)\leq \frac{4}{K}\left(\frac{\sqrt{2p}}{1-\gamma}\left(\frac{C_\phi^2}{\mu (1-\gamma)}+1\right)+\frac{C_\phi^2}{\mu (1-\gamma)^2}\right)^2,
    \end{equation}
    where $\vw^\star\in\arg\min_{\vw}\ell \left(\vw, Q^\xi,\tilde{d}_\xi\right)$.
\end{thm}
The proof of Theorem~\ref{thm:rate_critic} is postponed to Appendix~\ref{sec_app:proof_rate_critic}.

The following lemma provide a (very pessimistic) upper bound of $C_\nu$ in Assumption~\ref{asmp:bound_transfer}.
\begin{lm}[Upper bound of $C_\nu$]\label{lm:ub_C_nu}
    If $\nu(s,a)>0$ for all state-action pairs $(s,a)\in\S\times\A$, then we have 
    $$C_\nu\leq \frac{1}{(1-\gamma)^2\nu^2_{\min}}.$$
\end{lm}
\begin{proof}
    We only need to note that
    $$\sqrt{\E_{(s.a)\sim\tilde d^{(t)}}\left[\left(\frac{h^{(t)}(s,a)}{\tilde d_n^{(t)}(s,a)}\right)^2\right]}\leq \max_{(s,a)\in\S\times\A}\frac{h^{(t)}(s,a)}{\tilde d_n^{(t)}(s,a)}\leq \frac{1}{(1-\gamma)\nu_{\min}}\,,$$
    where the last inequality follows from~\eqref{eq:d_facts}.
\end{proof}

We give some key lemmas which will be used in our proof of Theorem~\ref{thm:FAC_convergence_informal}.

\begin{lm}[consensus properties]\label{lm:consensus}
    For all $t\in\NN$, we have
    \begin{align}
        \Bar{\vxi}^{(t+1)}&=\Bar{\vxi}^{(t)}+\alpha\hat{\vw}^{(t)},\label{eq:update_theta_mean}\\
        \frac{1}{N}\vone^\top\vh^{(t)}&=\frac{1}{N}\sum_{n=1}^N\vh_n^{(t)}=\hat{\vw}^{(t)}.\label{eq:qroperty_GT}
    \end{align}
\end{lm}
\begin{proof}
    \eqref{eq:qroperty_GT} could be obtained directly by using mathematical induction and update rule \eqref{eq:GT} (note that $\frac{1}{N}\vone^\top\vh^{(-1)}=\hat{\vw}^{(-1)}=\vzero$, see line~2 of Algorithm~\ref{alg:actor_critic}), and \eqref{eq:update_theta_mean} could be obtained by averaging both sides of \eqref{eq:actor_update} and using \eqref{eq:qroperty_GT}.
\end{proof}

\begin{lm}[Young's inequalities]
Let $\{\vx_1,\cdots,\vx_m\}$ be a set of $m$ vectors in $\R^l$. Then for any $\zeta>0$, we have
\begin{align}
    \norm{\vx_i+\vx_j}_2^2&\leq (1+\zeta)\norm{\vx_i}_2^2+(1+1/\zeta)\norm{\vx_j}_2^2,\label{eq:young_2}\\
    \norm{\sum_{i=1}^m \vx_i}_2^2&\leq m\sum_{i=1}^m \norm{\vx_i}_2^2.\label{eq:young_mul}
\end{align}
\end{lm}

\begin{lm}[Lipschitzness of $Q$-function with function approximation]\label{lm:L-Q}
    Assume that $r(s,a)\in[0,1], \forall (s,a)\in\S\times\A$. For any $\vxi$, $\vxi'\in\R^p$, we have 
    \begin{equation}\label{eq:Lip_Q}
        \forall (s,a)\in\S\times\A:\quad |Q^{f_{\xi'}}(s,a)-Q^{f_\xi}(s,a)|\leq \underbrace{\frac{2C_\phi\gamma(1+\gamma)}{(1-\gamma)^2}}_{\coloneqq L_Q}\norm{\vxi'-\vxi}_2\,.
    \end{equation}
\end{lm}
\begin{proof}
    See Appendix~\ref{sec_app:pf_lm_L_Q}.
\end{proof}

For each iteration step $t$ in Algorithm~\ref{alg:actor_critic}, we let $\bar\vxi^{(t)}\coloneqq \frac{1}{N}\sum_{n=1}^N\vxi^{(t)}_n=\frac{1}{N}\vxi^{(t)\top}\vone_N$. We define
\begin{align}
    \Omega_1^{(t)}&\coloneqq\E\norm{\vxi^{(t)}-\vone_N\Bar{\vxi}^{(t)\top}}_\F^2,\label{eq:Omega1}\\
    \Omega_2^{(t)}&\coloneqq\E\norm{\vh^{(t)}-\vone_N\hat{\vw}^{(t)\top}}_\F^2,\label{eq:Omega2}
\end{align}

We let 
\begin{align}
    \meanstaterror\coloneqq& \frac{1}{N}\sum_{n=1}^N\staterror^n\,,\label{eq:mean_stat_error}\\
    \meanapproxerror\coloneqq& \frac{1}{N}\sum_{n=1}^N\approxerror^n\,,\label{eq:mean_approx_error}
\end{align}
and define $\delta^{(t)}\coloneqq V^\star-\bar V^{(t)}(\rho)$, where $\bar V^{(t)}$ is shorthand for $V^{\bar f^{(t)}}$. We give the following performance improvement lemma.
\begin{lm}[Performance improvement of \alg]\label{lm:performance_improvement_fednac}
    Fix a state distribution $\rho$, then we have
    \begin{align}\label{eq:improve_fnac}
        \vartheta_\rho \delta^{(t+1)} + \frac{D_\star^{(t+1)}}{(1-\gamma)\alpha}&\leq\vartheta_\rho\delta^{(t)} +\frac{D_\star^{(t)}}{(1-\gamma)\alpha}-\delta^{(t)}\notag\\
        &+\frac{2\sqrt{C_\nu}(\vartheta_\rho +1)}{1-\gamma}\left(\sqrt{\bar\varepsilon_{stat}}+\sqrt{2\left(\bar\varepsilon_{\text{approx}}+\frac{L_Q^2}{N}\norm{\vxi^{(t)}-\vone_N\bar\vxi^{(t)\top}}_\F^2\right)}\right).
    \end{align}
\end{lm}
\begin{proof}
    See Appendix~\ref{sec_app:pf_lm_improve_fednac}.
\end{proof}

\begin{lm}[linear system]\label{lm:linear_sys_nac}
   For any $t\in\NN$, we let $\mathbf{\Omega}^{(t)}=(\Omega_1^{(t)},\Omega_2^{(t)})^\top$.
    Then for any $\zeta>0$, we have
    \begin{equation}\label{eq:matrix_fnac}
        \mathbf{\Omega}^{(t+1)}\leq \mC
        \mathbf{\Omega}^{(t)}+\vs,
    \end{equation}
    where 
    \begin{equation}\label{eq:C_matrix}
        \mC=(c_{ij})=
        \begin{pmatrix}
    (1+\zeta)\sigma^2 & \alpha^2(1+1/\zeta)\sigma^2\\
    (1+1/\zeta)\frac{96\sigma^2 L_Q^2}{(1-\gamma)\mu} & \sigma^2\left(1+\zeta+(1+1/\zeta)\frac{24L_Q^2\alpha^2}{(1-\gamma)\mu}\right)
\end{pmatrix},
    \end{equation}
    and
    \begin{equation}\label{eq:s}
        \vs=\begin{pmatrix}
            s_1\\
            s_2
        \end{pmatrix}
        =\begin{pmatrix}
            0\\
            (1+1/\zeta)\frac{6\sigma^2}{(1-\gamma)\mu}\left(N(\bar\varepsilon_{\text{stat}}+C_\nu \bar\varepsilon_{\text{approx}})+4L_Q^2\left(\frac{\alpha^2N\bar\varepsilon_{\text{stat}}}{(1-\gamma)\mu}+\frac{\alpha^2NC_\phi^2}{\mu^2(1-\gamma)^2}\right)\right)
        \end{pmatrix}.
    \end{equation}
\end{lm}
\begin{proof}
    See Appendix~\ref{sec_app:proof_lm_linear_sys}.
\end{proof}

%
Now we are ready to give the formal version of Theorem~\ref{thm:FAC_convergence_informal} and its proof.

\begin{thm}[Convergence rate of FedNAC (formal)]\label{thm:FAC_convergence_formal}
    Let $\vxi_1^{(0)}=\cdots=\vxi_N^{(0)}$ in \alg~(Algorithm~\ref{alg:actor_critic}), let the $\vw^{(0)}=\vzero$  and the critic stepsize $\beta=\frac{1}{2C_\phi}$ in Algorithm~\ref{alg:critic}. Then under Assumptions~\ref{asmp:mixing matrix}, \ref{asmp:psd}, \ref{asmp:eps_approx} and \ref{asmp:bound_transfer}, when the actor stepsize satisfies
    \begin{equation}\label{eq:actor_stepsize_ub}
        \alpha\leq \alpha_1\coloneqq\frac{(1-\sigma^2)^3\sqrt{(1-\gamma)\mu}}{768\sqrt{6}\sigma L_Q}\,,
    \end{equation}
    where $L_Q$ is defined in Lemma~\ref{lm:L-Q}, we have
    \begin{align}
    &V^{\star}(\rho)-\frac{1}{T}\sum_{t=0}^{T-1}\E\left[\bar V^{(t)}(\rho)\right]\notag\\
    &\leq \frac{D_\star^{(0)}+\alpha\vartheta_\rho}{T(1-\gamma)\alpha}+\frac{1}{T}\cdot \frac{512\sqrt{6}C_\phi\sqrt{C_\nu}(\vartheta_\rho+1)\sigma\alpha}{(1-\sigma^2)^{3/2}(1-\gamma)^3\sqrt{N}}\sqrt{\Omega_2^{(0)}}\notag\\
    &\quad+\left[\frac{2\sqrt{C_\nu}(\vartheta_\rho +1)}{1-\gamma}+\sqrt{1+\frac{64C_\phi^2 \alpha^2}{(1-\gamma)^5\mu}}\cdot \frac{3072\sqrt{3}C_\phi\sqrt{C_\nu}(\vartheta_\rho+1)\sigma^2\alpha}{(1-\sigma^2)^{3}(1-\gamma)^{7/2}\sqrt{\mu}}\right]\notag\\
    &\qquad \cdot \frac{2}{(1-\gamma)^2\sqrt{K}}\left((\sqrt{2p}+1)C_\phi^2+\sqrt{2p}\mu(1-\gamma)\right)\notag\\
    &\quad+\left[\frac{2\sqrt{2C_\nu}(\vartheta_\rho +1)}{1-\gamma}+\frac{3072\sqrt{3}C_\phi C_\nu(\vartheta_\rho+1)\sigma^2\alpha}{(1-\sigma^2)^{3}(1-\gamma)^{7/2}\sqrt{\mu}}\right]\sqrt{\bar\varepsilon_{\text{approx}}} +\frac{6144\sqrt{2}\sigma^2 C_\nu(\vartheta_\rho+1)C_\phi^3\alpha^2}{(1-\gamma)^{13/2}\mu^{3/2}(1-\sigma^2)^3}.
    \label{eq:rate_detail}
\end{align}
Moreover, the consensus errors could be upper bounded by
\begin{equation}\label{eq:consensus_error_fnac}
    \E\norm{\vxi^{(t)}-\vone_N\Bar{\vxi}^{(t)\top}}_\F^2\leq \left(\frac{49}{64}\sigma^2+\frac{15}{64}\right)^t\E\norm{\vh^{(0)}-\vone_N\hat{\vw}^{(0)\top}}_\F^2+\frac{64\delta(\alpha,K)}{15(1-\sigma^2)}\,,
\end{equation}
where
\begin{equation}\label{eq:delta_alpha_K}
    \delta(\alpha,K)\coloneqq\frac{18\sigma^2 N}{(1-\sigma^2)(1-\gamma)\mu}\left(\bar\varepsilon_{\text{stat}}+C_\nu \bar\varepsilon_{\text{approx}}\right)+\frac{72\sigma^2L_Q^2 N}{(1-\gamma)^3\mu^3(1-\sigma^2)}\left((1-\gamma)\mu\bar\varepsilon_{\text{stat}}+C_\phi^2\right)\alpha^2\,,
\end{equation}
and 
$$\bar\varepsilon_{stat}\leq\frac{4}{(1-\gamma)^4K}\left((\sqrt{2p}+1)C_\phi^2+\sqrt{2p}\mu(1-\gamma)\right)^2\,.$$
\end{thm}

\begin{rmk}[Sample and communication complexity]\label{rmk:complexity}
When $\sigma>0$ and 
$$\alpha=\frac{\sqrt{\mu}(D_\star^{(0)})^{1/3}}{6144^{1/3}2^{1/6}C_\nu^{1/3}(1+\vartheta_\rho)^{1/3}C_\phi}\cdot\frac{(1-\gamma)^{11/6}(1-\sigma^2)}{T^{1/3}\sigma^{2/3}},$$
it follows from Theorem~\ref{thm:FAC_convergence_formal} that
\begin{align}
    &V^{\star}(\rho)-\frac{1}{T}\sum_{t=0}^{T-1}\E\left[\bar V^{(t)}(\rho)\right]\notag\\
    &\leq \frac{3^{1/3}\cdot 2^{29/6}(D_\star^{(0)})^{2/3}C_\nu^{1/3}(1+\vartheta_\rho)^{1/3}C_\phi\sigma^{2/3}}{T^{2/3}(1-\gamma)^{17/6}(1-\sigma^2)\sqrt{\mu}}
    +\frac{\vartheta_\rho}{(1-\gamma)T}+\frac{2^{17/3}3^{1/6}C_\nu^{1/6}(1+\vartheta_\rho)^{2/3}\sigma^{1/3}\sqrt{\mu}(D_\star^{(0)})^{1/3}}{T^{4/3}(1-\sigma^2)^{1/2}(1-\gamma)^{7/6}\sqrt{N}}\notag\\
    &\quad+\left[\frac{2\sqrt{C_\nu}(\vartheta_\rho +1)}{1-\gamma}+\sqrt{1+\frac{(D_\star^{(0)})^{2/3}(1-\sigma^2)^2}{3^{3/2}\cdot 4 C_\nu^{2/3}(1-\gamma)^{4/3}(1+\vartheta_\rho)^{1/3}T^{2/3}\sigma^{4/3}}}\cdot \frac{2^{37/6}\cdot 3^{7/6}C_\nu^{1/6}(\vartheta_\rho+1)^{2/3}\sigma^{4/3}(D_\star^{(0)})^{1/3}}{(1-\sigma^2)^{2}(1-\gamma)^{5/3}T^{1/3}}\right]\notag\\
    &\qquad \cdot \frac{2}{(1-\gamma)^2\sqrt{K}}\left((\sqrt{2p}+1)C_\phi^2+\sqrt{2p}\mu(1-\gamma)\right)\notag\\
    &\quad+\left[\frac{2\sqrt{2C_\nu}(\vartheta_\rho +1)}{1-\gamma}+\frac{2^{37/6}\cdot 3^{7/6}C_\nu^{1/6}(\vartheta_\rho+1)^{2/3}\sigma^{4/3}(D_\star^{(0)})^{1/3}}{(1-\sigma^2)^{2}(1-\gamma)^{5/3}T^{1/3}}\right]\sqrt{\bar\varepsilon_{\text{approx}}}\,.
    \label{eq:opt_rate_detail_fnac}
\end{align}

Consequently, we need 
$$T\gtrsim\left\{\frac{\sigma}{\varepsilon^{3/2}(1-\gamma)^{17/4}(1-\sigma^2)^{3/2}},\frac{1}{\varepsilon(1-\gamma)},\frac{\sigma^{1/4}}{\varepsilon^{3/4}(1-\sigma^2)^{3/8}(1-\gamma)^{7/8}N^{3/8}},\frac{\sigma^4}{(1-\gamma)^2(1-\gamma^2)^6}\right\}$$
and
$$K=\mathcal{O}\left(\frac{1}{(1-\gamma)^6\varepsilon^2}\right)$$
such that $V^{\star}(\rho)-\frac{1}{T}\sum_{t=0}^{T-1}\E\left[\bar V^{(t)}(\rho)\right]\lesssim \varepsilon+\frac{\bar\varepsilon_{\text{approx}}}{1-\gamma}$. In Algorithm~\ref{alg:sampler}, each trajectory has the expected length $1/(1-\gamma)$. Consider only the term where $\varepsilon$ dominates, FedNAC requires $\mathcal{O}\left(\frac{1}{(1-\gamma)^{45/4}\varepsilon^{7/2}(1-\sigma^2)^{3/2}}\right)$ samples for each agent and $\mathcal{O}\left(\frac{1}{\varepsilon^{3/2}(1-\gamma)^{17/4}(1-\sigma^2)^{3/2}}\right)$ rounds of communication.

On the other end, 
when $\sigma=0$, \eqref{eq:rate_detail} becomes:
\begin{align}\label{eq:rate_0}
    V^{\star}(\rho)-\frac{1}{T}\sum_{t=0}^{T-1}\E\left[\bar V^{(t)}(\rho)\right]&\leq\frac{D_\star^{(0)}+\alpha\vartheta_\rho}{T(1-\gamma)\alpha}
    +\frac{4\sqrt{C_\nu}(\vartheta_\rho +1)}{(1-\gamma)^3\sqrt{K}}\left((\sqrt{2p}+1)C_\phi^2+\sqrt{2p}\mu(1-\gamma)\right)\notag\\
    &\quad+\frac{2\sqrt{2C_\nu}(\vartheta_\rho +1)}{1-\gamma}\sqrt{\bar\varepsilon_{\text{approx}}},
\end{align}
Consequently, for any fixed $\alpha>0$, when $\sigma=0$ or close to $0$, with $T=\mathcal{O}\left(\frac{1}{(1-\gamma)\varepsilon}\right)$ and $K=\mathcal{O}\left(\frac{1}{(1-\gamma)^6\varepsilon^2}\right)$, \alg~requires $KT/(1-\gamma)=\mathcal{O}\left(\frac{1}{(1-\gamma)^8\varepsilon^3}\right)$ samples for each agent and $T=\mathcal{O}\left(\frac{1}{(1-\gamma)\varepsilon}\right)$ rounds of communication such that $V^{\star}(\rho)-\frac{1}{T}\sum_{t=0}^{T-1}\E\left[\bar V^{(t)}(\rho)\right]\lesssim \varepsilon+\frac{\bar\varepsilon_{\text{approx}}}{1-\gamma}$. 
 
\end{rmk}

\subsection{Proof of Theorem~\ref{thm:FAC_convergence_formal}}
We suppose Assumptions~\ref{asmp:mixing matrix}, \ref{asmp:eps_stat}, \ref{asmp:psd}, \ref{asmp:eps_approx} and \ref{asmp:bound_transfer} holds. By Lemma~\ref{lm:linear_sys_nac} and nonnegativity of each entry of $\mC$, $\vs$ and $\mathbf{\Omega}^{(t)}$ where $t\in\NN$, it's easy to see that
\begin{equation}\label{eq:matrix_fnac_sqrt}
    \sqrt{\mathbf{\Omega}^{(t+1)}}\leq \sqrt{\mC}\sqrt{\mathbf{\Omega}^{(t)}}+\sqrt{\vs},
\end{equation}
where $\sqrt{\cdot}$ is exerted element-wise.

In addition, taking expectation on both sides of \eqref{eq:improve_fnac} and using the act that
$$\E\left[\sqrt{2\left(\bar\varepsilon_{\text{approx}}+\frac{L_Q^2}{N}\norm{\vxi^{(t)}-\vone_N\bar\vxi^{(t)\top}}_\F^2\right)}\right]\leq\sqrt{2\bar\varepsilon_{\text{approx}}}+\sqrt{\frac{2L_Q^2}{N}\Omega_1^{(t)}},$$
we have
\begin{align}\label{eq:improve_2}
        \vartheta_\rho \E[\delta^{(t+1)}] + \frac{\E [D_\star^{(t+1)}]}{(1-\gamma)\alpha}&\leq\vartheta_\rho\E[\delta^{(t)}] +\frac{\E [D_\star^{(t)}]}{(1-\gamma)\alpha}-\E[\delta^{(t)}]\notag\\
        &+\frac{2\sqrt{C_\nu}(\vartheta_\rho +1)}{1-\gamma}\left(\sqrt{\bar\varepsilon_{\text{stat}}}+\sqrt{2\bar\varepsilon_{\text{approx}}}+\sqrt{\frac{2L_Q^2}{N}\Omega_1^{(t)}}\right).
    \end{align}

We define the Lyapunov function $\Phi^{(t)}$ as follows:
\begin{equation}\label{eq:Lya}
    \Phi^{(t)}\coloneqq \vartheta_\rho\E[\delta^{(t)}] +\frac{\E [D_\star^{(t)}]}{(1-\gamma)\alpha} + \vq^\top\sqrt{\mathbf{\Omega}^{(t)}},
\end{equation}
where
\begin{equation}\label{eq:q}
    \vq=\begin{pmatrix}
        q_1\\
        q_2
    \end{pmatrix}
    =\begin{pmatrix}
        \frac{2L_Q\sqrt{2C_\nu}(\vartheta_\rho +1)}{(1-\gamma)\sqrt{N}}\cdot\frac{1}{1-\sqrt{1+\zeta}\sigma-\sqrt{(1+1/\zeta)c_{21}}\sigma\alpha/(1-\sqrt{c_{22}})}\\
        \frac{2L_Q\sqrt{2C_\nu}(\vartheta_\rho +1)}{(1-\gamma)\sqrt{N}}\cdot\frac{\sqrt{1+1/\zeta}\sigma\alpha}{(1-\sqrt{1+\zeta}\sigma)(1-\sqrt{c_{22}})-\sqrt{(1+1/\zeta)c_{21}}\sigma\alpha}
    \end{pmatrix}.
\end{equation}

It's straightforward to verify that when $\zeta=\frac{1-\sigma^2}{2}$, we have the entries in $\boldsymbol{C}$ (cf.~\eqref{eq:C_matrix}) satisfies
\begin{align}
    c_{11}&<\frac{1+\sigma^2}{2},\label{eq:c11}\\
    c_{12}&\leq\frac{3\sigma^2\alpha^2}{1-\sigma^2}.\label{eq:c12}
\end{align}

Moreover, from $\alpha\leq \frac{\sqrt{(1-\gamma)\mu}(1-\sigma^2)}{12\sqrt{2}\sigma L_Q}$ we deduce
\begin{equation}\label{eq:a22}
    c_{22}\leq \frac{3+\sigma^2}{4},
\end{equation}
which gives
\begin{equation}\label{eq:1-sqrt_a22}
    1-\sqrt{c_{22}}\geq 1-\sqrt{\frac{3+\sigma^2}{4}}\geq \frac{1-\sigma^2}{8},
\end{equation}

Also note that $\alpha\leq \frac{(1-\sigma^2)^3\sqrt{(1-\gamma)\mu}}{768\sqrt{6}\sigma^2 L_Q}$ yields
$$\sqrt{(1+1/\zeta)c_{21}}\sigma\alpha\leq \frac{(1-\sqrt{1+\zeta}\sigma)(1-\sqrt{c_{22}})}{2}.$$
which together with \eqref{eq:1-sqrt_a22} and the fact 
$1-\sqrt{1+\zeta}\sigma\geq \frac{1-\sigma^2}{4}$
indicates $q_1,q_2>0$ and that
\begin{align}
    q_1&\leq\frac{16\sqrt{2}L_Q\sqrt{C_\nu}(\vartheta_\rho+1)}{(1-\sigma^2)(1-\gamma)\sqrt{N}},\\
    q_2&\leq\frac{128\sqrt{6}L_Q\sqrt{C_\nu}(\vartheta_\rho+1)\sigma\alpha}{(1-\sigma^2)^{5/2}(1-\gamma)\sqrt{N}}.\label{eq:q2}
\end{align}


Thus by \eqref{eq:matrix_fnac_sqrt} and \eqref{eq:improve_2} we have
\begin{align}
    \Phi^{(t+1)}&=\vartheta_\rho \E[\delta^{(t+1)}] + \frac{\E [D_\star^{(t+1)}]}{(1-\gamma)\alpha}+\vq^\top\sqrt{\mathbf{\Omega}^{(t+1)}}\notag\\
    &\leq\vartheta_\rho\E[\delta^{(t)}] +\frac{\E [D_\star^{(t)}]}{(1-\gamma)\alpha}-\E[\delta^{(t)}]+\vq^\top\left(\sqrt{\mC}\sqrt{\mathbf{\Omega}^{(t)}}+\sqrt{\vs}\right)\notag\\
    &\qquad+\frac{2\sqrt{C_\nu}(\vartheta_\rho +1)}{1-\gamma}\left(\sqrt{\bar\varepsilon_{\text{stat}}}+\sqrt{2\bar\varepsilon_{\text{approx}}}+\sqrt{\frac{2L_Q^2}{N}\Omega_1^{(t)}}\right)\notag\\
    &=\Phi^{(t)}+\left(\underbrace{\vq^\top (\sqrt{\mC}-\mI)+\left(\frac{2L_Q\sqrt{2C_\nu}(\vartheta_\rho +1)}{(1-\gamma)\sqrt{N}},0\right)}_{=(0,0)}\right)\sqrt{\mathbf{\Omega}^{(t)}}\notag\\
    &\qquad+\frac{2\sqrt{C_\nu}(\vartheta_\rho +1)}{1-\gamma}\left(\sqrt{\bar\varepsilon_{\text{stat}}}+\sqrt{2\bar\varepsilon_{\text{approx}}}\right)+q_2\sqrt{s_2}-\E[\delta^{(t)}],
\end{align}
which gives 
\begin{equation}
   \E[\delta^{(t)}]\leq \Phi^{(t)}- \Phi^{(t+1)} +\frac{2\sqrt{C_\nu}(\vartheta_\rho +1)}{1-\gamma}\left(\sqrt{\bar\varepsilon_{\text{stat}}}+\sqrt{2\bar\varepsilon_{\text{approx}}}\right)+q_2\sqrt{s_2}.
\end{equation}
Summing the above inequality over $t=0,1,\cdots,T-1$ and divide both sides by $T$, we have
\begin{equation}
    \frac{1}{T}\sum_{t=0}^{T-1}\E[\delta^{(t)}]\leq \frac{\Phi^{(0)}-\Phi^{(t)}}{T}+\frac{2\sqrt{C_\nu}(\vartheta_\rho +1)}{1-\gamma}\left(\sqrt{\bar\varepsilon_{\text{stat}}}+\sqrt{2\bar\varepsilon_{\text{approx}}}\right)+q_2\sqrt{s_2}.
\end{equation}

Since
\begin{equation}\label{eq:b2}
    s_2\leq \frac{18\sigma^2 N}{(1-\sigma^2)(1-\gamma)\mu}\left(\bar\varepsilon_{\text{stat}}+C_\nu \bar\varepsilon_{\text{approx}}\right)+\frac{72\sigma^2L_Q^2 N}{(1-\gamma)^3\mu^3(1-\sigma^2)}\left((1-\gamma)\mu\bar\varepsilon_{\text{stat}}+C_\phi^2\right)\alpha^2,
\end{equation}
and
\begin{equation}\label{eq:qhi0}
    \Phi^{(0)}-\Phi^{(t)}\leq\Phi^{(0)}\leq\frac{\vartheta_\rho}{1-\gamma}+\frac{\E[D_\star^{(0)}]}{(1-\gamma)\alpha}+\frac{16\sqrt{2}L_Q\sqrt{C_\nu}(\vartheta_\rho+1)}{(1-\sigma^2)(1-\gamma)\sqrt{N}}\left(\sqrt{\Omega_1^{(0)}}+\frac{8\sqrt{3}\sigma\alpha}{\sqrt{1-\sigma^2}}\sqrt{\Omega_2^{(0)}}\right),
\end{equation}
we have (recall that $L_Q=\frac{2C_\phi\gamma(1+\gamma)}{(1-\gamma)^2}\leq\frac{4C_\phi}{(1-\gamma)^2}$)
\begin{align}
    &V^{\star}(\rho)-\frac{1}{T}\sum_{t=0}^{T-1}\E\left[\bar V^{(t)}(\rho)\right]\notag\\
    &\leq \frac{D_\star^{(0)}+\alpha\vartheta_\rho}{T(1-\gamma)\alpha}+\frac{1}{T}\cdot \frac{64\sqrt{2}C_\phi\sqrt{C_\nu}(\vartheta_\rho+1)}{(1-\sigma^2)(1-\gamma)^3\sqrt{N}}\left(\sqrt{\Omega_1^{(0)}}+\frac{8\sqrt{3}\sigma\alpha}{\sqrt{1-\sigma^2}}\sqrt{\Omega_2^{(0)}}\right)\notag\\
    &\quad+\left[\frac{2\sqrt{C_\nu}(\vartheta_\rho +1)}{1-\gamma}+\sqrt{\frac{18\sigma^2 N}{(1-\sigma^2)(1-\gamma)\mu}+\frac{1152\sigma^2 C_\phi^2 N\alpha^2}{(1-\gamma)^6\mu^2(1-\sigma^2)}}\cdot \frac{512\sqrt{6}C_\phi\sqrt{C_\nu}(\vartheta_\rho+1)\sigma\alpha}{(1-\sigma^2)^{5/2}(1-\gamma)^3\sqrt{N}}\right]\sqrt{\bar\varepsilon_{\text{stat}}}\notag\\
    &\quad+\left[\frac{2\sqrt{2C_\nu}(\vartheta_\rho +1)}{1-\gamma}+\sqrt{\frac{18\sigma^2 NC_\nu}{(1-\sigma^2)(1-\gamma)\mu}}\cdot \frac{512\sqrt{6}C_\phi\sqrt{C_\nu}(\vartheta_\rho+1)\sigma\alpha}{(1-\sigma^2)^{5/2}(1-\gamma)^3\sqrt{N}}\right]\sqrt{\bar\varepsilon_{\text{approx}}}\notag\\
    &\quad+\frac{6144\sqrt{2}\sigma^2\sqrt{C_\nu}(\vartheta_\rho+1)C_\phi^3\alpha^2}{(1-\gamma)^{13/2}\mu^{3/2}(1-\sigma^2)^3}.\label{eq:rate_pre}
\end{align}

By Theorem~\ref{thm:rate_critic} we know that $\sqrt{\bar\varepsilon_{\text{stat}}}$ could be upper bounded as follows:
\begin{equation}\label{eq:bound_l_minimizer_eps_stat}
    \sqrt{\bar\varepsilon_{stat}}\leq\frac{2}{(1-\gamma)^2\sqrt{K}}\left((\sqrt{2p}+1)C_\phi^2+\sqrt{2p}\mu(1-\gamma)\right).
\end{equation}

\eqref{eq:rate_detail} follows from plugging \eqref{eq:bound_l_minimizer_eps_stat} into \eqref{eq:rate_pre} and noting that when $\vxi_1^{(0)}=\cdots=\vxi_N^{(0)}$, $\Omega_1^{(0)}=0$.

\paragraph{Bounding the consensus errors.} Similar to Step~4 in Appendix~\ref{sec:pf_lm_linear_sys}, to bound the consensus error $\norm{\log f_n^{(t)}-\log\bar f^{(t)}}_\infty$ for all $n\in[N]$, we first upper bound the eigenvalue of $\rho(\mC)$---the spectral norm of $\mC$.

The characteristic polynomial of $\mC$ is
\begin{align*}
    f(\lambda)&=(\lambda-c_{11})(\lambda-c_{22})-c_{12}c_{21}\\
    &=\lambda^2-(c_{11}+c_{22})\lambda+c_{11}c_{22}-c_{12}c_{21}\,,
\end{align*}
which gives
\begin{align}
    \rho(\mC)&\leq\frac{c_{11}+c_{22}+\sqrt{(c_{11}+c_{22})^2-4(c_{11}c_{12}-c_{12}c_{21})}}{2}\notag\\
    &=\frac{c_{11}+c_{22}+\sqrt{(c_{22}-c_{11})^2+4c_{12}c_{21}}}{2}\notag\\
    &\leq\frac{c_{11}+c_{22}+c_{22}-c_{11}+2\sqrt{c_{12}c_{21}}}{2}\notag\\
    &=c_{22}+\sqrt{c_{12}c_{21}}\notag\\
    &\leq\frac{3+\sigma^2}{4}+\frac{\sqrt{3}\sigma\alpha}{\sqrt{1-\sigma^2}}\cdot\frac{12\sqrt{2}L_Q\sigma}{\sqrt{1-\sigma^2}(1-\gamma)\mu}\notag\\
    &\leq\frac{3+\sigma^2}{4}+\frac{\sigma(1-\sigma^2)^2}{64}\notag\\
    &\leq\frac{49+15\sigma^2}{64}<1\,,\label{eq:lambda(C)}
\end{align}
where the third inequality uses \eqref{eq:c12}, \eqref{eq:a22}, and the fourth inequality uses \eqref{eq:actor_stepsize_ub}.

Therefore, similar to \eqref{eq:consensus_error_pre}, when $\alpha\leq\alpha_1$, we have
\begin{align}
    \norm{\mathbf{\Omega}^{(t)}}_2
    &\leq\left(\frac{49}{64}\sigma+\frac{15}{64}\right)^t\norm{\mathbf{\Omega}^{(0)}}_2+\frac{64s_2}{15(1-\sigma^2)}\,.\label{eq:consensus_error_pre_fnac}
\end{align}
Combining the above inequality with \eqref{eq:b2}, and \eqref{eq:bound_l_minimizer_eps_stat}, 
we obtain \eqref{eq:consensus_error_fnac}.



\subsection{Proof of Theorem~\ref{thm:rate_critic}}\label{sec_app:proof_rate_critic}
The proof of Theorem~\ref{thm:rate_critic} could be found in Appendix C.5 in~\citet{yuan2022linear}. We present it for completeness. To prove Theorem~\ref{thm:rate_critic}, we need the following Theorem~\ref{thm:bach}.
\begin{thm}[Theorem 1 in \citet{bach2013non}]\label{thm:bach}
    Consider the following assumptions:
    \begin{enumerate}[label=(\roman*)]
        \item \label{(i)} The observations $(\va_k,\vb_k)\in\R^p\times\R^p$ are independent and identically distributed.
        \item \label{(ii)} $\E\left[\norm{\va_k}^2\right]$\footnote{Here $\norm{\cdot}$ could be any norm in $\R^p$.} and $\E\left[\norm{\vb_k}^2\right]$ are finite. The covariance $\E\left[\va_k\va_k^\top\right]$ is invertible.
        \item \label{(iii)} The global minimum of $g(w)=\frac{1}{2}\E\left[\langle \vw,\va_k\rangle^2-2\langle \vw,\vb_k\rangle\right]$ is attained at a certain $\vw^\star\in\R^p$. Let $\Delta_k=\vb_k-\langle \vw^\star,\va_k\rangle \va_k$ denote the residual. We have $\E[\Delta_k]=0$.
        \item \label{(iv)} $\exists R>0$ and $\sigma>0$ such that $\E\left[\Delta_k\Delta_k^\top\right]\leq \sigma^2\E\left[\va_k\va_k^\top\right]$ and $\E\left[\norm{\va_k}^2 \va_k\va_k^\top\right]\leq R^2\E\left[\va_k\va_k^\top\right]$.
        \end{enumerate}
Consider the stochastic gradient recursion
$$w_{k+1}=w_k-\eta\left(\langle w_k,a_k\rangle a_k-b_k\right)$$
started from $w_0\in\R^p$. Let $w_{\text{out}}=\frac{1}{K}\sum_{k=1}^K w_k$. When $\eta=\frac{1}{4R^2}$, we have
\begin{equation}\label{eq:opt_gap}
    \E\left[g(w_{\text{out}})-g(w^\star)\right]\leq\frac{2}{K}(\sigma\sqrt{p}+R\norm{w_0-w^\star})^2.
\end{equation}
\end{thm}

In the proof of Theorem~\ref{thm:rate_critic} we'll show that for Algorithm~\ref{alg:critic}, the assumptions in Theorem~\ref{thm:bach} are all satisfied and thus we can use the result~\eqref{eq:opt_gap}.
\begin{proof}[Proof of Theorem~\ref{thm:rate_critic}]
    We let $a_k$ and $b_k$ in Theorem~\ref{thm:bach} be $\phi(s,a)$ and $\widehat Q_\xi\phi(s,a)$ in Algorithm~\ref{alg:critic}, respectively. And we let $\norm{\cdot}=\norm{\cdot}_2$ in Theorem~\ref{thm:bach}. Since the observations $\left(\phi(s,a),\widehat Q_\xi(s,a)\phi(s,a)\right)\in\R^p\times\R^p$ are i.i.d., \ref{(i)} is satisfied.

  As we assume $\norm{\phi(s,a)}_2\leq C_\phi$, $\E\left[\norm{\phi(s,a)}_2^2\right]$ is finite. From Assumption~\ref{asmp:psd} we know that $\E\left[\phi(s,a)\phi(s,a)^\top\right]$ is invertible. 
    
    Let $H$ be the length of trajectory for estimating $\widehat Q_\xi(s,a)$. Then $\left(\widehat Q_\xi(s,a)\right)^2$ is bounded by
    \begin{align}
        \E\left[\left(\widehat Q_\xi(s,a)\right)^2\right]&=\E_{(s,a)\sim\tilde d_\nu^{\pi_\xi}}\left[\sum_{\tau=0}^\infty Pr(H=\tau)\E\left[\left(\sum_{t=0}^\tau r(s_t,a_t)\right)^2\bigg|H=\tau,s_0=s,a_0=a\right]\right]\notag\\
        &=\E_{(s,a)\sim\tilde d_\nu^{\pi_\xi}}\left[(1-\gamma)\sum_{\tau=0}^\infty \gamma^\tau\E\left[\left(\sum_{t=0}^\tau r(s_t,a_t)\right)^2\bigg|H=\tau,s_0=s,a_0=a\right]\right]\notag\\
        &\leq\E_{(s,a)\sim\tilde d_\nu^{\pi_\xi}}\left[(1-\gamma)\sum_{\tau=0}^\infty \gamma^\tau(\tau+1)^2\right]\leq\frac{2}{(1-\gamma)^2}\,,
    \end{align}
    from which we deduce $\E\left[\norm{\widehat Q_\xi(s,a)\phi(s,a)}^2_2\right]\leq C_\phi^2\E\left[\widehat Q_\xi(s,a)^2\right]$ is bounded. Thus \ref{(ii)} holds.

    Furthermore, we introduce the residual
    \begin{equation}\label{eq:res}
        \Delta\coloneqq\left(\widehat Q_\xi(s,a)-\phi(s,a)^\top w^\star\right)\phi(s,a)\,,
    \end{equation}
    then from Lemma~7 in~\citet{yuan2022linear} we know that $\E[\Delta]=\frac{1}{2}\nabla_w\ell(w,\widehat Q_\xi,d_\nu^{\pi_\xi})=0$, which gives \ref{(iii)}.

    To verify \ref{(iv)}, we let $R=C_\phi$ in Theorem~\ref{thm:bach}, then $\E\left[\norm{\phi(s,a)}_2^2\phi(s,a)\phi(s,a)^\top\right]\leq C_\phi^2 \E\left[\phi(s,a)\phi(s,a)^\top\right]$. Also note that
    \begin{align}
        w^\star&=\left(\E_{(s,a)\sim\tilde d_\nu^{\pi_\xi}}\left[\phi(s,a)\phi(s,a)^\top\right]\right)^\dag \E_{(s,a)\sim\tilde d_\nu^{\pi_\xi}}\left[\widehat Q_\xi(s,a)\phi(s,a)\right]\notag\\
        &\leq\frac{1}{1-\gamma}\left(\E_{(s,a)\sim\nu}\left[\phi(s,a)\phi(s,a)^\top\right]\right)^\dag \E_{(s,a)\sim\tilde d_\nu^{\pi_\xi}}\left[\widehat Q_\xi(s,a)\phi(s,a)\right]\,,
    \end{align}
    from which we deduce
    \begin{equation}\label{eq:bound_l_minimizer}
        \norm{w^\star}_2\leq\frac{B}{\mu(1-\gamma)^2}\,.
    \end{equation}
    \begin{align}
        \E\left[\left(\widehat Q_\xi(s,a)-\phi(s,a)^\top w^\star\right)^2|s,a\right]&=\E\left[\left(\widehat Q_\xi(s,a)\right)^2|s,a\right]-2Q_\xi(s,a)\phi(s,a)^\top w^\star+(\phi(s,a)^\top w^\star)^2\\
        &\leq\frac{2}{(1-\gamma)^2}+\frac{2C_\phi^2}{\mu(1-\gamma)^3}+\frac{C_\phi^4}{\mu^2(1-\gamma)^4}\notag\\
        &\leq\frac{2}{(1-\gamma)^2}\left(\frac{C_\phi^2}{\mu(1-\gamma)}+1\right)^2\,.
    \end{align}
    The above expression implies
    \begin{align}\label{eq:var_delta}
        \E\left[\Delta\Delta^\top\right]
        &=\E_{(s,a)\sim\tilde d_\nu^{\pi_\xi}}\left[\left(\widehat Q_\xi(s,a)-\phi(s,a)^\top w^\star\right)^2 \phi(s,a)\phi(s,a)^\top\big| s,a\right]\notag\\
        &=\E_{(s,a)\sim\tilde d_\nu^{\pi_\xi}}\left[\E\left[\left(\widehat Q_\xi(s,a)-\phi(s,a)^\top w^\star\right)^2\big| s,a\right] \phi(s,a)\phi(s,a)^\top\right]\notag\\
        &\leq\left(\underbrace{\frac{\sqrt{2}}{1-\gamma}\left(\frac{C_\phi^2}{\mu(1-\gamma)}+1\right)}_{\sigma}\right)\E[\phi(s,a)\phi(s,a)^\top]\,.
    \end{align}
    Therefore, \ref{(iv)} is verified.

    Thus by~\eqref{eq:opt_gap}, with stepsize $\beta=\frac{1}{2C_\phi^2}$, initialization $w_0=0$ and $K$ steps of critic updates, we have
    \begin{align*}
         \E\left[\ell\left(w_
        {\text{out}}, \widehat Q_{\xi}, \tilde d_\xi\right)\right]-\ell\left(w^\star, \widehat Q_{\xi}, \tilde d_\xi\right)
        &\leq\frac{4}{K}\left(\sigma\sqrt{p}+C_\phi\norm{w^\star}_2\right)^2\\
        &\leq\frac{4}{K}\left(\frac{\sqrt{2p}}{1-\gamma}\left(\frac{C_\phi^2}{\mu (1-\gamma)}+1\right)+\frac{C_\phi^2}{\mu (1-\gamma)^2}\right)^2,
    \end{align*}
    which gives \eqref{eq:rate_critic}.
\end{proof}

\section{Proof of key lemmas for FedNPG}\label{sec_app:proofs}

\subsection{Proof of Lemma \ref{lm:linear_sys_ent}}
\label{sec:pf_lm_lin_sys_ent}
Before proceeding, we summarize several useful properties of the auxiliary sequences (cf.~\eqref{eq:def_vxi} and \eqref{eq:def_avg_xi}), whose proof is postponed to Appendix~\ref{sec_app:lm_aux}.

\begin{lm}[Properties of auxiliary sequences $\{\overline{\xi}^{(t)}\}$ and $\{\vxi^{(t)}\}$]\label{lm:auxiliary_seq}

$\{\overline{\xi}^{(t)}\}$ and $\{\vxi^{(t)}\}$ have the following properties:
\begin{enumerate}
    \item $\vxi^{(t)}$ can be viewed as an unnormalized version of $\vpi^{(t)}$, i.e.,
    \begin{equation}
        \pi_n^{(t)}(\cdot|s)=\frac{\xi_n^{(t)}(s,\cdot)}{\normbig{\xi_n^{(t)}(s,\cdot)}_1}\,,\,\,\forall n\in[N], \, s\in\S\,.\label{eq:prop_auxiliary}
    \end{equation}
    \item For any $t\geq 0$, $\log \overline{\xi}^{(t)}$ keeps track of the average of $\log \vxi^{(t)}$, i.e.,
    \begin{equation}\label{eq:xi_avg_property}
        \frac{1}{N}\vone_N^\top\log \vxi^{(t)}=\log \overline{\xi}^{(t)}\,.
    \end{equation}
    It follows that
    \begin{align}
        \forall s\in\S,\,t\geq 0:\quad \overline{\pi}^{(t)}(\cdot|s)&=\frac{\overline{\xi}^{(t)}(s,\cdot)}{\normbig{\overline{\xi}^{(t)}(s,\cdot)}_1}.\label{eq:prop_auxiliary_avg}
    \end{align}
\end{enumerate}
\end{lm}

\begin{lm}[\mbox{\cite[Appendix.~A.2]{cen2022fast}}]\label{lm:log_policy}
    For any vector $\theta=[\theta_a]_{a\in\A}\in\R^{|\A|}$, we denote by $\pi_{\theta}\in\R^{|\A|}$ the softmax transform of $\theta$ such that
    \begin{equation}\label{eq:pi_theta}
        \pi_\theta(a)=\frac{\exp(\theta_a)}{\sum_{a'\in\A}\exp(\theta_{a'})}\,,\quad a\in\A\,.
    \end{equation}    
    For any $\theta_1,\theta_2\in\R^{|\A|}$, we have
    \begin{align}
        \big|\log(\norm{\exp(\theta_1)}_1)-\log(\norm{\exp(\theta_2)}_1)\big| & \leq \norm{\theta_1-\theta_2}_\infty\,,\label{eq:log_exp_sum}\\
        \norm{\log\pi_{\theta_1}-\log\pi_{\theta_2}}_\infty & \leq 2\norm{\theta_1-\theta_2}_\infty\,.\label{eq:log_policy}
    \end{align}
\end{lm}

\paragraph{Step 1: bound $u^{(t+1)}(s,a)=\normbig{\log\vxi^{(t+1)}(s,a)-\log\overline\xi^{(t+1)}(s,a)\vone_N}_2$.}
By \eqref{eq:update_auxiliary} and \eqref{eq:update_auxiliary_avg} we have
\begin{align}
    u^{(t+1)}(s,a)&=\normbig{\log\vxi^{(t+1)}(s,a)-\log\overline\xi^{(t+1)}(s,a)\vone_N}_2\notag\\
    &=\norm{\alpha\Big(\mW\log\vxi^{(t)}(s,a)-\log\overline{\xi}^{(t)}(s,a)\vone_N\Big)+(1-\alpha)\Big(\mW\mT^{(t)}(s,a)-\widehat{Q}_\tau^{(t)}(s,a)\vone_N\Big)/\tau}_2\notag\\
    &\leq\sigma\alpha\normbig{\log\vxi^{(t)}(s,a)-\log\overline{\xi}^{(t)}(s,a)\vone_N}_2+\frac{1-\alpha}{\tau}\sigma\normbig{\mT^{(t)}(s,a)-\widehat{Q}_\tau^{(t)}(s,a)\vone_N}_2\notag\\
    &\leq\sigma\alpha\normbig{u^{(t)}}_\infty+\frac{1-\alpha}{\tau}\sigma\normbig{v^{(t)}}_\infty,\label{eq:B}
\end{align}
where the penultimate step results from the averaging property of $\mW$ (property~\eqref{eq:property_W}). Taking maximum over $(s,a) \in \S\times\A$ establishes the bound on $\Omega_1^{(t+1)}$ in \eqref{eq:matrix}.


\paragraph{Step 2: bound $v^{(t+1)}(s,a)=\normbig{\mT^{(t+1)}(s,a)-\widehat{Q}_\tau^{(t+1)}(s,a)\vone_N}_2$.} By \eqref{eq:Q_tracking} we have
\begin{align}
    &\normbig{\mT^{(t+1)}(s,a)-\widehat{Q}_\tau^{(t+1)}(s,a)\vone_N}_2\notag\\
    &=\norm{\mW\left(\mT^{(t)}(s,a)+\mQ_\tau^{(t+1)}(s,a)-\mQ_\tau^{(t)}(s,a)\right)-\widehat{Q}_\tau^{(t+1)}(s,a)\vone_N}_2\notag\\
    &=\norm{\left(\mW\mT^{(t)}(s,a)-\widehat{Q}_\tau^{(t)}(s,a)\vone_N\right)+\mW\left(\mQ_\tau^{(t+1)}(s,a)-\mQ_\tau^{(t)}(s,a)\right)+\left(\widehat{Q}_\tau^{(t)}(s,a)-\widehat{Q}_\tau^{(t+1)}(s,a)\right)\vone_N}_2\notag\\
    &\leq\sigma\normbig{\mT^{(t)}(s,a)-\widehat{Q}_\tau^{(t)}(s,a)\vone_N}_2+\sigma\norm{\left(\mQ_\tau^{(t+1)}(s,a)-\mQ_\tau^{(t)}(s,a)\right)+\left(\widehat{Q}_\tau^{(t)}(s,a)-\widehat{Q}_\tau^{(t+1)}(s,a)\right)\vone_N}_2\notag\\
    &\leq\sigma\normbig{\mT^{(t)}(s,a)-\widehat{Q}_\tau^{(t)}(s,a)\vone_N}_2+\sigma\normbig{\mQ_\tau^{(t+1)}(s,a)-\mQ_\tau^{(t)}(s,a)}_2\,,\label{eq:diff_T_Q}
\end{align}
where the penultimate step uses property~\eqref{eq:property_W}, and the last step is due to
\begin{align*}
    &\norm{\left(\mQ_\tau^{(t+1)}(s,a)-\mQ_\tau^{(t)}(s,a)\right)+\left(\widehat{Q}_\tau^{(t)}(s,a)-\widehat{Q}_\tau^{(t+1)}(s,a)\right)\vone_N}_2^2\\
    &=\normbig{\mQ_\tau^{(t+1)}(s,a)-\mQ_\tau^{(t)}(s,a)}_2^2+N\big(\widehat{Q}_\tau^{(t)}(s,a)-\widehat{Q}_\tau^{(t+1)}(s,a)\big)^2\\
    &\qquad -2\sum_{n=1}^N\left(Q_{\tau,n}^{\pi_n^{(t+1)}}(s,a)-Q_{\tau,n}^{\pi_n^{(t)}}(s,a)\right)\left(\widehat{Q}_\tau^{(t+1)}(s,a)-\widehat{Q}_\tau^{(t)}(s,a)\right)\\
    &=\normbig{\mQ_\tau^{(t+1)}(s,a)-\mQ_\tau^{(t)}(s,a)}_2^2-N\big(\widehat{Q}_\tau^{(t)}(s,a)-\widehat{Q}_\tau^{(t+1)}(s,a)\big)^2\\
    &\leq \normbig{\mQ_\tau^{(t+1)}(s,a)-\mQ_\tau^{(t)}(s,a)}_2^2\,.
\end{align*}




\paragraph{Step 3: bound $\normbig{Q_\tau^\star -\tau \log \overline{\xi}^{(t+1)}}_\infty$.} We decompose the term of interest as
\begin{align*}
    Q_\tau^\star -\tau \log \overline{\xi}^{(t+1)}&=Q_\tau^\star -\tau\alpha \log \overline{\xi}^{(t)} - (1-\alpha)\widehat{Q}_\tau^{(t)}\notag\\
    &=\alpha(Q_\tau^\star -\tau \log \overline{\xi}^{(t)})+(1-\alpha)(Q_\tau^\star-\overline{Q}_\tau^{(t)})+(1-\alpha)(\overline{Q}_\tau^{(t)}-\widehat{Q}_\tau^{(t)}),
\end{align*}
which gives
\begin{equation}\label{eq:pre_A}
    \normbig{Q_\tau^\star -\tau \log \overline{\xi}^{(t+1)}}_\infty\leq \alpha\normbig{Q_\tau^\star -\tau \log \overline{\xi}^{(t)}}_\infty+(1-\alpha)\normbig{Q_\tau^\star-\overline{Q}_\tau^{(t)}}_\infty+(1-\alpha)\normbig{\overline{Q}_\tau^{(t)}-\widehat{Q}_\tau^{(t)}}_\infty\,.
\end{equation}
Note that we can upper bound $\normbig{\overline{Q}_\tau^{(t)}-\widehat{Q}_\tau^{(t)}}_\infty$ by
\begin{align}
    \normbig{\overline{Q}_\tau^{(t)} -\widehat{Q}_\tau^{(t)}}_\infty
    &= \norm{\frac{1}{N}\sum_{n=1}^N Q_{\tau,n}^{\pi_n^{(t)}}-\frac{1}{N}\sum_{n=1}^N Q_{\tau,n}^{\overline\pi^{(t)}}}_\infty \notag\\
    &\leq\frac{1}{N}\sum_{n=1}^N\normbig{Q_{\tau,n}^{\pi_n^{(t)}} - Q_{\tau,n}^{\overline\pi^{(t)}}}_\infty\notag\\
    &\leq \frac{M}{N}\sum_{n=1}^N \normbig{\log\xi_n^{(t)}-\log\overline\xi^{(t)}}_\infty \le M\normbig{u^{(t)}}_\infty.\label{eq:diff_Q_tau_overline}
\end{align}
The last step is due to $\big|\log\xi_n^{(t)}(s,a)-\log\overline\xi^{(t)}(s,a)\big|\leq u^{(t)}(s,a)$, while the penultimate step results from writing
\begin{align*}
    \overline{\pi}^{(t)}(\cdot|s)&=\softmax\left({\log\overline{\xi}^{(t)}(s,\cdot)}\right)\,,\\
    \pi_n^{(t)}(\cdot|s)&=\softmax\left({\log\xi_n^{(t)}(s,\cdot)}\right)\,,
\end{align*}
and applying the following lemma.
\begin{lm}[Lipschitz constant of soft Q-function]\label{lm:difference_soft_Q}
    Assume that $r(s,a)\in[0,1], \forall (s,a)\in\S\times\A$ and $\tau\geq 0$. For any $\theta$, $\theta'\in\R^{|\S||\A|}$, we have 
    \begin{equation}\label{eq:difference_soft_Q}
        \norm{Q_\tau^{\pi_{\theta'}}-Q_\tau^{\pi_\theta}}_\infty\leq \underbrace{\frac{1+\gamma+2\tau(1-\gamma)\log|\A|}{(1-\gamma)^2}\cdot\gamma}_{ =: M}\norm{\theta'-\theta}_\infty\,.
    \end{equation}
\end{lm}
Plugging \eqref{eq:diff_Q_tau_overline} into \eqref{eq:pre_A} gives
\begin{equation}\label{eq:A}
    \normbig{Q_\tau^\star -\tau \log \overline{\xi}^{(t+1)}}_\infty\leq \alpha\normbig{Q_\tau^\star -\tau \log \overline{\xi}^{(t)}}_\infty+(1-\alpha)\normbig{Q_\tau^\star-\overline{Q}_\tau^{(t)}}_\infty+(1-\alpha)M\normbig{u^{(t)}}_\infty\,.
\end{equation}

\paragraph{Step 4: bound $\normbig{\mQ_\tau^{(t+1)}(s,a)-\mQ_\tau^{(t)}(s,a)}_2$.} 

Let $w^{(t)}:\S\times\A\rightarrow\R$ be defined as
\begin{equation}\label{eq:r}
    \forall (s,a)\in\S\times\A:\quad w^{(t)}(s,a)\coloneqq\normbig{\log\vxi^{(t+1)}(s,a)-\log\vxi^{(t)}(s,a)-(1-\alpha)V_\tau^\star(s)\vone_N/\tau}_2\,.
\end{equation}
Again, we treat $w^{(t)}$ as vectors in $\R^{|\S||\A|}$ whenever it is clear from context.
For any $(s,a)\in\S\times\A$ and $n\in[N]$, by Lemma~\ref{lm:difference_soft_Q} it follows that
\begin{align}
 \left|Q_{\tau,n}^{\pi_n^{(t+1)}}(s,a)-Q_{\tau,n}^{\pi_n^{(t)}}(s,a)\right|  
    & \leq  M\max_{s\in\S}\normbig{\log\xi_n^{(t+1)}(s,\cdot)-\log\xi_n^{(t)}(s,\cdot)-(1-\alpha)V_\tau^\star(s)\vone_{|\A|}/\tau}_\infty\notag\\
  &  \leq  M \max_{s\in\S} \max_{a\in\A}w^{(t)}(s,a) \leq M \normbig{w^{(t)}}_\infty\,,\label{eq:Q_to_pi}
\end{align}
and consequently
\begin{equation}\label{eq:bound_by_r_3}
    \normbig{\mQ_\tau^{(t+1)}(s,a)-\mQ_\tau^{(t)}(s,a)}_2\leq M\sqrt{N}\normbig{w^{(t)}}_\infty\,.
\end{equation}
It boils down to control $\norm{w^{(t)}}_\infty$.
To do so, we first note that for each $(s,a)\in\S\times\A$, we have
\begin{align}
    &w^{(t)}(s,a)\notag\\
    &=\normbig{\mW\left(\alpha\log\vxi^{(t)}(s,a)+(1-\alpha)\mT^{(t)}(s,a)/\tau\right)-\log\vxi^{(t)}(s,a)-(1-\alpha)V_\tau^\star(s)\vone_N/\tau}_2\notag\\
    &\overset{(a)}{=}\norm{\alpha(\mW-\mI_N)\left(\log\vxi^{(t)}(s,a)-\log\overline\xi^{(t)}(s,a)\vone_N\right)+(1-\alpha)\left(\mW\mT^{(t)}(s,a)/\tau-\log\vxi^{(t)}(s,a)-V_\tau^\star(s)\vone_N/\tau\right)}_2\notag\\
    &\overset{(b)}{\leq} 2\alpha \normbig{\log\vxi^{(t)}(s,a)-\log\overline\xi^{(t)}(s,a)\vone_N}_2+\frac{1-\alpha}{\tau}\normbig{\mW\mT^{(t)}(s,a)-\tau\log\vxi^{(t)}(s,a)-V_\tau^\star(s)\vone_N}_2\label{eq:bound_w}
\end{align}
where (a) is due to the doubly stochasticity property of $\mW$ and (b) is from the fact $\|\mW-\mI_N\|_2 \leq 2$. We further bound the second term as follows:
\begin{align}
    &\norm{\mW\mT^{(t)}(s,a)-\tau\log\vxi^{(t)}(s,a)-V_\tau^\star(s)\vone_N}_2\notag\\
    &= \norm{\mW\mT^{(t)}(s,a)-\tau\log\vxi^{(t)}(s,a)-\big(Q_\tau^\star(s,a)-\tau\log\pi_\tau^\star(a|s)\big)\vone_N}_2\notag\\
    &\leq \normbig{\mW\mT^{(t)}(s,a)-Q_\tau^\star(s,a)\vone_N}_2+ \tau\normbig{\log\vxi^{(t)}(s,a)-\log\pi_\tau^\star(a|s)\vone_N}_2\notag\\
    &\leq \normbig{\mW\mT^{(t)}(s,a)-\widehat{Q}_\tau(s,a)\vone_N}_2 + \normbig{\widehat{Q}_\tau(s,a)\vone_N - Q_\tau^\star(s,a)\vone_N}_2 \notag\\
    &\qquad+\tau\normbig{\log\vxi^{(t)}(s,a)-\log\overline{\pi}^{(t)}(a|s)\vone_N}_2 + \tau\normbig{\log\overline{\pi}^{(t)}(a|s)\vone_N-\log\pi_\tau^\star(a|s)\vone_N}_2 \notag\\
    &= \sigma\normbig{\mT^{(t)}(s,a)-\widehat{Q}_\tau^{(t)}(s,a)\vone_N}_2+\sqrt{N}\big|\widehat{Q}_\tau^{(t)}(s,a)-Q_\tau^\star(s,a)\big|\notag\\
    &\qquad+\tau\normbig{\log\vxi^{(t)}(s,a)-\log\overline\pi^{(t)}(a|s)\vone_N}_2+\tau\sqrt{N}\big|\log\overline\pi^{(t)}(a|s)-\log\pi_\tau^\star(a|s)\big|\,.\label{eq:bound_intermediate}
\end{align}
Here, the first step results from the following relation established in \cite{nachum2017bridging}:
\begin{equation}\label{eq:relation_soft}
    \forall (s,a)\in\S\times\A:\quad V_\tau^\star(s)=-\tau\log\pi_\tau^\star(a|s)+Q_\tau^\star(s,a)\,,
\end{equation}
which also leads to
\begin{align}
    \normbig{\log\overline\pi^{(t)}-\log\pi_\tau^\star}_\infty\leq \frac{2}{\tau} \normbig{Q_\tau^\star -\tau \log \overline{\xi}^{(t)}}_\infty\,\label{eq:bound_pi_star}
\end{align}
by Lemma~\ref{lm:log_policy}. For the remaining terms in \eqref{eq:bound_intermediate}, we have
\begin{align}
     \big|\widehat{Q}_\tau^{(t)}(s,a)-Q_\tau^\star(s,a)\big|\leq \normbig{\widehat Q_\tau^{(t)}-\overline Q_\tau^{(t)}}_\infty+\normbig{\overline Q_\tau^{(t)}-Q_\tau^\star}_\infty\,,\label{eq:bound_intermediate_1}
\end{align}
and
\begin{align}
    \normbig{\log\vxi^{(t)}(s,a)-\log\overline\pi^{(t)}(a|s)\vone_N}_2
    =&\sqrt{\sum_{n=1}^N\left(\log\xi_n^{(t)}(s,a)-\log\overline\pi^{(t)}(a|s)\right)^2}\notag\\
    \leq&\sqrt{\sum_{n=1}^N 2\normbig{\log\xi_n^{(t)}-\log\overline\xi^{(t)}}_\infty^2}\notag\\
    \leq&\sqrt{\sum_{n=1}^N 2\normbig{u^{(t)}}_\infty^2}
    =\sqrt{2N}\normbig{u^{(t)}}_\infty\,,\label{eq:bound_intermediate_2}
\end{align}
where the first inequality again results from Lemma~\ref{lm:log_policy}. Plugging \eqref{eq:bound_pi_star}, \eqref{eq:bound_intermediate_1}, \eqref{eq:bound_intermediate_2} into \eqref{eq:bound_intermediate} and using the definition of $u^{(t)},v^{(t)}$, we arrive at
\begin{align*}
    w^{(t)}(s,a)&\leq \left(2\alpha+(1-\alpha)\cdot \sqrt{2N}\right) \normbig{u^{(t)}}_\infty+\frac{1-\alpha}{\tau}\normbig{v^{(t)}}_\infty+\frac{1-\alpha}{\tau}\cdot\sqrt{N}\left( \normbig{\widehat Q_\tau^{(t)}-\overline Q_\tau^{(t)}}_\infty+\normbig{\overline Q_\tau^{(t)}-Q_\tau^\star}_\infty\right)\\
    &\qquad+\frac{1-\alpha}{\tau}\cdot 2\sqrt{N} \normbig{Q_\tau^\star -\tau \log \overline{\xi}^{(t)}}_\infty\,.
\end{align*}
Using previous display, we can write \eqref{eq:bound_by_r_3} as
\begin{align}
        &\normbig{\mQ_\tau^{(t+1)}(s,a)-\mQ_\tau^{(t)}(s,a)}_2\notag\\
        &\leq M\sqrt{N}\bigg\{\left(2\alpha+(1-\alpha)\cdot \sqrt{2N}\right) \normbig{u^{(t)}}_\infty+\frac{1-\alpha}{\tau}\sigma\normbig{v^{(t)}}_\infty\notag\\
        &\qquad+\frac{1-\alpha}{\tau}\cdot\sqrt{N}\left(M\normbig{u^{(t)}}_\infty+\normbig{\overline Q_\tau^{(t)}-Q_\tau^\star}_\infty\right)+\frac{1-\alpha}{\tau}\cdot 2\sqrt{N} \normbig{Q_\tau^\star -\tau \log \overline{\xi}^{(t)}}_\infty\bigg\}\,.
        \label{eq:diff_Q_delta_t}
\end{align}
Combining \eqref{eq:diff_T_Q} with the above expression \eqref{eq:diff_Q_delta_t}, we get
\begin{align}\label{eq:C}
    \normbig{v^{(t+1)}}_\infty  & \leq\sigma\left(1+\frac{\eta M\sqrt{N}}{1-\gamma}\sigma\right)\normbig{v^{(t)}}_\infty+\sigma M\sqrt{N}\Bigg\{\left(2\alpha+(1-\alpha)\cdot \sqrt{2N}+\frac{1-\alpha}{\tau}\cdot\sqrt{N}M\right) \normbig{u^{(t)}}_\infty \nonumber \\
        &\qquad +\frac{1-\alpha}{\tau}\cdot\sqrt{N}\normbig{\overline Q_\tau^{(t)}-Q_\tau^\star}_\infty+\frac{1-\alpha}{\tau}\cdot 2\sqrt{N} \normbig{Q_\tau^\star -\tau \log \overline{\xi}^{(t)}}_\infty\Bigg\}\,.
\end{align}

\paragraph{Step 5: bound $\normbig{\overline{Q}_\tau^{(t+1)}-Q_\tau^\star}_\infty$.}
For any state-action pair $(s,a)\in\S\times\A$, we observe that
\begin{align}
    & Q_\tau^\star(s,a)-\overline{Q}_\tau^{(t+1)}(s,a)\notag\\
    &=r(s,a)+\gamma \exlim{s'\sim P(\cdot|s,a)}{V_\tau^\star(s')}
    -\left(r(s,a)+\gamma \exlim{s'\sim P(\cdot|s,a)}{V_\tau^{\overline\pi^{(t+1)}}(s')}\right)\notag\\
    &=\gamma \exlim{s'\sim P(\cdot|s,a)}{\tau\log\left(\norm{\exp\left(\frac{Q_\tau^\star (s',\cdot)}{\tau}\right)}_1\right)}-\gamma\exlim{s'\sim P(\cdot|s,a),\atop a'\sim\overline\pi^{(t+1)}(\cdot|s')}{\overline Q_\tau^{(t+1)}(s',a')-\tau\log\overline\pi^{(t+1)}(a'|s')}\,,\label{eq:diff_Q_star_Q_overline_1}
\end{align}
where the first step invokes the definition of $Q_\tau$ (cf. \eqref{eq:Q_V}), and the second step is due to the following expression of $V_\tau^\star$ established in \cite{nachum2017bridging}:
\begin{equation}\label{eq:V_tau_star}
    V_\tau^\star(s)=\tau\log\left(\norm{\exp\left(\frac{Q_\tau^\star (s,\cdot)}{\tau}\right)}_1\right)\,.
\end{equation}
To continue, note that by \eqref{eq:prop_auxiliary_avg} and \eqref{eq:update_auxiliary_avg} we have
\begin{align}
    \log\overline\pi^{(t+1)}(a|s) & =\log\overline\xi^{(t+1)}(s,a)-\log\left(\normbig{\overline\xi^{(t+1)}(s,\cdot)}_1\right)\notag\\
    & =\alpha\log\overline\xi^{(t)}(s,a)+(1-\alpha)\frac{\widehat Q_\tau^{(t)}(s,a)}{\tau}-\log\left(\normbig{\overline\xi^{(t+1)}(s,\cdot)}_1\right)\,.\label{eq:pi_avg_intermediate}
\end{align}
Plugging \eqref{eq:pi_avg_intermediate} into \eqref{eq:diff_Q_star_Q_overline_1} and \eqref{eq:diff_Q_delta_t} establishes the bounds on 
\begin{align}
    Q_\tau^\star(s,a)-\overline{Q}_\tau^{(t+1)}(s,a)&=\gamma \exlim{s'\sim P(\cdot|s,a)}{\tau\log\left(\norm{\exp\left(\frac{Q_\tau^\star (s',\cdot)}{\tau}\right)}_1\right)-\tau\log\left(\norm{\overline\xi^{(t+1)}(s',\cdot)}_1\right)}\notag\\
    &\qquad-\gamma\mathop{\mathbb{E}}\limits_{s'\sim P(\cdot|s,a),\atop a'\sim\overline\pi^{(t+1)}(\cdot|s')}\Bigg[\overline Q_\tau^{(t+1)}(s',a')-\tau\underbrace{\left(\alpha\log\overline\xi^{(t)}(s',a')+(1-\alpha)\frac{\widehat Q_\tau^{(t)}(s',a')}{\tau}\right)}_{=\log\overline\xi^{(t+1)}(s',a')}\Bigg]\label{eq:diff_Q_star_Q_overline_2}
\end{align}
for any $(s,a)\in\S\times\A$. 
In view of property \eqref{eq:log_exp_sum}, the first term on the right-hand side of \eqref{eq:diff_Q_star_Q_overline_2} can be bounded by
\begin{equation*}
    \tau\log\left(\norm{\exp\left(\frac{Q_\tau^\star (s',\cdot)}{\tau}\right)}_1\right)-\tau\log\left(\normbig{\overline\xi^{(t+1)}(s',\cdot)}_1\right)\leq \normbig{Q_\tau^\star-\tau\log\overline\xi^{(t+1)}}_\infty\,.
\end{equation*}
Plugging the above expression into \eqref{eq:diff_Q_star_Q_overline_2}, we have
\begin{equation*}
    0\leq Q_\tau^\star(s,a)-\overline{Q}_\tau^{(t+1)}(s,a)
    \leq \gamma \normbig{Q_\tau^\star-\tau\log\overline\xi^{(t+1)}}_\infty-\gamma\min_{s,a}\left(\overline Q_\tau^{(t+1)}(s,a)-\tau\log\overline\xi^{(t+1)}(s,a)\right)\,,
\end{equation*}
which gives
\begin{equation}\label{eq:E}
    \normbig{Q_\tau^\star-\overline{Q}_\tau^{(t+1)}}_\infty
    \leq \gamma \normbig{Q_\tau^\star-\tau\log\overline\xi^{(t+1)}}_\infty+\gamma\max\Big\{0, -\min_{s,a}\left(\overline Q_\tau^{(t+1)}(s,a)-\tau\log\overline\xi^{(t+1)}(s,a)\right)\Big\}\,.
\end{equation}
Plugging the above inequality into \eqref{eq:A} and \eqref{eq:C} establishes the bounds on $\Omega_3^{(t+1)}$ and $\Omega_2^{(t+1)}$ in \eqref{eq:matrix}, respectively.
\paragraph{Step 6: bound $-\min_{s,a}\big(\overline Q_\tau^{(t+1)}(s,a)-\tau\log\overline\xi^{(t+1)}(s,a)\big)$.}
We need the following lemma which is adapted from Lemma~1 in \cite{cen2022fast}:
\begin{lm}[Performance improvement of FedNPG with entropy regularization]\label{lm:performance_improvement}
    Suppose $0<\eta\leq (1-\gamma)/\tau$. For any state-action pair $(s_0,a_0)\in\S\times\A$, one has
    \begin{align}
        \overline V_\tau^{(t+1)}(s_0)-\overline V_\tau^{(t)}(s_0)&\geq\frac{1}{\eta}\underset{s\sim d_{s_0}^{\overline\pi^{(t+1)}}}{\E}\left[\alpha \KL{\overline{\pi}^{(t+1)}(\cdot|s_0)}{\overline{\pi}^{(t)}(\cdot|s_0)}+\KL{\overline{\pi}^{(t)}(\cdot|s_0)}{\overline{\pi}^{(t+1)}(\cdot|s_0)}\right]\notag\\
        &\qquad-\frac{2}{1-\gamma}\normbig{\widehat Q_\tau^{(t)}-\overline Q_\tau^{(t)}}_\infty\,,\label{eq:V_improvement}\\
        \overline Q_\tau^{(t+1)}(s_0,a_0)-\overline Q_\tau^{(t)}(s_0,a_0) & \geq -\frac{2\gamma}{1-\gamma}\normbig{\widehat Q_\tau^{(t)}-\overline Q_\tau^{(t)}}_\infty\,.\label{eq:Q_improvement}
    \end{align}
\end{lm}
\begin{proof}
    See Appendix~\ref{sec:pf_lm_perf_imprv}.
\end{proof}

Using \eqref{eq:Q_improvement}, we have
\begin{align}
    &\overline Q_\tau^{(t+1)}(s,a)-\tau\left(\alpha\log\overline\xi^{(t)}(s,a)+(1-\alpha)\frac{\widehat Q_\tau^{(t)}(s,a)}{\tau}\right)\notag\\
    &\geq \overline Q_\tau^{(t)}(s,a)-\tau\left(\alpha\log\overline\xi^{(t)}(s,a)+(1-\alpha)\frac{\widehat Q_\tau^{(t)}(s,a)}{\tau}\right)-\frac{2\gamma}{1-\gamma}\normbig{\widehat Q_\tau^{(t)}-\overline Q_\tau^{(t)}}_\infty\notag\\
    &\geq \alpha\left(\overline Q_\tau^{(t)}(s,a)-\tau\log\overline\xi^{(t)}(s,a)\right)-\frac{2\gamma+\eta\tau}{1-\gamma}\normbig{\widehat Q_\tau^{(t)}-\overline Q_\tau^{(t)}}_\infty\,,\label{eq:geq}
\end{align}
which gives
\begin{align}
    &-\min_{s,a}\left(\overline Q_\tau^{(t+1)}(s,a)-\tau\log\overline\xi^{(t+1)}(s,a)\right)\notag\\
    &\leq -\alpha\min_{s,a}\left(\overline Q_\tau^{(t)}(s,a)-\tau\log\overline\xi^{(t)}(s,a)\right)+\frac{2\gamma+\eta\tau}{1-\gamma}M\normbig{u^{(t)}}_\infty\notag\\
    &\leq \alpha\max\Big\{0, \min_{s,a}\left(\overline Q_\tau^{(t)}(s,a)-\tau\log\overline\xi^{(t)}(s,a)\right)\Big\}+\frac{2\gamma+\eta\tau}{1-\gamma}M\normbig{u^{(t)}}_\infty\,.\label{eq:D}
\end{align}
This establishes the bounds on $\Omega_4^{(t+1)}$ in \eqref{eq:matrix}.



\subsection{Proof of Lemma~\ref{lm:bound_rho}}
\label{sec:pf_lm_bound_rho}


Let $f(\lambda)$ denote the characteristic function. In view of some direct calculations, we obtain 
\begin{equation}\label{eq:f_lambda}
\begin{split}
    f(\lambda)&=(\lambda-\alpha)\bigg\{ \underbrace{(\lambda-\sigma\alpha)(\lambda-\sigma(1+\sigma b\eta))(\lambda-(1-\alpha)\gamma-\alpha)}_{=: f_0(\lambda)}\\
    &\qquad-\frac{\eta\sigma^2}{1-\gamma}\underbrace{\left[S(\lambda-(1-\alpha)\gamma-\alpha)+\gamma cdM\eta+(1-\alpha)(2+\gamma)Mc\eta\right]}_{=: f_1(\lambda)}\bigg\}\\
    &\qquad-\frac{\tau\eta^3\gamma}{(1-\gamma)^2}\cdot 2cdM\sigma^2\,,
\end{split}
\end{equation}
where, for the notation simplicity, we let
\begin{subequations}
\begin{align}
    b&\coloneqq\frac{M\sqrt{N}}{1-\gamma}\,,\label{eq:b}\\
    c&\coloneqq\frac{MN}{1-\gamma}=\sqrt{N}b\,,\label{eq:c}\\
    d&\coloneqq \frac{2\gamma+\eta\tau}{1-\gamma}\,.\label{eq:d}
\end{align}
\end{subequations}
Note that among all these new notation we introduce, $S$, $d$ are dependent of $\eta$. To decouple the dependence, we give their upper bounds as follows
        \begin{align}
            d_0&\coloneqq\frac{1+\gamma}{1-\gamma}\geq d\,,\label{eq:d0}\\
            S_0&\coloneqq M\sqrt{N}\left(2+\sqrt{2N}+\frac{M\sqrt{N}}{\tau}\right)\geq S\,,\label{eq:S0}
        \end{align}
        where \eqref{eq:d0} follows from $\eta\leq (1-\gamma)/\tau$, and \eqref{eq:S0} uses the fact that $\alpha\leq 1$ and $1-\alpha\leq 1$.

    Let 
    \begin{equation}\label{eq:lambda_star}
        \lambda^\star\coloneqq \max\Big\{\frac{3+\sigma}{4}, \frac{1+(1-\alpha)\gamma+\alpha}{2}\Big\}\,.
    \end{equation}

    Since $\mA(\rho)$ is a nonnegative matrix, by Perron-Frobenius Theorem (see \cite{horn2012matrix}, Theorem 8.3.1), $\rho(\eta)$ is an eigenvalue of $\mA(\rho)$. So to verify \eqref{eq:rho_bound}, it suffices to show that $f(\lambda)>0$ for any $\lambda\in[\lambda^\star,\infty)$. To do so, in the following we first show that $f(\lambda^\star)>0$, and then we prove that $f$ is non-decreasing on $[\lambda^\star,\infty)$.

    \begin{itemize}
        \item \textit{Showing $f(\lambda^\star)>0$.}
        We first lower bound $f_0(\lambda^\star)$. Since
        $\lambda^\star\geq \frac{3+\sigma}{4}$,
        we have
        \begin{equation}\label{eq:f0_term2}
            \lambda^\star-\sigma(1+\sigma b\eta)\geq \frac{1-\sigma}{4}\,,
        \end{equation}
        
        and from $\lambda^\star\geq \frac{1+(1-\alpha)\gamma+\alpha}{2}$ we deduce
        \begin{equation}\label{eq:f0_term3}
            \lambda^\star-(1-\alpha)\gamma-\alpha\geq \frac{(1-\gamma)(1-\alpha)}{2}
        \end{equation}
        and
        \begin{equation}\label{eq:1st_term}
            \lambda^\star>\frac{1+\alpha}{2}\,,
        \end{equation}
        which gives
        \begin{equation}\label{eq:f0_term1}
            \lambda^\star-\sigma\alpha\geq \frac{1+\alpha}{2}-\sigma\alpha\,.
        \end{equation}

        Combining \eqref{eq:f0_term1}, \eqref{eq:f0_term2}, \eqref{eq:f0_term3}, we have that
        \begin{equation}\label{eq:f0_bound}
            f_0(\lambda^\star)\geq \frac{1-\sigma}{8}\left(\frac{1+\alpha}{2}-\sigma\alpha \right)\eta\tau\,.
        \end{equation}
        
        To continue, we upper bound $f_1(\lambda^\star)$ as follows. 
    \begin{align}
            f_1(\lambda^\star) & \leq S\tau\eta+\gamma cdM\eta+\frac{2+\gamma}{1-\gamma}cM\tau\eta^2 \nonumber \\
       &  =\eta\left(\tau\left(S+\frac{2+\gamma}{1-\gamma}Mc\eta\right)+\gamma cdM\right)\,.\label{eq:f1_bound}
    \end{align}

        Plugging \eqref{eq:f0_bound},\eqref{eq:f1_bound} into \eqref{eq:f_lambda} and using \eqref{eq:1st_term}, we have
        \begin{align*}
            f(\lambda^\star)
            & > \frac{1-\alpha}{2}\left(f_0(\lambda^\star)-\frac{\eta\sigma^2}{1-\gamma}f_1(\lambda^\star)\right)
            -\frac{\tau\eta^3\gamma}{(1-\gamma)^2}\cdot 2cdM\sigma^2\\
           & \geq\frac{\tau\eta^2}{2(1-\gamma)}\left[\frac{1-\sigma}{8}\tau\left(1-\sigma+(1-\alpha)(\sigma-\frac{1}{2})\right)-\frac{\eta\sigma^2}{1-\gamma}\left(\tau\left(S+\frac{2+\gamma}{1-\gamma}Mc\eta\right)+5\gamma cdM\right)\right]\\
           & =\frac{\tau\eta^2}{2(1-\gamma)}\left[\frac{(1-\sigma)^2}{8}\tau-\frac{\eta}{1-\gamma}\left(S\tau\sigma^2+\frac{2+\gamma}{1-\gamma}Mc\sigma^2 \tau\eta+\tau^2\left(\frac{1}{2}-\sigma^2\right)\cdot\frac{1-\sigma}{8}+5\gamma cdM\sigma^2\right)\right]\\
            &\geq\frac{\tau\eta^2}{2(1-\gamma)}\left[\frac{(1-\sigma)^2}{8}\tau-\frac{\eta}{1-\gamma}\left(S_0\tau\sigma^2+\frac{(1-\sigma)^2}{16}\tau^2+\left(2+\gamma+5\gamma d_0\right)cM\sigma^2\right)\right]\geq 0\,,
        \end{align*}
        where the penultimate inequality uses 
        $\frac{1}{2}-\sigma\leq\frac{1-\sigma}{2}$,
        and the last inequality follows from the definition of $\zeta$ (cf. \eqref{eq: zeta}).
        
        \item \textit{Proving $f$ is non-decreasing on $[\lambda^\star,\infty)$.}
    Note that
    \begin{equation*}
        \eta\leq \zeta\leq \frac{(1-\gamma)(1-\sigma)^2}{8S_0\sigma^2}\,,
    \end{equation*}
    thus we have
    \begin{align*}
        \forall \lambda\geq\lambda^\star:\quad f_0'(\lambda)-\frac{\eta\sigma^2}{1-\gamma}f_1'(\lambda)\geq (\lambda-\sigma\alpha)(\lambda-\sigma(1+\sigma b\eta))-\frac{\eta}{1-\gamma}S\sigma^2\geq 0\,,
    \end{align*}
    which indicates that $f_0-f_1$ is non-decreasing on $[\lambda^\star,\infty)$. Therefore, $f$ is non-decreasing on $[\lambda^\star,\infty)$.
    \end{itemize}

\subsection{Proof of Lemma~\ref{lm:linear_sys_ent_approx}}\label{sec_app:lm_linear_sys_ent_approx}

Note that bounding $u^{(t+1)}(s,a)$ is identical to the proof in Appendix~\ref{sec:pf_lm_lin_sys_ent} and shall be omitted.
The rest of the proof also follows closely that of Lemma~\ref{lm:linear_sys_ent}, and we only highlight the differences due to approximation error for simplicity. 

\paragraph{Step 2: bound $v^{(t+1)}(s,a)=\normbig{\mT^{(t+1)}(s,a)-\widehat{q}_\tau^{(t+1)}(s,a)\vone_N}_2$.} Let $\vq_\tau^{(t)}  \coloneqq \Big( q_{\tau,1}^{\pi_1^{(t)}},\cdots, q_{\tau,N}^{\pi_N^{(t)}}\Big)^\top.$ Similar to \eqref{eq:diff_T_Q} we have
\begin{align}
    &\normbig{\mT^{(t+1)}(s,a)-\widehat{q}_\tau^{(t+1)}(s,a)\vone_N}_2\notag\\
    &\leq\sigma\normbig{\mT^{(t)}(s,a)-\widehat{q}_\tau^{(t)}(s,a)\vone_N}_2+\sigma\normbig{\vq_\tau^{(t+1)}(s,a)-\vq_\tau^{(t)}(s,a)}_2\notag\\
    &\leq \sigma\normbig{\mT^{(t)}(s,a)-\widehat{q}_\tau^{(t)}(s,a)\vone_N}_2+\sigma\normbig{\mQ_\tau^{(t+1)}(s,a)-\mQ_\tau^{(t)}(s,a)}_2+2\sigma\norm{\ve}_2. \label{eq:diff_T_Q_inexact}
\end{align}

\paragraph{Step 3: bound $\normbig{Q_\tau^\star -\tau \log \overline{\xi}^{(t+1)}}_\infty$.}
In the context of inexact updates, \eqref{eq:pre_A} writes
\begin{equation*}
    \normbig{Q_\tau^\star -\tau \log \overline{\xi}^{(t+1)}}_\infty\leq \alpha\normbig{Q_\tau^\star -\tau \log \overline{\xi}^{(t)}}_\infty+(1-\alpha)\normbig{Q_\tau^\star-\overline{Q}_\tau^{(t)}}_\infty+(1-\alpha)\normbig{\overline{Q}_\tau^{(t)}-\widehat{q}_\tau^{(t)}}_\infty\,.
\end{equation*}
For the last term, following a similar argument in \eqref{eq:diff_Q_tau_overline} leads to
\begin{align*}
    \normbig{\overline{Q}_\tau^{(t)} -\widehat{q}_\tau^{(t)}}_\infty
   &  = \norm{\frac{1}{N}\sum_{n=1}^N Q_{\tau,n}^{\pi_n^{(t)}}-\frac{1}{N}\sum_{n=1}^N Q_{\tau,n}^{\overline\pi^{(t)}}}_\infty + \norm{\frac{1}{N}\sum_{n=1}^N \left(Q_{\tau,n}^{\pi_n^{(t)}}-q_{\tau,n}^{\pi_n^{(t)}}\right)}_\infty\notag\\
  &  \leq M\cdot\frac{1}{N}\sum_{n=1}^N\normbig{\log\xi_n^{(t)}-\log\overline\xi^{(t)}}_\infty+\frac{1}{N}\sum_{n=1}^N e_n\notag\\
  &  \leq M\normbig{u^{(t)}}_\infty+\norm{\ve}_\infty\,.
\end{align*}
Combining the above two inequalities, we obtain
\begin{equation}\label{eq:A_inexact}
    \normbig{Q_\tau^\star -\tau \log \overline{\xi}^{(t+1)}}_\infty\leq \alpha\normbig{Q_\tau^\star -\tau \log \overline{\xi}^{(t)}}_\infty+(1-\alpha)\normbig{Q_\tau^\star-\overline{Q}_\tau^{(t)}}_\infty+(1-\alpha)\left(M\normbig{u^{(t)}}_\infty+\normbig{\ve}_\infty\right)\,.
\end{equation}

\paragraph{Step 4: bound $\norm{\mQ_\tau^{(t+1)}(s,a)-\mQ_\tau^{(t)}(s,a)}_2$.} 
We remark that the bound established in \eqref{eq:bound_by_r_3} still holds in the inexact setting, with the same definition for $w^{(t)}$:
\begin{equation}\label{eq:bound_by_r_3_inexact}
    \norm{\mQ_\tau^{(t+1)}(s,a)-\mQ_\tau^{(t)}(s,a)}_2\leq M\sqrt{N}\norm{w^{(t)}}_\infty\,.
\end{equation}
To deal with the approximation error, we rewrite \eqref{eq:bound_intermediate} as
\begin{align}
    &\norm{\mW\mT^{(t)}(s,a)-\tau\log\vxi^{(t)}(s,a)-V_\tau^\star(s)\vone_N}_2\notag\\
    &= \norm{\mW\mT^{(t)}(s,a)-\tau\log\vxi^{(t)}(s,a)-\big(Q_\tau^\star(s,a)-\tau\log\pi_\tau^\star(a|s)\big)\vone_N}_2\notag\\
    &\leq \normbig{\mW\mT^{(t)}(s,a)-Q_\tau^\star(s,a)\vone_N}_2+ \tau\normbig{\log\vxi^{(t)}(s,a)-\log\pi_\tau^\star(a|s)\vone_N}_2\notag\\
    &\leq \normbig{\mW\mT^{(t)}(s,a)-\widehat{q}_\tau(s,a)\vone_N}_2 + \normbig{\widehat{q}_\tau(s,a)\vone_N - Q_\tau^\star(s,a)\vone_N}_2 \notag\\
    &\qquad+\tau\normbig{\log\vxi^{(t)}(s,a)-\log\overline{\pi}^{(t)}(a|s)\vone_N}_2 + \tau\normbig{\log\overline{\pi}^{(t)}(a|s)\vone_N-\log\pi_\tau^\star(a|s)\vone_N}_2 \notag\\
    &\leq \sigma\norm{\mT^{(t)}(s,a)-\widehat{q}_\tau^{(t)}(s,a)\vone_N}_2+\sqrt{N}\big|\widehat{q}_\tau^{(t)}(s,a)-Q_\tau^\star(s,a)\big|\notag\\
    &\qquad+\tau\norm{\log\vxi^{(t)}(s,a)-\log\overline\pi^{(t)}(a|s)\vone}_2+\tau\sqrt{N}\big|\log\overline\pi^{(t)}(a|s)-\log\pi_\tau^\star(a|s)\big|\,,\label{eq:bound_intermediate_inexact}
\end{align}
where the second term can be upper-bounded by
\begin{align}
     \big|\widehat{q}_\tau^{(t)}(s,a)-Q_\tau^\star(s,a)\big|&\leq \normbig{\widehat Q_\tau^{(t)}-\overline Q_\tau^{(t)}}_\infty+\normbig{\overline Q_\tau^{(t)}-Q_\tau^\star}_\infty+\norm{\widehat{q}_\tau^{(t)}(s,a)-\widehat{Q}_\tau^{(t)}(s,a)}_\infty\,\notag\\
     &\leq \normbig{\widehat Q_\tau^{(t)}-\overline Q_\tau^{(t)}}_\infty+\normbig{\overline Q_\tau^{(t)}-Q_\tau^\star}_\infty+\norm{\ve}_\infty.\label{eq:bound_intermediate_1_inexact}
\end{align}
Combining \eqref{eq:bound_intermediate_1_inexact}, \eqref{eq:bound_intermediate_inexact} and the established bounds in \eqref{eq:bound_w}, \eqref{eq:bound_pi_star}, \eqref{eq:bound_intermediate_2} leads to
\begin{align*}
    w^{(t)}(s,a)\leq &\left(2\alpha+(1-\alpha)\cdot \sqrt{2N}\right) \norm{u^{(t)}}_\infty+\frac{1-\alpha}{\tau}\norm{v^{(t)}}_\infty\\
    &+\frac{1-\alpha}{\tau}\cdot\sqrt{N}\left( \norm{\widehat Q_\tau^{(t)}-\overline Q_\tau^{(t)}}_\infty+\norm{\overline Q_\tau^{(t)}-Q_\tau^\star}_\infty+\norm{\ve}_\infty\right)
    +\frac{1-\alpha}{\tau}\cdot 2\sqrt{N} \norm{Q_\tau^\star -\tau \log \overline{\xi}^{(t)}}_\infty\,.
\end{align*}
Combining the above inequality with \eqref{eq:bound_by_r_3_inexact} and \eqref{eq:diff_T_Q_inexact} gives
\begin{equation}\label{eq:C_inexact}
\begin{split}
    \normbig{v^{(t+1)}}_\infty&\leq \sigma\left(1+\frac{\eta M\sqrt{N}}{1-\gamma}\sigma\right)\normbig{v^{(t)}}_\infty+\sigma M\sqrt{N}\Bigg\{\left(2\alpha+(1-\alpha)\cdot \sqrt{2N}+\frac{1-\alpha}{\tau}\cdot\sqrt{N}M\right) \normbig{u^{(t)}}_\infty\\
        &+\frac{1-\alpha}{\tau}\cdot\sqrt{N}\left(\normbig{\overline Q_\tau^{(t)}-Q_\tau^\star}_\infty+\normbig{\ve}_\infty\right)+\frac{1-\alpha}{\tau}\cdot 2\sqrt{N} \normbig{Q_\tau^\star -\tau \log \overline{\xi}^{(t)}}_\infty\Bigg\}+2\sigma\sqrt{N}\norm{\ve}_\infty\,.
\end{split}
\end{equation}

\paragraph{Step 5: bound $\norm{\overline{Q}_\tau^{(t+1)}-Q_\tau^\star}_\infty$.}
    It is straightforward to verify that \eqref{eq:E} applies to the inexact updates as well:
    \begin{equation*}
        \norm{Q_\tau^\star-\overline{Q}_\tau^{(t+1)}}_\infty
        \leq \gamma \norm{Q_\tau^\star-\tau\log\overline\xi^{(t+1)}}_\infty+\gamma\left(-\min_{s,a}\left(\overline Q_\tau^{(t+1)}(s,a)-\tau\log\overline\xi^{(t+1)}(s,a)\right)\right)\,.
    \end{equation*}
    Plugging the above inequality into \eqref{eq:A_inexact} and \eqref{eq:C_inexact} establishes the bounds on $\Omega_3^{(t+1)}$ and $\Omega_2^{(t+1)}$ in \eqref{eq:matrix_inexact}, respectively.
    \paragraph{Step 6: bound $-\min_{s,a}\big(\overline Q_\tau^{(t+1)}(s,a)-\tau\log\overline\xi^{(t+1)}(s,a)\big)$.} We obtain the following lemma by interpreting the approximation error $\ve$ as part of the consensus error $\norm{\widehat Q_\tau^{(t)}-\overline Q_\tau^{(t)}}_\infty$ in Lemma~\ref{lm:performance_improvement}.
    \begin{lm}[inexact version of Lemma~\ref{lm:performance_improvement}]\label{lm:performance_improvement_inexact}
        Suppose $0<\eta\leq (1-\gamma)/\tau$. For any state-action pair $(s_0,a_0)\in\S\times\A$, one has
        \begin{align}
            \overline V_\tau^{(t+1)}(s_0)-\overline V_\tau^{(t)}(s_0) & \geq\frac{1}{\eta}\underset{s\sim d_{s_0}^{\overline\pi^{(t+1)}}}{\E}\left[\alpha \KL{\overline{\pi}^{(t+1)}(\cdot|s_0)}{\overline{\pi}^{(t)}(\cdot|s_0)}+\KL{\overline{\pi}^{(t)}(\cdot|s_0)}{\overline{\pi}^{(t+1)}(\cdot|s_0)}\right]\notag\\
            &\qquad -\frac{2}{1-\gamma}\left(\norm{\widehat Q_\tau^{(t)}-\overline Q_\tau^{(t)}}_\infty+\norm{\ve}_\infty\right)\,,\label{eq:V_improvement_inexact}\\
            \overline Q_\tau^{(t+1)}(s_0,a_0)-\overline Q_\tau^{(t)}(s_0,a_0) & \geq -\frac{2\gamma}{1-\gamma}\left(\norm{\widehat Q_\tau^{(t)}-\overline Q_\tau^{(t)}}_\infty+\norm{\ve}_\infty\right)\,.\label{eq:Q_improvement_inexact}
        \end{align}
    \end{lm}
Using \eqref{eq:Q_improvement_inexact}, we have
    \begin{align}
        &\overline Q_\tau^{(t+1)}(s,a)-\tau\left(\alpha\log\overline\xi^{(t)}(s,a)+(1-\alpha)\frac{\widehat Q_\tau^{(t)}(s,a)}{\tau}\right)\notag\\
        & \geq \overline Q_\tau^{(t)}(s,a)-\tau\left(\alpha\log\overline\xi^{(t)}(s,a)+(1-\alpha)\frac{\widehat Q_\tau^{(t)}(s,a)}{\tau}\right)-\frac{2\gamma}{1-\gamma}\left(\norm{\widehat Q_\tau^{(t)}-\overline Q_\tau^{(t)}}_\infty+\norm{\ve}_\infty\right)\notag\\
       &  \geq\alpha\left(\overline Q_\tau^{(t)}(s,a)-\tau\log\overline\xi^{(t)}(s,a)\right)-\frac{2\gamma+\eta\tau}{1-\gamma}\norm{\widehat Q_\tau^{(t)}-\overline Q_\tau^{(t)}}_\infty-\frac{2\gamma}{1-\gamma}\norm{\ve}_\infty\,,\label{eq:geq_inexact}
    \end{align}
    which gives
    \begin{equation}\label{eq:D_inexact}
    \begin{split}
        &-\min_{s,a}\left(\overline Q_\tau^{(t+1)}(s,a)-\tau\log\overline\xi^{(t+1)}(s,a)\right)\\
       &  \leq-\alpha\min_{s,a}\left(\overline Q_\tau^{(t)}(s,a)-\tau\log\overline\xi^{(t)}(s,a)\right)+\frac{2\gamma+\eta\tau}{1-\gamma}M\norm{u^{(t)}}_\infty+\frac{2\gamma}{1-\gamma}\norm{\ve}_\infty\,.
    \end{split}
    \end{equation}

\subsection{Proof of Lemma \ref{lm:linear_sys}}
\label{sec:pf_lm_linear_sys}

\paragraph{Step 1: bound $u^{(t+1)}(s,a)=\norm{\log\vxi^{(t+1)}(s,a)-\log\overline\xi^{(t+1)}(s,a)\vone_N}_2$.}
Following the same strategy in establishing \eqref{eq:B}, we have
\begin{align}
    &\norm{\log\vxi^{(t+1)}(s,a)-\log\overline\xi^{(t+1)}(s,a)\vone_N}_2\notag\\
    &=\norm{\left(\mW\log\vxi^{(t)}(s,a)-\log\overline{\xi}^{(t)}(s,a)\vone_N\right)+\frac{\eta}{1-\gamma}\left(\mW\mT^{(t)}(s,a)-\widehat{Q}^{(t)}(s,a)\vone_N\right)}_2\notag\\
   & \leq\sigma\norm{\log\vxi^{(t)}(s,a)-\log\overline{\xi}^{(t)}(s,a)\vone_N}_2+\frac{\eta}{1-\gamma}\sigma\norm{\mT^{(t)}(s,a)-\widehat{Q}^{(t)}(s,a)\vone_N}_2\,,\label{eq:bound_u_0}
\end{align}
or equivalently
\begin{equation}\label{eq:B_0}
    \normbig{u^{(t+1)}}_\infty\leq \sigma\normbig{u^{(t)}}_\infty+\frac{\eta}{1-\gamma}\sigma\normbig{v^{(t)}}_\infty\,.
\end{equation}

\paragraph{Step 2: bound $v^{(t+1)}(s,a)=\normbig{\mT^{(t+1)}(s,a)-\widehat{Q}^{(t+1)}(s,a)\vone_N}_2$.} In the same vein of establishing \eqref{eq:diff_T_Q}, we have
\begin{align}
    &\normbig{\mT^{(t+1)}(s,a)-\widehat{Q}^{(t+1)}(s,a)\vone_N}_2\notag\\
    &\leq\sigma\normbig{\mT^{(t)}(s,a)-\widehat{Q}^{(t)}(s,a)\vone_N}_2+\sigma\normbig{\mQ^{(t+1)}(s,a)-\mQ^{(t)}(s,a)}_2\,,\label{eq:diff_T_Q_0}
\end{align}
The term $\norm{\mQ^{(t+1)}(s,a)-\mQ^{(t)}(s,a)}_2$ can be bounded in a similar way in \eqref{eq:bound_by_r_3}:
\begin{equation}\label{eq:bound_by_r_3_0}
    \normbig{\mQ^{(t+1)}(s,a)-\mQ^{(t)}(s,a)}_2\leq \frac{(1+\gamma)\gamma}{(1-\gamma)^2}\sqrt{N}\normbig{w_0^{(t)}}_\infty\,,
\end{equation}
where the coefficient $\frac{(1+\gamma)\gamma}{(1-\gamma)^2}$ comes from $M$ in Lemma~\ref{lm:difference_soft_Q} when $\tau = 0$, and $w_0^{(t)}\in\R^{|\S||\A|}$ is defined as
\begin{equation}\label{eq:w_0}
    \forall (s,a)\in\S\times\A:\quad w_0^{(t)}(s,a)\coloneqq\norm{\log\vxi^{(t+1)}(s,a)-\log\vxi^{(t)}(s,a)-\frac{\eta}{1-\gamma}V^\star(s)\vone_N}_2\,.
\end{equation}
It remains to bound $\normbig{w_0^{(t)}}_\infty$.
Towards this end, we rewrite \eqref{eq:bound_w} as
\begin{align}
    &w_0^{(t)}(s,a)\notag\\
    &=\normbig{\mW\left(\log\vxi^{(t)}(s,a)+\frac{\eta}{1-\gamma}\mT^{(t)}(s,a)\right)-\log\vxi^{(t)}(s,a)-\frac{\eta}{1-\gamma}V^\star(s)\vone_N}_2\notag\\
    &=\norm{(\mW-\mI)\left(\log\vxi^{(t)}(s,a)-\log\overline\xi^{(t)}(s,a)\vone_N\right)+\frac{\eta}{1-\gamma}\left(\mW\mT^{(t)}(s,a)-V^\star(s)\vone_N\right)}_2\notag\\
    &\leq 2 \normbig{\log\vxi^{(t)}(s,a)-\log\overline\xi^{(t)}(s,a)\vone_N}_2+\frac{\eta}{1-\gamma}\normbig{\mW\mT^{(t)}(s,a)-V^\star(s)\vone_N}_2\notag\\
    &\leq 2 \normbig{\log\vxi^{(t)}(s,a)-\log\overline\xi^{(t)}(s,a)\vone_N}_2+\frac{\eta}{1-\gamma}\normbig{\mW\mT^{(t)}(s,a)-\widehat{Q}^{(t)}(s,a)\vone_N}_2+\frac{\eta}{1-\gamma}\cdot\sqrt{N}\big|\widehat{Q}^{(t)}(s,a)-V^\star(s)\big|\,. \label{eq:bound_intermediate_0}
\end{align}
Note that it holds for all $(s,a)\in\S\times\A$: 
\begin{equation*}
    \big|\widehat{Q}^{(t)}(s,a)-V^\star(s)\big|\leq \frac{1}{1-\gamma}
\end{equation*}    
since $\widehat{Q}^{(t)}(s,a)$ and $V^\star(s)$ are both in $[0,1/(1-\gamma)]$. This along with \eqref{eq:bound_intermediate_0} gives
\begin{align*}
    w_0^{(t)}(s,a) &\leq 2\normbig{u^{(t)}}_\infty+\frac{\eta}{1-\gamma}\normbig{v^{(t)}}_\infty+\frac{\eta\sqrt{N}}{(1-\gamma)^2}\,.
\end{align*}

Combining the above inequality with \eqref{eq:bound_by_r_3_0} and \eqref{eq:diff_T_Q_0}, we arrive at
\begin{equation}\label{eq:C_0}
    \begin{split}
        \normbig{v^{(t+1)}}_\infty & \leq\sigma\left(1+\frac{(1+\gamma)\gamma\sqrt{N}\eta}{(1-\gamma)^3}\sigma\right)\normbig{v^{(t)}}_\infty
        +\frac{(1+\gamma)\gamma}{(1-\gamma)^2}\sqrt{N}\sigma\Bigg\{2\normbig{u^{(t)}}_\infty
        +\frac{\eta}{(1-\gamma)^2}\cdot\sqrt{N}\Bigg\}\,.
    \end{split}
\end{equation}

\paragraph{Step 3: establish the descent equation.} 
The following lemma characterizes the improvement in $\phi^{(t)}(\eta)$ for every iteration of Algorithm~\ref{alg:DNPG_exact}, with the proof postponed to Appendix~\ref{sec:pf_lm_perf_imprv_0}.
\begin{lm}[Performance improvement of exact FedNPG]\label{lm:performance_improvement_0}
For all starting state distribution $\rho\in\Delta(\S)$, we have the iterates of FedNPG satisfy
    \begin{equation}\label{eq:V_improvement_0}
        \phi^{(t+1)}(\eta)\leq \phi^{(t)}(\eta)+\frac{2\eta}{(1-\gamma)^2}\normbig{\widehat Q^{(t)}-\overline Q^{(t)}}_\infty-\eta\left(V^{\star}(\rho)-\overline{V}^{(t)}(\rho)\right)\,,
    \end{equation}
    where 
    \begin{equation}\label{eq:phi}
        \phi^{(t)}(\eta)\coloneqq \ex{s\sim d_\rho^{\pi^\star}}{\KL{\pi^\star(\cdot|s)}{\overline\pi^{(t)}(\cdot|s)}}-\frac{\eta}{1-\gamma}\overline{V}^{(t)}(d_\rho^{\pi^\star})\,,\quad \forall t\geq 0\,.
    \end{equation}
\end{lm}
It remains to control the term $\normbig{\overline{Q}^{(t)} -\widehat{Q}^{(t)}}_\infty$. Similar to \eqref{eq:diff_Q_tau_overline}, 
for all $t\geq 0$, we have
\begin{align}
    \normbig{\overline{Q}^{(t)} -\widehat{Q}^{(t)}}_\infty
 &   = \norm{\frac{1}{N}\sum_{n=1}^N Q_{n}^{\pi_n^{(t)}}-\frac{1}{N}\sum_{n=1}^N Q_{n}^{\overline\pi^{(t)}}}_\infty\notag\\
 &   \overset{(a)}{\leq} \frac{(1+\gamma)\gamma}{(1-\gamma)^2}\cdot\frac{1}{N}\sum_{n=1}^N\normbig{\log\xi_n^{(t)}-\log\overline\xi^{(t)}}_\infty\notag\\
  &  \overset{(b)}{\leq} \frac{(1+\gamma)\gamma}{(1-\gamma)^2}\normbig{u^{(t)}}_\infty\,,\label{eq:diff_Q_overline}
\end{align}
where (a) invokes Lemma~\ref{lm:difference_soft_Q} with $\tau=0$ and (b) stems from the definition of $u^{(t)}$. This along with \eqref{eq:V_improvement_0} gives
\begin{equation*}
    \phi^{(t+1)}(\eta)\leq \phi^{(t)}(\eta)+\frac{2(1+\gamma)\gamma}{(1-\gamma)^4}\eta\normbig{u^{(t)}}_\infty-\eta\left({V}^{\star}(\rho)-\overline{V}^{(t)}(\rho)\right)\,.
\end{equation*}

\paragraph{Step 4: bound the consensus error.} To bound the consensus error $\norm{\log\pi_n^{(t)}-\log\overline\pi^{(t)}}_\infty$ for all $n\in[N]$, we first upper bound the spectral norm of $\mB(\eta)$ which we denote as $\rho(\mB(\eta))$. Since $\mB(\eta)$ is a nonnegative matrix, by Perron-Frobenius Theorem, $\rho(\mB(\eta))$ is an eigenvalue of $\mB(\eta)$. Hence, we only need to upper bound the eigenvalue of $\rho(\mB(\eta))$.

The characteristic polynomial of $\mB(\eta)$ is
\begin{align*}
    f(\lambda)& =(\lambda-\sigma)\left(\lambda-\sigma\left(1+\frac{(1+\gamma)\gamma\sqrt{N}\eta}{(1-\gamma)^3}\sigma\right)\right)-\frac{\eta J}{1-\gamma}\sigma^2\\
    &=\lambda^2-\left(2+\frac{(1+\gamma)\gamma\sqrt{N}\eta}{(1-\gamma)^3}\sigma\right)\sigma\lambda+\left(1+\frac{(1+\gamma)\gamma\sqrt{N}\eta}{(1-\gamma)^3}\sigma-\frac{\eta J}{1-\gamma}\right)\sigma^2\,.
\end{align*}
This gives
\begin{align}
    \rho(\mB(\eta)) &\leq \frac{\sigma}{2}\left[\left(2+\frac{(1+\gamma)\gamma\sqrt{N}\eta}{(1-\gamma)^3}\sigma\right)+\sqrt{\left(2+\frac{(1+\gamma)\gamma\sqrt{N}\eta}{(1-\gamma)^3}\sigma\right)^2-4\left(1+\frac{(1+\gamma)\gamma\sqrt{N}\eta}{(1-\gamma)^3}\sigma\right)+4\frac{\eta J}{1-\gamma}}\right]\notag\\
    &\leq\frac{\sigma}{2}\left[\left(2+\frac{(1+\gamma)\gamma\sqrt{N}\eta}{(1-\gamma)^3}\sigma\right)+\sqrt{\left(\frac{(1+\gamma)\gamma\sqrt{N}\eta}{(1-\gamma)^3}\sigma\right)^2+4\frac{\eta J}{1-\gamma}}\right]\notag\\
    &\leq\sigma\left[1+\frac{(1+\gamma)\gamma\sqrt{N}\eta}{(1-\gamma)^3}\sigma+\sqrt{\frac{\eta J}{1-\gamma}}\right]\,.\label{eq:lambda(B)}
\end{align}
Note that when $\eta\leq\eta_1$, we have (recall that $J=\frac{2(1+\gamma)\gamma}{(1-\gamma)^2}\sqrt{N}$):
$$\frac{(1+\gamma)\gamma\sqrt{N}\eta}{(1-\gamma)^3}\sigma\leq\frac{(1-\sigma)^2}{8}\,, \quad\mbox{and} \qquad
 \frac{\eta J}{1-\gamma}\leq \frac{(1-\sigma)^2}{4\sigma}\,.$$
Plugging the above two expressions into \eqref{eq:lambda(B)} yields 
\begin{align*}
    \rho(\mB(\eta)) &\leq\sigma\left(1+(1-\sigma)^2/8+(1-\sigma)/(2\sqrt{\sigma})\right)\\
   & \leq \sigma\left(1+(1-\sigma)/(8\sigma)+(1-\sigma)/(2\sigma)\right)=\frac{3}{8}\sigma+\frac{5}{8}<1\,.
\end{align*}

Therefore, when $\eta\leq\eta_1$, we have
\begin{align}
    \norm{\mathbf{\Omega}^{(t)}}_2&\leq \rho(\mB(\eta))\norm{\mathbf{\Omega}^{(t-1)}}_2+d_2(\eta)\notag\\
    &\leq \cdots \leq \rho^t(\mB(\eta))\norm{\mathbf{\Omega}^{(0)}}_2+\sum_{i=0}^{t-1}\rho^i(\mB(\eta))\frac{(1+\gamma)\gamma N\sigma}{(1-\gamma)^4}\eta\notag\\
    &\leq \rho^t(\mB(\eta))\norm{\mathbf{\Omega}^{(0)}}_2 + \frac{2N\sigma}{(1-\gamma)^4(1-\rho(\mB(\eta)))}\eta\notag\\
    &\leq \left(\frac{3}{8}\sigma+\frac{5}{8}\right)^t\norm{\mathbf{\Omega}^{(0)}}_2+\frac{16N\sigma}{3(1-\gamma)^4(1-\sigma)}\eta\,.\label{eq:consensus_error_pre0}
\end{align}
Combining the above inequality with the following fact:
\begin{equation*}
    \forall n\in[N]:\quad \norm{\log\pi_n^{(t)}-\log\overline\pi^{(t)}}_\infty\leq 2\norm{\log\xi_n^{(t)}-\log\overline\xi^{(t)}}_\infty\leq\Omega_1^{(t)}\leq \norm{\mathbf{\Omega}^{(t)}}_2,
\end{equation*}
where the first inequality uses \eqref{eq:log_policy}, we obtain \eqref{eq:consensus_error_vanilla}.

\subsection{Proof of Lemma~\ref{lm:linear_sys_inexact}}\label{sec:pf_lm_linear_sys_inexact}
The bound on $u^{(t+1)}(s,a)$ is already established in Step 1 in Appendix~\ref{sec:pf_lm_lin_sys_ent} and shall be omitted. As usual we only highlight the key differences with the proof of Lemma \ref{lm:linear_sys} due to approximation error.

\paragraph{Step 1: bound $v^{(t+1)}(s,a)=\norm{\mT^{(t+1)}(s,a)-\widehat{q}^{(t+1)}(s,a)\vone_N}_2$.} Let $\vq^{(t)}  \coloneqq \Big( q_{1}^{\pi_1^{(t)}},\cdots, q_{N}^{\pi_N^{(t)}}\Big)^\top$. From \eqref{eq:Q_tracking_0_inexact}, we have  
\begin{align}
    &\norm{\mT^{(t+1)}(s,a)-\widehat{q}^{(t+1)}(s,a)\vone_N}_2\notag\\
   & =\norm{\mW\left(\mT^{(t)}(s,a)+\vq^{(t+1)}(s,a)-\vq^{(t)}(s,a)\right)-\widehat{q}^{(t+1)}(s,a)\vone_N}_2\notag\\
   & =\norm{\left(\mW\mT^{(t)}(s,a)-\widehat{q}^{(t)}(s,a)\vone_N \right)+\mW\left(\vq^{(t+1)}(s,a)-\vq^{(t)}(s,a)\right)+\left(\widehat{q}^{(t)}(s,a)-\widehat{q}^{(t+1)}(s,a)\right)\vone_N }_2\notag\\
   & \leq\sigma\normbig{\mT^{(t)}(s,a)-\widehat{q}^{(t)}(s,a)\vone_N }_2+\sigma\norm{\left(\vq^{(t+1)}(s,a)-\vq^{(t)}(s,a)\right)+\left(\widehat{q}^{(t)}(s,a)-\widehat{q}^{(t+1)}(s,a)\right)\vone_N}_2\notag\\
   & \leq\sigma\normbig{\mT^{(t)}(s,a)-\widehat{q}^{(t)}(s,a)\vone_N}_2+\sigma\normbig{\vq^{(t+1)}(s,a)-\vq^{(t)}(s,a)}_2\notag\\
  &  \leq\sigma\normbig{\mT^{(t)}(s,a)-\widehat{q}^{(t)}(s,a)\vone_N}_2+\sigma\normbig{\mQ^{(t+1)}(s,a)-\mQ^{(t)}(s,a)}_2 + 2\sigma\sqrt{N}\norm{\ve}_\infty\,.\label{eq:diff_T_Q_0_inexact}
\end{align}
Note that \eqref{eq:bound_by_r_3_0} still holds for inexact FedNPG:
\begin{equation}\label{eq:bound_by_r_3_0_inexact}
    \norm{\mQ^{(t+1)}(s,a)-\mQ^{(t)}(s,a)}_2\leq \frac{(1+\gamma)\gamma}{(1-\gamma)^2}\sqrt{N}\norm{w_0^{(t)}}_\infty,
\end{equation}
where $w_0^{(t)}$ is defined in \eqref{eq:w_0}. We rewrite \eqref{eq:bound_intermediate_0}, the bound on $w_0^{(t)}(s,a)$, as
\begin{align}
     w_0^{(t)}(s,a)  &  \leq 2 \normbig{\log\vxi^{(t)}(s,a)-\log\overline\xi^{(t)}(s,a)\vone_N}_2 \nonumber \\
     & \qquad +\frac{\eta}{1-\gamma}\normbig{\mT^{(t)}(s,a)-\widehat{q}^{(t)}(s,a)\vone_N}_2+\frac{\eta\sigma}{1-\gamma}\cdot\sqrt{N}\big|\widehat{q}^{(t)}(s,a)-V^\star(s)\big|\,.\label{eq:bound_intermediate_0_inexact}
\end{align}
With the following bound 
$$
    \forall (s,a)\in\S\times\A:\quad \big|\widehat{q}^{(t)}(s,a)-V^\star(s)\big|\leq \normbig{\widehat q^{(t)}-\overline Q^{(t)}}_\infty+\frac{1}{1-\gamma}
$$
in mind, we write \eqref{eq:bound_intermediate_0} as
\begin{align*}
    w_0^{(t)}(s,a) & \leq 2\normbig{u^{(t)}}_\infty+\frac{\eta\sigma}{1-\gamma}\normbig{v^{(t)}}_\infty+\frac{\eta}{1-\gamma}\cdot\sqrt{N}\left( \normbig{\widehat q^{(t)}-\overline q^{(t)}}_\infty+\frac{1}{1-\gamma}\right)\,.
\end{align*}
Putting all pieces together, we obtain
\begin{equation}\label{eq:C_0_inexact}
    \begin{split}
        \normbig{v^{(t+1)}}_\infty & \leq\sigma\left(1+\frac{(1+\gamma)\gamma\sqrt{N}\eta}{(1-\gamma)^3}\sigma\right)\normbig{v^{(t)}}_\infty\\
        &\qquad +\frac{(1+\gamma)\gamma}{(1-\gamma)^2}\sqrt{N}\sigma\Bigg\{2\normbig{u^{(t)}}_\infty
        +\frac{\eta\sqrt{N}}{(1-\gamma)^2} +\frac{\eta\sqrt{N}}{1-\gamma}\norm{\ve}_\infty\Bigg\} \\
        & \qquad +2\sigma\sqrt{N}\norm{\ve}_\infty\,.
    \end{split}
\end{equation}

\paragraph{Step 2: establish the descent equation.} Note that Lemma~\ref{lm:performance_improvement_0} directly applies by replacing $\widehat{Q}^{(t)}$ with $\widehat{q}^{(t)}$:
\begin{equation*}
    \phi^{(t+1)}(\eta)\leq \phi^{(t)}(\eta)+\frac{2\eta}{(1-\gamma)^2}\norm{\widehat q^{(t)}-\overline Q^{(t)}}_\infty-\eta\left(V^{\star}(\rho)-\overline{V}^{(t)}(\rho)\right)\,.
\end{equation*}
To bound the middle term, for all $t\geq 0$, we have
\begin{align}
    \norm{\overline{Q}^{(t)} -\widehat{q}^{(t)}}_\infty
   & = \norm{\frac{1}{N}\sum_{n=1}^N Q_{n}^{\pi_n^{(t)}}-\frac{1}{N}\sum_{n=1}^N Q_{n}^{\overline\pi^{(t)}}}_\infty+\frac{1}{N}\norm{\sum_{n=0}^N\left(q_{n}^{\pi_n^{(t)}}-Q_{n}^{\pi_n^{(t)}}\right)}_\infty\notag\\
  &  \leq \frac{(1+\gamma)\gamma}{(1-\gamma)^2}\cdot\frac{1}{N}\sum_{n=1}^N\norm{\log\xi_n^{(t)}-\log\overline\xi^{(t)}}_\infty+\frac{1}{N}\sum_{n=1}^N e_n\notag\\
&    \leq \frac{(1+\gamma)\gamma}{(1-\gamma)^2}\norm{u^{(t)}}_\infty+\norm{\ve}_\infty\,.\label{eq:diff_Q_overline_inexact}
\end{align}
Hence, \eqref{eq:A0_inexact} is established by combining the above two inequalities.

\paragraph{Step 3: bound the consensus error.} Similar to~\eqref{eq:consensus_error_pre0}, here we have
\begin{align}
    \norm{\mathbf{\Omega}^{(t)}}_2 &\leq \rho(\mB(\eta))\norm{\mathbf{\Omega}^{(t-1)}}_2+(d_2(\eta)+c_2(\eta))\notag\\
    &\leq \cdots \leq \rho^t(\mB(\eta))\norm{\mathbf{\Omega}^{(0)}}_2+\sum_{i=0}^{t-1}\rho^i(\mB(\eta))\left(\frac{(1+\gamma)\gamma N\sigma}{(1-\gamma)^4}\eta+\sqrt{N}\sigma\left(\frac{(1+\gamma)\gamma \eta \sqrt{N}}{(1-\gamma)^3}+2\right)\norm{\ve}_\infty\right)\notag\\
   & \leq \rho^t(\mB(\eta))\norm{\mathbf{\Omega}^{(0)}}_2 + \frac{2}{1-\rho(\mB(\eta))}\left(\frac{N\sigma}{(1-\gamma)^4}\eta+\sqrt{N}\sigma\left(\frac{\eta \sqrt{N}}{(1-\gamma)^3}+1\right)\norm{\ve}_\infty\right)\notag\\
   & \leq \left(\frac{3}{8}\sigma+\frac{5}{8}\right)^t\norm{\mathbf{\Omega}^{(0)}}_2+\frac{16}{3(1-\sigma)}\left(\frac{N\sigma}{(1-\gamma)^4}\eta+\sqrt{N}\sigma\left(\frac{\eta \sqrt{N}}{(1-\gamma)^3}+1\right)\norm{\ve}_\infty\right)\,,\label{eq:consensus_error_pre}
\end{align}
which indicates~\eqref{eq:consensus_error_0_inexact}.


\subsection{Proof of Lemma~\ref{lm:auxiliary_seq}}\label{sec_app:lm_aux}
    The first claim is easily verified as $\log \xi_n^{(t)}(s,\cdot)$ always deviate from $\log \pi_n^{(t)}(\cdot|s)$ by a global constant shift, as long as it holds for $t=0$:
    \begin{align*}
        \log \xi_n^{(t+1)}(s,\cdot) &= \sum_{n'=1}^{N}[W]_{n,n'} \left(\alpha\log \xi_{n'}^{(t)}(s,\cdot) + (1-\alpha)T_n^{(t)}(s,\cdot)/\tau\right)\\
        &=\alpha\sum_{n'=1}^{N}[W]_{n,n'} \left(\alpha\Big(\log \pi_{n'}^{(t)}(s,\cdot) + c_{n'}^{(t)}(s)\vone_{|\A|}\Big) + (1-\alpha)T_n^{(t)}(s,\cdot)/\tau\right)\\
        &=\alpha\sum_{n'=1}^{N}[W]_{n,n'} \left(\alpha\log \pi_{n'}^{(t)}(s,\cdot) + (1-\alpha)T_n^{(t)}(s,\cdot)/\tau\right) - \log z_{n}^{(t)}(s)\vone_{|\A|} + c_n^{(t+1)}(s)\vone_{|\A|}\\
        &=\log \pi_n^{(t+1)}(\cdot|s) + c_n^{(t+1)}(s)\vone_{|\A|},
    \end{align*}
    where $z_n^{(t)}$ is the normalization term (cf. line 5, Algorithm~\ref{alg:EDNPG_exact}) and $\{c_n^{(t)}(s)\}$ are some constants.
    To prove the second claim,
    $\forall t\geq 0, \forall (s,a)\in\S\times\A$, let
    \begin{equation}\label{eq:def_avg_T}
        \overline{T}^{(t)}(s,a)\coloneqq\frac{1}{N}\vone^\top\mT^{(t)}(s,a)\,.
    \end{equation}

    Taking inner product with $\frac{1}{N}\vone$ for both sides of \eqref{eq:Q_tracking} and using the double stochasticity property of $\mW$, we get
    \begin{equation}\label{eq:avg_T}
        \overline{T}^{(t+1)}(s,a)=\overline{T}^{(t)}(s,a)+\widehat{Q}_\tau^{(t+1)}(s,a)-\widehat{Q}_\tau^{(t)}(s,a)\,.
    \end{equation}
    By the choice of $\mT^{(0)}$ (line~2 of Algorithm~\ref{alg:EDNPG_exact}), we have $\overline{T}^{(0)}=\widehat{Q}_\tau^{(0)}$ and hence by induction
    \begin{equation}\label{eq:avg_T=Q}
        \forall t\geq 0: \quad \overline{T}^{(t)}=\widehat{Q}_\tau^{(t)}\,.
    \end{equation}
    This implies
    \begin{align*}
        \log\overline{\xi}^{(t+1)}(s,a) - \alpha \log\overline{\xi}^{(t)}(s,a) &= (1-\alpha) \widehat{Q}_\tau^{(t)}(s,a)/\tau\\
        &= (1-\alpha) \overline{T}^{(t)}(s,a)/\tau\\
        &= \frac{1}{N}\vone^\top\log\vxi^{(t+1)}(s,a) - \alpha\frac{1}{N}\vone^\top\log\vxi^{(t)}(s,a).
    \end{align*}
    Therefore, to prove \eqref{eq:xi_avg_property}, it suffices to verify the claim for $t=0$:
    \begin{align*}
        \frac{1}{N}\vone^\top\log\vxi^{(0)}(s,a)&=\log \norm{\exp \left(Q_\tau^\star(s,\cdot)/\tau\right)}_1+\frac{1}{N}\vone^\top\log\vpi^{(0)}(a|s)-\log\norm{\exp \left(\frac{1}{N}\sum_{n=1}^N\log\pi_n^{(0)}(\cdot|s)\right)}_1\\
        &=\log \norm{\exp \left(Q_\tau^\star(s,\cdot)/\tau\right)}_1+\log\overline{\pi}^{(0)}(a|s)=\log\overline{\xi}^{(0)}(s,a)\,.
    \end{align*} 
    By taking logarithm over both sides of the definition of $\overline{\pi}^{(t+1)}$ (cf. \eqref{eq:update_marl}), we get
    \begin{align}
        \log \overline{\pi}^{(t+1)}(a|s) = \alpha \log \overline{\pi}^{(t)}(a|s) + (1-\alpha)\widehat{Q}^{(t)}(s,a)/\tau - z^{(t)}(s)
        \label{eq:update_pi_avg}
    \end{align}
    for some constant $z^{(t)}(s)$, which deviate from the update rule of $\log \overline{\xi}^{(t+1)}$ by a global constant shift and hence verifies \eqref{eq:prop_auxiliary_avg}.

\subsection{Proof of Lemma~\ref{lm:difference_soft_Q}}


For notational simplicity, we let $Q_\tau^{\xi'}$ and  $Q_\tau^{\xi}$ denote $Q_\tau^{\pi_{\xi'}}$ and  $Q_\tau^{\pi_{\xi}}$, respectively.
From \eqref{eq:Q_V} we immediately know that to bound $\norm{Q_\tau^{\xi'}-Q_\tau^{\xi}}_\infty$, it suffices to control $\big|V_\tau^\xi(s)-V_\tau^{\xi'}(s)\big|$ for each $s\in \mathcal{S}$. By \eqref{eq:V_reg_s} we have
\begin{equation}\label{eq:ineq_V_tau}
    \big|V_\tau^\xi(s)-V_\tau^{\xi'}(s)\big|\leq \big|V^\xi(s)-V^{\xi'}(s)\big|+\tau\big|\mathcal{H}(s,\pi_\xi)-\mathcal{H}(s,\pi_{\xi'})\big|\,,
\end{equation}
so in the following we bound both terms in the RHS of \eqref{eq:ineq_V_tau}.

\paragraph{Step 1: bounding $\big|\mathcal{H}(s,\pi_\xi)-\mathcal{H}(s,\pi_{\xi'})\big|$.} We first bound $\big|\mathcal{H}(s,\pi_\xi)-\mathcal{H}(s,\pi_{\xi'})\big|$ using the idea in the proof of Lemma~14 in \cite{mei2020global}. We let 
\begin{equation}\label{eq:theta_t}
    \xi^{(t)}=\xi+t(\xi'-\xi)\,,\quad \forall t\in \R\,,
\end{equation}
and let $h^{(t)}\in\R^{|\S|}$ be
\begin{equation}\label{eq:h_t}
    \forall s\in\S:\quad h^{(t)}(s)\coloneqq -\sum_{a\in\A}\pi_{\xi^{(t)}}(a|s)\log\pi_{\xi^{(t)}}(a|s)\,.
\end{equation}
Note that $\norm{h^{(t)}}_\infty\leq\log|\A|$.
We also denote $H^{(t)}:\S\rightarrow \R^{|\A|\times|\A|}$ by:
\begin{equation}\label{eq:H_t}
   \forall s\in\S:\quad H^{(t)}(s)\coloneqq\frac{\partial \pi_\xi(\cdot|s)}{\partial \xi}\bigg|_{\xi=\xi^{(t)}}=\diag\{\pi_{\xi^{(t)}}(\cdot|s)\}-\pi_{\xi^{(t)}}(\cdot|s)\pi_{\xi^{(t)}}(\cdot|s)^\top\,,
\end{equation}
then we have
\begin{align}\label{eq:derivative_t}
    \forall s\in\S:\quad \left|\frac{d h^{(t)}(s)}{dt}\right|&=\left|\left\langle\frac{\partial h^{(t)}(s)}{\partial \xi^{(t)}(\cdot|s)},\xi'(s,\cdot)-\xi(s,\cdot)\right\rangle\right|\notag\\
    &=\left|\left\langle H^{(t)}(s)\log \pi_{\xi^{(t)}}(\cdot|s),\xi'(s,\cdot)-\xi(s,\cdot)\right\rangle\right|\notag\\
    & \leq\norm{H^{(t)}(s)\log \pi_{\xi^{(t)}}(\cdot|s)}_1\norm{\xi'(s,\cdot)-\xi(s,\cdot)}_\infty\,,
\end{align}
where
$\frac{\partial h^{(t)}(s)}{\partial \xi^{(t)}(\cdot|s)}$ stands for
$\frac{\partial h^{(t)}(s)}{\partial \xi(\cdot|s)}\big|_{\xi=\xi^{(t)}}$.
The first term in \eqref{eq:derivative_t} is further upper bounded as
\begin{align*}
    \norm{H^{(t)}(s)\log \pi_{\xi^{(t)}}(\cdot|s)}_1&=\sum_{a\in\A}\pi_{\xi^{(t)}}(a|s)\left|\log\pi_{\xi^{(t)}}(a|s)-\pi_{\xi^{(t)}}(\cdot|s)^\top \log\pi_{\xi^{(t)}}(\cdot|s) \right|\\
  &  \leq\sum_{a\in\A}\pi_{\xi^{(t)}}(a|s)\left(\left|\log\pi_{\xi^{(t)}}(a|s)\right|+\left|\pi_{\xi^{(t)}}(\cdot|s)^\top \log\pi_{\xi^{(t)}}(\cdot|s) \right|\right)\\
    &=-2\sum_{a\in\A}\pi_{\xi^{(t)}}(a,s)\log\pi_{\xi^{(t)}}(a|s)\leq 2\log|\A|\,.
\end{align*}

By Lagrange mean value theorem, there exists $t\in(0,1)$ such that
$$\left|h_1(s)-h_0(s)\right|=\left|\frac{d h^{(t)}(s)}{dt}\right|\leq 2\log|\A|\norm{\xi'(s,\cdot)-\xi(s,\cdot)}_\infty\,,$$
where the inequality follows from \eqref{eq:derivative_t} and the above inequality. 
Combining \eqref{eq:entropy_reg_s} with the above inequality, we arrive at
\begin{equation}\label{eq:H_difference}
    \big|\mathcal{H}(s,\pi_\xi)-\mathcal{H}(s,\pi_{\xi'})\big|\leq \frac{2\log|\A|}{1-\gamma}\norm{\xi'-\xi}_\infty\,.
\end{equation}

\paragraph{Step 2: bounding $\big|V^\xi(s)-V^{\xi'}(s)\big|$.} Similar to the previous proof, we bound $\big|V^\xi(s)-V^{\xi'}(s)\big|$ by bounding $\left|\frac{dV^{{\xi^{(t)}}}}{dt}(s)\right|$. By Bellman's consistency equation, the value function of $\pi_{\xi^{(t)}}$ is given by
$$V^{\xi^{(t)}}(s)=\sum_{a\in\A}\pi_{\xi^{(t)}}(a|s)r(s,a)+\gamma\sum_{a}\pi_{\xi_\alpha}(a|s)\sum_{s'\in\S}\mathcal{P}(s'|s,a)V^{\xi^{(t)}}(s')\,,$$
which can be represented in a matrix-vector form as
\begin{equation}\label{eq:V_t_matrix_form}
    V^{\xi^{(t)}_\star}(s)=e_s^\top \mM_t r_t\,,
\end{equation}
where $e_s\in \R^{|\S|}$ is a one-hot vector whose $s$-th entry is 1,
\begin{equation}\label{eq:M_t}
    \mM_t\coloneqq(\mI-\gamma\mP_t)^{-1}\,,
\end{equation}
with $\mP_t\in\R^{|\S|\times|\S|}$ denoting the induced state transition matrix by $\pi_{\xi^{(t)}}$
\begin{equation}\label{eq:P_t}
    \mP_t(s,s')=\sum_{a\in\A}\pi_{\xi^{(t)}}(a|s)\mathcal{P}(s'|s,a)\,,
\end{equation}
and $r_t\in\R^{|\S|}$ is given by
\begin{equation}\label{eq:r_t}
    \forall s\in\S:\quad r_t(s)\coloneqq \sum_{a\in\A}\pi_{\xi^{(t)}}(a|s)r(s,a)\,.
\end{equation}
Taking derivative w.r.t. $t$ in \eqref{eq:V_t_matrix_form}, we obtain \citep{petersen2008matrix}
\begin{equation}\label{eq:V_t_derivative_matrix}
    \frac{dV^{\xi^{(t)}}(s)}{dt}=\gamma\cdot e_s^\top \mM_t\frac{d\mP_t}{dt}\mM_t r_t+e_s^\top \mM_t\frac{dr_t}{dt}\,.
\end{equation}
We now calculate each term respectively. 
\begin{itemize}
\item For the first term, it follows that
\begin{align}
    \left|\gamma\cdot e_s^\top \mM_t\frac{d\mP_t}{dt}\mM_t r_t\right| & \leq \gamma\norm{\mM_t\frac{d\mP_t}{dt}\mM_t r_t}_\infty\notag\\
  &  \leq \frac{\gamma}{1-\gamma}\norm{\frac{d\mP_t}{dt}\mM_t r_t}_\infty\notag\\
  &  \leq \frac{2\gamma}{1-\gamma}
    \norm{\mM_t r_t}_\infty\norm{\xi'-\xi}_\infty \label{eq:P_t_derivative_bound} \\
&   \leq \frac{2\gamma}{(1-\gamma)^2}
    \norm{r_t}_\infty\norm{\xi'-\xi}_\infty\notag\\
 &   \leq \frac{2\gamma}{(1-\gamma)^2}
    \norm{\xi'-\xi}_\infty\,.\label{eq:V_t_derivative_bound1}
\end{align}
where the second and fourth lines use the fact $\|\mM_t \|_1 \leq 1/(1-\gamma)$ \citep[Lemma 10]{li2020breaking}, and the last line follow from 
$$    \norm{r_t}_\infty=\max_{s\in\S} \left |\sum_{a\in\A}\pi_{\xi^{(t)}}(a|s)r(s,a) \right |\leq 1. $$ 
We defer the proof of \eqref{eq:P_t_derivative_bound} to the end of proof.

\item For the second term, it follows that 
\begin{align}
    \left|e_s^\top \mM_t\frac{dr_t}{dt}\right| \leq \frac{1}{1-\gamma}\norm{\frac{dr_t}{dt}}_\infty \leq \frac{1}{1-\gamma}\norm{\xi'-\xi}_\infty\,.\label{eq:V_t_derivative_bound2}
\end{align}
where the first inequality follows again from $\|\mM_t \|_1 \leq 1/(1-\gamma)$, and the second inequality follows from
\begin{align}
   \norm{\frac{dr_t}{dt}}_\infty =\max_{s\in \S} \left|\frac{ dr_t(s)}{dt}\right|&= \max_{s\in \S}  \left|\left\langle\frac{\partial \pi_{\xi^{(t)}}(\cdot|s)^\top r(s,\cdot)}{\partial \xi^{(t)}(s,\cdot)}, \xi'(s,\cdot)-\xi(s,\cdot)\right\rangle\right|\notag\\
    & \leq \max_{s\in \S}   \norm{\frac{\partial \pi_{\xi^{(t)}}(\cdot|s)^\top}{\partial \xi^{(t)}(s,\cdot)}r(s,\cdot)}_1\norm{\xi'(s,\cdot)-\xi(s,\cdot)}_\infty\notag\\
    &= \max_{s\in \S}  \left(\sum_{a\in\A}\pi_{\xi^{(t)}}(a|s)\left|r(s,a)-\pi_{\xi^{(t)}}(\cdot|s)^\top r(s,\cdot)\right|\right)\norm{\xi'(s,\cdot)-\xi(s,\cdot)}_\infty\notag\\
    &  \leq  \max_{s\in \S}  \underbrace{\max_{a\in\A}\left|r(s,a)-\pi_{\xi^{(t)}}(\cdot|s)^\top r(s,\cdot)\right|}_{\leq 1 \text{ since } r(s,a)\in[0,1]}\norm{\xi'(s,\cdot)-\xi(s,\cdot)}_\infty\notag\\
   & \leq  \max_{s\in \S}  \norm{\xi'(s,\cdot)-\xi(s,\cdot)}_\infty =\norm{\xi'-\xi}_\infty.\label{eq:l_inf_drt}
\end{align}
 
\end{itemize}





 Plugging the above two inequalities into \eqref{eq:V_t_derivative_matrix} and using Lagrange mean value theorem, we have
 \begin{equation}\label{eq:difference_V}
     \big|V^\xi(s)-V^{\xi'}(s)\big|\leq \frac{1+\gamma}{(1-\gamma)^2}\norm{\xi'-\xi}_\infty\,.
 \end{equation}
 
\paragraph{Step 3: sum up.} Combining \eqref{eq:difference_V}, \eqref{eq:H_difference} and \eqref{eq:ineq_V_tau}, we have
\begin{equation}\label{eq:difference_V_tau}
    \forall s\in\mathcal{S}:\quad \big|V_\tau^\xi(s)-V_\tau^{\xi'}(s)\big|\leq \frac{1+\gamma+2\tau(1-\gamma)\log|\A|}{(1-\gamma)^2}\norm{\log\pi-\log\pi'}_\infty\,.
\end{equation}

Combining \eqref{eq:difference_V_tau} and \eqref{eq:Q_V}, \eqref{eq:difference_soft_Q} immediately follows.

\paragraph{Proof of \eqref{eq:P_t_derivative_bound}.}
For any vector $x\in\R^{|\S|}$, we have
$$\left[\frac{d\mP_t}{dt}x\right]_{s}=\sum_{s'\in\S}\sum_{a\in\A}\frac{d\pi_{\xi^{(t)}}(a|s)}{dt}\mathcal{P}(s'|s,a)x(s')\,,$$
from which we can bound the $l_\infty$ norm as
\begin{align}
    \norm{\frac{d\mP_t}{dt}x}_\infty & \leq  \max_s\sum_{a\in\A}\sum_{s'\in\S}\mathcal{P}(s'|s,a)\left|\frac{d\pi_{\xi^{(t)}}(a|s)}{dt}\right|\norm{x}_\infty\notag\\
    &= \max_s\sum_{a\in\A}\left|\frac{d\pi_{\xi^{(t)}}(a|s)}{dt}\right|\norm{x}_\infty\notag\\
 &   \leq  2\norm{\xi'-\xi}_\infty\norm{x}_\infty\,\label{eq:l_inf_dPt}
\end{align}
as desired, where the last line follows from the following fact:
\begin{align*}
    \sum_{a\in\A}\left|\frac{d\pi_{\xi^{(t)}}(a|s)}{dt}\right|&=\sum_{a\in\A}\left|\left\langle\frac{\partial\pi_{\xi^{(t)}}(a|s)}{\partial\xi^{(t)}}, \xi'-\xi\right\rangle\right|\\
    &=\sum_{a\in\A}\left|\left\langle\frac{\partial\pi_{\xi^{(t)}}(a|s)}{\partial\xi^{(t)}(s,\cdot)}, \xi'(s,\cdot)-\xi(s,\cdot)\right\rangle\right|\\
    &=\sum_{a\in\A}\pi_{\xi^{(t)}}(a|s)\left|\left(\xi'(s,a)-\xi(s,a)\right)-\pi_{\xi^{(t)}}(\cdot|s)^\top\left(\xi'(s,\cdot)-\xi(s,\cdot)\right)\right|\\
    &\leq\max_a\left|\xi'(s,a)-\xi(s,a)\right|+\left|\pi_{\xi^{(t)}}(\cdot|s)^\top\left(\xi'(s,\cdot)-\xi(s,\cdot)\right)\right|\\
 &   \leq 2\norm{\xi'-\xi}_\infty\,.
\end{align*}

\subsection{Proof of Lemma~\ref{lm:performance_improvement}}
\label{sec:pf_lm_perf_imprv}
        To simplify the notation, we denote
        \begin{equation}\label{eq:delta}
            \delta^{(t)}\coloneqq \widehat{Q}^{(t)}_\tau-\overline{Q}^{(t)}_\tau\,.
        \end{equation}
        We first rearrange the terms of \eqref{eq:update_pi_avg} and obtain
        \begin{equation}
            -\tau\log\overline{\pi}^{(t)}(a|s)+\left(\overline{Q}^{(t)}_\tau(s,a)+\delta^{(t)}(s,a)\right)=\frac{1-\gamma}{\eta}\left(\log\overline{\pi}^{(t+1)}(a|s)-\log\overline{\pi}^{(t)}(a|s)\right)+\frac{1-\gamma}{\eta}z^{(t)}(s)\,.\label{eq:lm_pi_1}
        \end{equation}
    This in turn allows us to express $\overline V_\tau^{(t)}(s_0)$ for any $s_0\in\S$ as follows
    \begin{align}
        \overline V_\tau^{(t)}(s_0)&=\underset{a_0\sim\overline\pi^{(t)}(\cdot|s_0)}{\E}\left[-\tau\log\overline\pi^{(t)}(a_0|s_0)+\overline{Q}^{(t)}_\tau(s_0,a_0)\right]\notag\\
        &=\underset{a_0\sim\overline\pi^{(t)}(\cdot|s_0)}{\E}\left[\frac{1-\gamma}{\eta}z^{(t)}(s_0)\right]+\underset{a_0\sim\overline\pi^{(t)}(\cdot|s_0)}{\E}\left[\frac{1-\gamma}{\eta}\left(\log\overline{\pi}^{(t+1)}(a_0|s_0)-\log\overline{\pi}^{(t)}(a_0|s_0)\right)-\delta^{(t)}(s_0,a_0)\right]\notag\\
        &=\frac{1-\gamma}{\eta}z^{(t)}(s_0)-\frac{1-\gamma}{\eta}\KL{\overline{\pi}^{(t)}(\cdot|s_0)}{\overline{\pi}^{(t+1)}(\cdot|s_0)}-\underset{a_0\sim\overline\pi^{(t)}(\cdot|s_0)}{\E}\left[\delta^{(t)}(s_0,a_0)\right]\notag\\
        &=\underset{a_0\sim\overline\pi^{(t+1)}(\cdot|s_0)}{\E}\left[\frac{1-\gamma}{\eta}z^{(t)}(s_0)\right]-\frac{1-\gamma}{\eta}\KL{\overline{\pi}^{(t)}(\cdot|s_0)}{\overline{\pi}^{(t+1)}(\cdot|s_0)}-\underset{a_0\sim\overline\pi^{(t)}(\cdot|s_0)}{\E}\left[\delta^{(t)}(s_0,a_0)\right]\,,\label{eq:lm_pi_2}
    \end{align}
    where the first identity makes use of \eqref{eq:V_Q}, the second line follows from \eqref{eq:lm_pi_1}. Invoking \eqref{eq:V_Q} again to rewrite the $z(s_0)$ appearing in the first term of \eqref{eq:lm_pi_2}, we reach
    \begin{align}
        &\overline V_\tau^{(t)}(s_0)\notag\\
        &=\underset{a_0\sim\overline\pi^{(t+1)}(\cdot|s_0)}{\E}\left[-\tau\log\overline{\pi}^{(t+1)}(a_0|s_0)+\overline{Q}^{(t)}_\tau(s_0,a_0)+\left(\tau-\frac{1-\gamma}{\eta}\right)\left(\log\overline{\pi}^{(t+1)}(a_0|s_0)-\log\overline{\pi}^{(t)}(a|s)\right)\right]\notag\\
        &\qquad-\frac{1-\gamma}{\eta}\KL{\overline{\pi}^{(t)}(\cdot|s_0)}{\overline{\pi}^{(t+1)}(\cdot|s_0)}-\underset{a_0\sim\overline\pi^{(t)}(\cdot|s_0)}{\E}\left[\delta^{(t)}(s_0,a_0)\right]+\underset{a_0\sim\overline\pi^{(t+1)}(\cdot|s_0)}{\E}\left[\delta^{(t)}(s_0,a_0)\right]\notag\\
        &= \underset{a_0\sim\overline\pi^{(t+1)}(\cdot|s_0), \atop s_1\sim P(\cdot|s_0,a_0)}{\E}\left[-\tau\log\overline{\pi}^{(t+1)}(a_0|s_0)+r(s_0,a_0)+\gamma\overline{V}^{(t)}_\tau(s_0)\right]\notag\\
        &\qquad-\left(\frac{1-\gamma}{\eta}-\tau\right)\KL{\overline{\pi}^{(t+1)}(\cdot|s_0)}{\overline{\pi}^{(t)}(\cdot|s_0)}-\frac{1-\gamma}{\eta}\KL{\overline{\pi}^{(t)}(\cdot|s_0)}{\overline{\pi}^{(t+1)}(\cdot|s_0)}\notag\\
        &\qquad-\underset{a_0\sim\overline\pi^{(t)}(\cdot|s_0)}{\E}\left[\delta^{(t)}(s_0,a_0)\right]+\underset{a_0\sim\overline\pi^{(t+1)}(\cdot|s_0)}{\E}\left[\delta^{(t)}(s_0,a_0)\right]\,.\label{eq:lm_pi_3}
    \end{align}

    Note that for any $(s_0,a_0)\in\S\times \A$, we have
    \begin{align}
        &-\underset{a_0\sim\overline\pi^{(t)}(\cdot|s_0)}{\E}\left[\delta^{(t)}(s_0,a_0)\right]+\underset{a_0\sim\overline\pi^{(t+1)}(\cdot|s_0)}{\E}\left[\delta^{(t)}(s_0,a_0)\right]\notag\\
        &=\sum_{a_0\in\A}\left(\overline\pi^{(t+1)}(a_0|s_0)-\overline\pi^{(t)}(a_0|s_0)\right)\delta^{(t)}(s_0,a_0)\notag\\
        &\leq \normbig{\overline\pi^{(t+1)}(\cdot|s_0)-\overline\pi^{(t)}(\cdot|s_0)}_1\normbig{\delta^{(t)}}_\infty\leq 2 \normbig{\delta^{(t)}}_\infty\,.\label{eq:bound_delta}
    \end{align}

    To finish up, applying \eqref{eq:lm_pi_3} recursively to expand $\overline V_\tau^{(t)}(s_i)$, $i\geq1$ and making use of \eqref{eq:bound_delta}, we arrive at
    \begin{align}
        &\overline V_\tau^{(t)}(s_0)\notag\\
        &\leq \sum_{i=1}^\infty\gamma^i \cdot 2\norm{\delta^{(t)}}_\infty+\underset{a_i\sim\overline\pi^{(t+1)}(\cdot|s_i), \atop s_{i+1}\sim P(\cdot|s_i,a_i),\forall i\geq 0}{\E}\Bigg[\sum_{i=1}^\infty\gamma^i\left\{r(s_i,a_i)-\tau\log\overline\pi^{(t+1)}(a_i|s_i)\right\}\notag\\
        &\qquad-\sum_{i=1}^\infty\gamma^i\left\{\left(\frac{1-\gamma}{\eta}-\tau\right)\KL{\overline{\pi}^{(t+1)}(\cdot|s_i)}{\overline{\pi}^{(t)}(\cdot|s_i)}+\frac{1-\gamma}{\eta}\KL{\overline{\pi}^{(t)}(\cdot|s_i)}{\overline{\pi}^{(t+1)}(\cdot|s_i)}\right\}\Bigg]\notag\\
        &= \frac{2}{1-\gamma}\norm{\delta^{(t)}}_\infty+\overline V_\tau^{(t+1)}(s_0)\notag\\
        &\qquad-\underset{s\sim d_{s_0}^{\overline\pi^{(t+1)}}}{\E}\left[\left(\frac{1}{\eta}-\frac{\tau}{1-\gamma}\right)\KL{\overline{\pi}^{(t+1)}(\cdot|s_i)}{\overline{\pi}^{(t)}(\cdot|s_i)}+\frac{1}{\eta}\KL{\overline{\pi}^{(t)}(\cdot|s_i)}{\overline{\pi}^{(t+1)}(\cdot|s_i)}\right]\,,\label{eq:lm_ip_final}
    \end{align}
    where the third line follows since $\overline V_\tau^{(t+1)}$ can be viewed as the value function of $\overline\pi^{(t+1)}$ with adjusted rewards $\overline r^{(t+1)}(s,a)\coloneqq r(s,a)-\tau\log\overline\pi^{(t+1)}(s|a)$. And \eqref{eq:V_improvement} follows immediately from the above inequality \eqref{eq:lm_ip_final}. By \eqref{eq:Q_V} we can easily see that \eqref{eq:Q_improvement} is a consequence of \eqref{eq:V_improvement}.

\subsection{Proof of Lemma~\ref{lm:performance_improvement_0}}

\label{sec:pf_lm_perf_imprv_0}

We first introduce the famous performance difference lemma which will be used in our proof.
\begin{lm}[Performance difference lemma]\label{lm:performance_difference}
    For any policy $\pi,\pi'\in\Delta(\A)^\S$ and $\rho\in\Delta(\S)$, we have
    \begin{align}
        V^\pi(\rho)-V^{\pi'}(\rho)&=\frac{1}{1-\gamma}\E_{(s,a)\sim \bar d^\pi}\left[A^{\pi'}(s,a)\right]\label{eq:V_diff_A}\\
        &=\frac{1}{1-\gamma}\E_{s\sim d^\pi}\left[\langle Q^{\pi'}(s), \pi(s)-\pi'(s)\rangle\right].\label{eq:V_diff_Q}
    \end{align}
\end{lm}
\begin{proof}
    See Lemma~3 in \citet{yuan2022linear}.
\end{proof}

    For all $t\geq 0$, we define the advantage function $\overline{A}^{(t)}$ as:
    \begin{equation}\label{eq:advantage_avg}
        \forall (s,a)\in\S\times\A:\quad\overline{A}^{(t)}(s,a)\coloneqq \overline{Q}^{(t)}(s,a)-\overline{V}^{(t)}(s)\,.
    \end{equation}
    Then for Alg.~\ref{alg:DNPG_exact}, the update rule of $\overline{\pi}$ (Eq.~\eqref{eq:update_pi_avg}) can be written as
    \begin{equation}\label{eq:update_pi_avg_0}
        \log\overline{\pi}^{(t+1)}(a|s)=\log\overline{\pi}^{(t)}(a|s)+\frac{\eta}{1-\gamma}\left(\overline A^{(t)}(s,a)+\delta^{(t)}(s,a)\right)-\log\widehat{z}^{(t)}(s)\,,
    \end{equation}
    where $\delta^{(t)}$ is defined in \eqref{eq:delta} and
    \begin{align}
        \log\widehat{z}^{(t)}(s)&=\log\sum_{a'\in\A}\overline{\pi}^{(t)}(a'|s)\exp\left\{\frac{\eta}{1-\gamma}\left(\overline A^{(t)}(s,a')+\delta^{(t)}(s,a')\right)\right\}\notag\\
        &\geq \sum_{a'\in\A}\overline{\pi}^{(t)}(a'|s)\log\exp\left\{\frac{\eta}{1-\gamma}\left(\overline A^{(t)}(s,a')+\delta^{(t)}(s,a')\right)\right\}\notag\\
        &= \frac{\eta}{1-\gamma}\sum_{a'\in\A}\overline{\pi}^{(t)}(a'|s)\left(\overline A^{(t)}(s,a')+\delta^{(t)}(s,a')\right)\notag\\
        &= \frac{\eta}{1-\gamma}\sum_{a'\in\A}\overline{\pi}^{(t)}(a'|s)\delta^{(t)}(s,a')\geq -\frac{\eta}{1-\gamma}\norm{\delta^{(t)}}_\infty\,,\label{eq:log_hat_z}
    \end{align}
    where the first inequality follows by Jensen's inequality on the concave function $\log x$ and the last equality uses $\sum_{a'\in\A}\overline{\pi}^{(t)}(a'|s)\overline A^{(t)}(s,a')=0$.

    For all starting state distribution $\mu$, we use $d^{(t+1)}$ as shorthand for $d_\mu^{\overline\pi^{(t+1)}}$, the performance difference lemma (Lemma~\ref{lm:performance_difference}) implies:
    \begin{align}
        &\overline{V}^{(t+1)}(\mu)-\overline{V}^{(t)}(\mu)\notag\\
        &=\frac{1}{1-\gamma}\mathbb{E}_{s\sim d^{(t+1)}}\sum_{a\in\A}\overline\pi^{(t+1)}(a|s)\left(\overline A^{(t)}(s,a)+\delta^{(t)}(s,a)\right)-\frac{1}{1-\gamma}\mathbb{E}_{s\sim d^{(t+1)}}\mathbb{E}_{a\sim\overline\pi^{(t+1)}(\cdot|s)}\left[\delta^{(t)}(s,a)\right]\notag\\
        &=\frac{1}{\eta}\mathbb{E}_{s\sim d^{(t+1)}} \sum_{a\in\A}\overline\pi^{(t+1)}(a|s)\log\frac{\overline\pi^{(t+1)}(a|s)\widehat z^{(t)}(s)}{\overline\pi^{(t)}(a|s)}-\frac{1}{1-\gamma}\mathbb{E}_{s\sim d^{(t+1)}}\mathbb{E}_{a\sim\overline\pi^{(t+1)}(\cdot|s)}\left[\delta^{(t)}(s,a)\right]\notag\\
        &= \frac{1}{\eta}\mathbb{E}_{s\sim d^{(t+1)}} \KL{\overline\pi^{(t+1)}(\cdot|s)}{\overline\pi^{(t)}(\cdot|s)}+\frac{1}{\eta}\mathbb{E}_{s\sim d^{(t+1)}}\log\widehat z^{(t)}(s)-\frac{1}{1-\gamma}\mathbb{E}_{s\sim d^{(t+1)}}\mathbb{E}_{a\sim\overline\pi^{(t+1)}(\cdot|s)}\left[\delta^{(t)}(s,a)\right]\notag\\
        &\geq \frac{1}{\eta}\mathbb{E}_{s\sim d^{(t+1)}}\left(\log\widehat z^{(t)}(s)+\frac{\eta}{1-\gamma}\normbig{\delta^{(t)}}_\infty\right)-\frac{2}{1-\gamma}\normbig{\delta^{(t)}}_\infty\,,\notag
    \end{align}
    from which we can see that
    \begin{equation}\label{eq:delta_V_0_bound}
        \overline{V}^{(t+1)}(\mu)-\overline{V}^{(t)}(\mu)\geq -\frac{2}{1-\gamma}\normbig{\delta^{(t)}}_\infty\,,
    \end{equation}
    where we use \eqref{eq:log_hat_z}, and that
    \begin{equation}\label{eq:delta_V_0_ineq}
        \overline{V}^{(t+1)}(\mu)-\overline{V}^{(t)}(\mu)\geq \frac{1-\gamma}{\eta}\mathbb{E}_{s\sim \mu}\left(\log\widehat z^{(t)}(s)+\frac{\eta}{1-\gamma}\normbig{\delta^{(t)}}_\infty\right)-\frac{2}{1-\gamma}\normbig{\delta^{(t)}}_\infty\,,
    \end{equation}
    which follows from $d^{(t+1)}=d_\mu^{\overline\pi^{(t+1)}}\geq (1-\gamma)\mu$ and the fact that $\log\widehat z^{(t)}(s)+\frac{\eta}{1-\gamma}\normbig{\delta^{(t)}}_\infty\geq 0$ (by \eqref{eq:log_hat_z}).

    For any fixed $\rho$, we use $d^\star$ as shorthand for $d_\rho^{\pi^\star}$. By the performance difference lemma (Lemma~\ref{lm:performance_difference}),
    \begin{align}
        &V^\star(\rho)-\overline V^{(t)}(\rho)\notag\\
        &=\frac{1}{1-\gamma}\mathbb{E}_{s\sim d^\star}\sum_{a\in\A}\pi^\star(a|s)\left(\overline A^{(t)}(s,a)+\delta^{(t)}(s,a)\right)-\frac{1}{1-\gamma}\mathbb{E}_{s\sim d^\star}\mathbb{E}_{a\sim\pi^\star(\cdot|s)}\left[\delta^{(t)}(s,a)\right]\notag\\
        &= \frac{1}{\eta}\mathbb{E}_{s\sim d^\star} \sum_{a\in\A}\pi^\star(a|s)\log\frac{\overline\pi^{(t+1)}(a|s)\widehat z^{(t)}(s)}{\overline\pi^{(t)}(a|s)}-\frac{1}{1-\gamma}\mathbb{E}_{s\sim d^\star}\mathbb{E}_{a\sim\pi^\star(\cdot|s)}\left[\delta^{(t)}(s,a)\right]\notag\\
        &= \frac{1}{\eta}\mathbb{E}_{s\sim d^\star} \left(\KL{\pi^\star(\cdot|s)}{\overline\pi^{(t)}(\cdot|s)}-\KL{\pi^\star(\cdot|s)}{\overline\pi^{(t+1)}(\cdot|s)}+\log\widehat z^{(t)}(s)\right)-\frac{1}{1-\gamma}\mathbb{E}_{s\sim d^\star}\mathbb{E}_{a\sim\pi^\star(\cdot|s)}\left[\delta^{(t)}(s,a)\right]\notag\\
        &\leq \frac{1}{\eta}\mathbb{E}_{s\sim d^\star} \left(\KL{\pi^\star(\cdot|s)}{\overline\pi^{(t)}(\cdot|s)}-\KL{\pi^\star(\cdot|s)}{\overline\pi^{(t+1)}(\cdot|s)}+\left(\log\widehat z^{(t)}(s)+\frac{\eta}{1-\gamma}\normbig{\delta^{(t)}}_\infty\right)\right)\,,\label{eq:V_diff_opt_0}
    \end{align}
    where we use \eqref{eq:update_pi_avg_0} in the second equality.

    By applying \eqref{eq:delta_V_0_ineq} with $\mu=d^\star$ as the initial state distribution, we have 
    \begin{equation*}
        \frac{1}{\eta}\mathbb{E}_{s\sim \mu}\Big(\log\widehat z^{(t)}(s)+\frac{\eta}{1-\gamma}\normbig{\delta^{(t)}}_\infty\Big)\leq \frac{1}{1-\gamma}\Big(\overline{V}^{(t+1)}(d^\star)-\overline{V}^{(t)}(d^\star)\Big)
        +\frac{2}{(1-\gamma)^2}\normbig{\delta^{(t)}}_\infty\,.
    \end{equation*}
    
    Plugging the above equation into \eqref{eq:V_diff_opt_0}, we obtain
    \begin{align*}
        V^\star(\rho)-\overline V^{(t)}(\rho)&\leq \frac{1}{\eta}\mathbb{E}_{s\sim d^\star} \left(\KL{\pi^\star(\cdot|s)}{\overline\pi^{(t)}(\cdot|s)}-\KL{\pi^\star(\cdot|s)}{\overline\pi^{(t+1)}(\cdot|s)}\right)\\
        &\qquad+\frac{1}{1-\gamma}\left(\overline{V}^{(t+1)}(d^\star)-\overline{V}^{(t)}(d^\star)\right)
        +\frac{2}{(1-\gamma)^2}\normbig{\delta^{(t)}}_\infty\,,
    \end{align*}
    which gives Lemma~\ref{lm:performance_improvement_0}.

\section{Proof of key lemmas for FedNAC}

\subsection{Proof of Lemma~\ref{lm:L-Q}}\label{sec_app:pf_lm_L_Q} 

  For notational simplicity we let $V^\xi,V^{\xi'}$ denote $V^{f_\xi},V^{f_{\xi'}}$, respectively. Same as in Lemma~\ref{lm:difference_soft_Q}, we define $\xi^{(t)}=\xi+t(\xi'-\xi)$ and define $\mP_t, \mM_t, r_t$ by replacing $\pi_{\xi^{(t)}}$ with $f_{\xi^{(t)}}$ in \eqref{eq:P_t},\eqref{eq:M_t} and \eqref{eq:r_t}, respectively. 
  Define $$\bar\phi_\xi(s,a)=\phi(s,a)-\E_{a'\sim f_{\xi^{(t)}}}[\phi(s,a')],$$ then we have
  \begin{equation}\label{eq:df_xi}
      \frac{\partial f_\xi(a|s)}{\partial \xi}=f_\xi(a|s)\bar\phi_\xi(s,a)\,.
  \end{equation}
  Analogous to \eqref{eq:l_inf_dPt}, we have
  \begin{align}
    \norm{\frac{d\mP_t}{dt}x}_\infty & \leq  \max_s\sum_{a\in\A}\sum_{s'\in\S}\mathcal{P}(s'|s,a)\left|\frac{d\pi_{\xi^{(t)}}(a|s)}{dt}\right|\norm{x}_\infty\notag\\
    &= \max_s\sum_{a\in\A}\left|\frac{d\pi_{\xi^{(t)}}(a|s)}{dt}\right|\norm{x}_\infty\notag\\
 &   \leq  2C_\phi\norm{\xi'-\xi}_2\norm{x}_\infty\,,\notag
\end{align}
where the last line  is due to
\begin{align*}
    \sum_{a\in\A}\left|\frac{df_{\xi^{(t)}}(a|s)}{dt}\right|&=\sum_{a\in\A}\left|\left\langle\frac{\partial f_{\xi^{(t)}}(a|s)}{\partial\xi^{(t)}}, \xi'-\xi\right\rangle\right|\\
    &=\sum_{a\in\A}f_{\xi^{(t)}}(a|s)\left|\langle\bar\phi_\xi(s,a),\xi'-\xi\rangle\right|\\
    &\leq\sum_{a\in\A}f_{\xi^{(t)}}(a|s)\norm{\bar\phi_\xi(s,a)}_2\norm{\xi'-\xi}_2\\
 &   \leq 2C_\phi\norm{\xi'-\xi}_\infty\,.
\end{align*}

Same as \eqref{eq:V_t_derivative_matrix} in Lemma~\ref{lm:difference_soft_Q}, we have
\begin{equation}\label{eq:V_t_derivative_matrix_fa}
    \frac{dV^{\xi^{(t)}}(s)}{dt}=\gamma\cdot e_s^\top \mM_t\frac{d\mP_t}{dt}\mM_t r_t+e_s^\top \mM_t\frac{dr_t}{dt}\,.
\end{equation}
And similar to \eqref{eq:l_inf_drt}, we deduce
\begin{align}
   \norm{\frac{dr_t}{dt}}_\infty =\max_{s\in \S} \left|\frac{ dr_t(s)}{dt}\right|&= \max_{s\in \S}  \left|\left\langle\frac{\partial f_{\xi^{(t)}}(\cdot|s)^\top r(s,\cdot)}{\partial \xi^{(t)}}, \xi'-\xi\right\rangle\right|\notag\\
    &= \left|\langle\sum_{a\in\A}f_\xi(a|s)\bar\phi_\xi(s,a)r(s,a),\xi'-\xi\rangle\right|\notag\\
    &= \sum_{a\in\A}f_\xi(a|s)r(s,a)\left|\langle\bar\phi_\xi(s,a),\xi'-\xi\rangle\right|\notag\\
    &\leq 2C_\phi\norm{\xi'-\xi}_2\,,\notag
\end{align}
which gives
\begin{align}
    \left|e_s^\top \mM_t\frac{dr_t}{dt}\right| \leq \frac{1}{1-\gamma}\norm{\frac{dr_t}{dt}}_\infty \leq \frac{2C_\phi}{1-\gamma}\norm{\xi'-\xi}_2\,.\label{eq:V_t_derivative_bound2_fa}
\end{align}

Following the same steps in \eqref{eq:V_t_derivative_bound1}, we deduce
\begin{align}
    \left|\gamma\cdot e_s^\top \mM_t\frac{d\mP_t}{dt}\mM_t r_t\right| 
    \leq \frac{2\gamma C_\phi}{(1-\gamma)^2}
    \norm{\xi'-\xi}_2\,.\label{eq:V_t_derivative_bound1_fa}
\end{align}

Combining the above two expressions \eqref{eq:V_t_derivative_bound2_fa} and \eqref{eq:V_t_derivative_bound1_fa} with \eqref{eq:V_t_derivative_matrix_fa}, we deduce
\begin{equation}\label{eq:V_diff_fa}
    |V^\xi(s)-V^{\xi'}(s)|\leq \frac{2C_\phi (1+\gamma)}{(1-\gamma)^2}\norm{\xi'-\xi}_2\,,
\end{equation}
which implies
\begin{equation}
    \forall (s,a)\in\S\times\A:\quad |Q^\xi(s,a)-Q^{\xi'}(s,a)|\leq\frac{2C_\phi \gamma(1+\gamma)}{(1-\gamma)^2}\norm{\xi'-\xi}_2\,.
\end{equation}

\subsection{Proof of Lemma~\ref{lm:performance_improvement_fednac}}\label{sec_app:pf_lm_improve_fednac}
This proof is inspired by the proof of Theorem~1 in~\citet{yuan2022linear}. To give the proof, we first introduce the following three-point descent lemma.
\begin{lm}[Three-point descent lemma, Lemma 6 in \citet{xiao2022convergence}]\label{lm:3_point}
    Suppose that $\mathcal{C}\subset\R^m$ is a
closed convex set, $g:\mathcal{C}\rightarrow\R$ is a proper, closed, convex function, $D_h(\cdot,\cdot)$ is the Bregman divergence
generated by a function $h$ of Lengendre type and rint dom$h\cap\mathcal{C}\ne\emptyset$. For any $x\in \text{rint} \text{dom}h$, let
$$x^+\in\arg\min_{u\in\text{dom} h\cap \mathcal{C}}\{f(u)+D_h(u,x)\}\,,$$
then $x^+\in\text{dom} h\cap \mathcal{C}$ and for any $u\in\text{dom} h\cap \mathcal{C}$, it holds that
\begin{equation}\label{eq:3_point_ineq}
    f(x^+)+D_h(x^+,x)\leq f(u)+D_h(u,x)-D_h(u,x^+)\,.
\end{equation}
\end{lm}
 
By the update rule~\eqref{eq:update_theta_mean} and the parameterization~\eqref{eq:param_fa} we know know that
$$\forall (s,a)\in\S\times\A:\quad\bar f^{(t+1)}(a|s)=\frac{1}{Z^{(t)}(s)}f^{(t)}(a|s)\exp\left(\alpha\phi^\top(s,a)\hat w^{(t)}\right),$$
where $Z^{(t)}(s)$ is a normalization coefficient to ensure $\sum_{a\in\A}f^{(t+1)}(s,a)=1$ for each $s\in\S$. Note that the above $\pi^{(t+1)}$ could also be obtained by a mirror descent update:
\begin{equation}\label{eq:mirror_descent}
    \forall s\in\S:\quad f^{(t+1)}(\cdot|s)=\arg\min_{g\in\Delta(\A)}\left\{-\alpha\langle\Phi(s)\hat w^{(t)},g\rangle+\KL{g}{f^{(t)}(\cdot|s)} \right\}\,,
\end{equation}
where $\Phi(s)\in\R^{|\A|\times p}$ is a matrix with rows $\phi^\top(s,a)\in\R^p$ for $a\in\A$.

We apply the three-point descent lemma (cf.~Lemma~\ref{lm:3_point}) with $\mathcal{C}=\Delta(\A)$, $f=-\alpha\langle\Phi(s)\hat w^{(t)},\cdot\rangle$ and $h:\Delta(\A)\rightarrow \R$ is the negative entropy with $h(q)=\sum_{a\in\A}q(a)\log q(a)$ and deduce that for any $q\in\Delta(\A)$, we have
$$-\alpha\langle\Phi(s)\hat w^{(t)},\bar f^{(t+1)}(\cdot|s)\rangle+D\left(\bar f^{(t+1)}(\cdot|s), \bar f^{(t)}(\cdot|s)\right)\leq -\alpha\langle\Phi(s)\hat w^{(t)}, q\rangle+D\left(q, \bar f^{(t)}(\cdot|s)\right)-D\left(q, \bar f^{(t+1)}(\cdot|s)\right)\,.$$
Rearranging terms and dividing both sides by $-\alpha$, we obtain
\begin{equation}\label{eq:intermediate_3_points}
    \langle\Phi(s)\hat w^{(t)},\bar f^{(t+1)}(\cdot|s)-q\rangle-\frac{1}{\alpha}D\left(\bar f^{(t+1)}(\cdot|s), \bar f^{(t)}(\cdot|s)\right)\geq -\frac{1}{\alpha}D\left(q, \bar f^{(t)}(\cdot|s)\right)+\frac{1}{\alpha}D\left(q, \bar f^{(t+1)}(\cdot|s)\right)\,.
\end{equation}

Let $q=\bar f^{(t)}(\cdot|s)$ and $\pi^\star(\cdot|s)$,resp., we have the following two inequalities:
\begin{align}
     \langle\Phi(s)\hat w^{(t)},\bar f^{(t+1)}(\cdot|s)-\bar f^{(t)}(\cdot|s)\rangle\geq \frac{1}{\alpha}D\left(\bar f^{(t+1)}(\cdot|s), \bar f^{(t)}(\cdot|s)\right)+&\frac{1}{\alpha}D\left(\bar f^{(t)}(\cdot|s), \bar f^{(t+1)}(\cdot|s)\right)\geq 0\,.\label{eq:q=bar_f_t}\\
     \langle\Phi(s)\hat w^{(t)},\bar f^{(t+1)}(\cdot|s)-\bar f^{(t)}(\cdot|s)\rangle+\langle\Phi(s)\hat w^{(t)},\bar f^{(t)}(\cdot|s)-\pi^\star(\cdot|s)\rangle&\geq -\frac{1}{\alpha}D\left(\pi^\star(\cdot|s), \bar f^{(t)}(\cdot|s)\right)+\frac{1}{\alpha}D\left(\pi^\star(\cdot|s), \bar f^{(t+1)}(\cdot|s)\right)\,.\label{eq:q=pi_star}
\end{align}

Taking expectation w.r.t. distribution $d^\star$ on both sides of \eqref{eq:q=pi_star}, we arrive at
\begin{equation}\label{eq:q=pi_star_exp}
    \E_{s\sim d^\star}\left[\langle\Phi(s)\hat w^{(t)},\bar f^{(t+1)}(\cdot|s)-\bar f^{(t)}(\cdot|s)\rangle\right]+\E_{s\sim d^\star}\left[\langle\Phi(s)\hat w^{(t)},\bar f^{(t)}(\cdot|s)-\pi^\star(\cdot|s)\rangle\right]\geq \frac{1}{\alpha}(D_\star^{(t+1)}-D_\star^{(t)})\,.
\end{equation}

To simplify the notation we let $\bar Q^{(t)}$ and $\bar V^{(t)}$ denote $Q^{\bar f^{(t)}}$ and $V^{\bar f^{(t)}}$, respectively. Note that the first expectation in the above expression \eqref{eq:q=pi_star_exp} could be upper bounded as follows:
\begin{align}
    &\E_{s\sim d^\star}\left[\langle\Phi(s)\hat w^{(t)},\bar f^{(t+1)}(\cdot|s)-\bar f^{(t)}(\cdot|s)\rangle\right]\notag\\
    &=\sum_{s\in\S}d^\star(s)\langle\Phi(s)\hat w^{(t)}, \bar f^{(t+1)}(\cdot|s)-\bar f^{(t)}(\cdot|s)\rangle\notag\\
    &=\sum_{s\in\S}\frac{d^\star(s)}{d^{\bar f^{(k+1)}}(s)}d^{\bar f^{(k+1)}}(s)\langle\Phi(s)\hat w^{(t)}, \bar f^{(t+1)}(\cdot|s)-\bar f^{(t)}(\cdot|s)\rangle\notag\\
    &\leq\vartheta_{\rho}\sum_{s\in\S}d^{\bar f^{(k+1)}}(s)\langle\Phi(s)\hat w^{(t)}, \bar f^{(t+1)}(\cdot|s)-\bar f^{(t)}(\cdot|s)\rangle\notag\\
    &=\vartheta_{\rho}\sum_{s\in\S}d^{\bar f^{(k+1)}}(s)\langle\bar Q^{(t)}(s,\cdot), \bar f^{(t+1)}(\cdot|s)-\bar f^{(t)}(\cdot|s)\rangle+\vartheta_{\rho}\sum_{s\in\S}d^{\bar f^{(k+1)}}(s)\langle\bar \Phi(s)\hat w^{(t)}-\bar Q^{(t)}(s,\cdot), \bar f^{(t+1)}(\cdot|s)-\bar f^{(t)}(\cdot|s)\rangle\notag\\
    &=\vartheta_\rho(1-\gamma)\left(\bar V^{(t+1)}(\rho)-\bar V^{(t)}(\rho)\right)+\vartheta_{\rho}\sum_{s\in\S}d^{\bar f^{(k+1)}}(s)\langle\bar \Phi(s)\hat w^{(t)}-\bar Q^{(t)}(s,\cdot), \bar f^{(t+1)}(\cdot|s)-\bar f^{(t)}(\cdot|s)\rangle\,,\label{eq:ub_exp_1}
\end{align}
where the first inequality uses \eqref{eq:d_facts} and the definition of $\vartheta_\rho$~\eqref{eq:vartheta_rho} and the last line follows from \eqref{eq:V_diff_Q} in Lemma~\ref{lm:performance_difference}. We separate the second term of the last line into four terms as follows:
\begin{align}
    &\sum_{s\in\S}d^{\bar f^{(t+1)}}(s)\langle\bar \Phi(s)\hat w^{(t)}-\bar Q^{(t)}(s,\cdot), \bar f^{(t+1)}(\cdot|s)-\bar f^{(t)}(\cdot|s)\rangle\notag\\
    &=\underbrace{\sum_{s\in\S}\sum_{a\in\A}d^{\bar f^{(t+1)}}(s)\bar f^{(t+1)}(a|s)\phi^\top(s,a)(\hat w^{(t)}-\hat w_\star^{(t)})}_{(I)}+\underbrace{\sum_{s\in\S}\sum_{a\in\A}d^{\bar f^{(t+1)}}(s)\bar f^{(t+1)}(a|s)\left(\phi^\top(s,a)\hat w_\star^{(t)}-\bar Q^{(t)}(s,a)\right)}_{(II)}\notag\\
    &+\underbrace{\sum_{s\in\S}\sum_{a\in\A}d^{\bar f^{(t+1)}}(s)\bar f^{(t)}(a|s)\phi^\top(s,a)(\hat w_\star^{(t)}-\hat w^{(t)})}_{(III)}
    + \underbrace{\sum_{s\in\S}\sum_{a\in\A}d^{\bar f^{(t+1)}}(s)\bar f^{(t)}(a|s)\left(\bar Q^{(t)}(s,a)-\phi^\top(s,a)\hat w_\star^{(t)}\right)}_{(IV)}\,.\label{eq:I-IV}
\end{align}

Applying again Lemma~\ref{lm:performance_difference}, we deduce the equivalent form of the second expectation in~\eqref{eq:q=pi_star_exp} as follows:
\begin{align}
    &\E_{s\sim d^\star}\left[\langle\Phi(s)\hat w^{(t)},\bar f^{(t)}(\cdot|s)-\pi^\star(\cdot|s)\rangle\right]\notag\\
  &  =\E_{s\sim d^\star}\left[\langle\bar Q^{(t)}(s,\cdot),\bar f^{(t)}(\cdot|s)-\pi^\star(\cdot|s)\rangle\right]+\E_{s\sim d^\star}\left[\langle\Phi(s)\hat w^{(t)}-\bar Q^{(t)}(s,\cdot),\bar f^{(t)}(\cdot|s)-\pi^\star(\cdot|s)\rangle\right]\notag\\
&    =(1-\gamma)\left(\bar V^{(t)}(\rho)-V^{\pi^\star}(\rho)\right)+\E_{s\sim d^\star}\left[\langle\Phi(s)\hat w^{(t)}-\bar Q^{(t)}(s,\cdot),\bar f^{(t)}(\cdot|s)-\pi^\star(\cdot|s)\rangle\right]\,,\label{eq:ub_exp_2}
\end{align}
where the second term of the last line could be decomposed into the following terms:
\begin{align}
    &\E_{s\sim d^\star}\left[\langle\Phi(s)\hat w^{(t)}-\bar Q^{(t)}(s,\cdot),\bar f^{(t)}(\cdot|s)-\pi^\star(\cdot|s)\rangle\right]\notag\\
  &  =\underbrace{\sum_{s\in\S}\sum_{a\in\A}d^{\star}(s)\bar f^{(t)}(a|s)\phi^\top(s,a)(\hat w^{(t)}-\hat w^{(t)}_\star)}_{(A)} 
    + \underbrace{\sum_{s\in\S}\sum_{a\in\A}d^{\star}(s)\bar f^{(t)}(a|s)\left(\phi^\top(s,a)\hat w_\star^{(t)}-\bar Q^{(t)}(s,a)\right)}_{(B)}\notag\\
    &+ \underbrace{\sum_{s\in\S}\sum_{a\in\A}d^{\star}(s)\pi^\star(a|s)\phi^\top(s,a)(\hat w_\star^{(t)}-\hat w^{(t)})}_{(C)}
    + \underbrace{\sum_{s\in\S}\sum_{a\in\A}d^{\star}(s)\pi^\star(a|s)\left(\bar Q^{(t)}(s,a)-\phi^\top(s,a)\hat w_\star^{(t)}\right)}_{(D)}\,.\label{eq:A-D}
\end{align}

Plugging \eqref{eq:I-IV}, \eqref{eq:A-D} into \eqref{eq:ub_exp_1} and \eqref{eq:ub_exp_2}, resp., and making use of \eqref{eq:q=pi_star_exp}, we have
\begin{align}\label{eq:8_terms}
    &\vartheta_\rho (1-\gamma) \left(\bar V^{(t+1)}(\rho)-\bar V^{(t)}(\rho)\right)+(1-\gamma) \left(\bar V^{(t)}(\rho)-V^{\pi^\star}(\rho)\right) \nonumber \\
    &+ \vartheta_\rho \bigg(\underbrace{\sum_{s\in\S}\sum_{a\in\A}d^{\bar f^{(t+1)}}(s)\bar f^{(t+1)}(a|s)\phi^\top(s,a)(\hat w^{(t)}-\hat w_\star^{(t)})}_{(\mathrm{I})}
    +\underbrace{\sum_{s\in\S}\sum_{a\in\A}d^{\bar f^{(t+1)}}(s)\bar f^{(t+1)}(a|s)\left(\phi^\top(s,a)\hat w_\star^{(t)}-\bar Q^{(t)}(s,a)\right)}_{(\mathrm{II})} \nonumber\\
    &+ \underbrace{\sum_{s\in\S}\sum_{a\in\A}d^{\bar f^{(t+1)}}(s)\bar f^{(t)}(a|s)\phi^\top(s,a)(\hat w_\star^{(t)}-\hat w^{(t)})}_{(\mathrm{III})}
    + \underbrace{\sum_{s\in\S}\sum_{a\in\A}d^{\bar f^{(t+1)}}(s)\bar f^{(t)}(a|s)\left(\bar Q^{(t)}(s,a)-\phi^\top(s,a)\hat w_\star^{(t)}\right)}_{(\mathrm{IV})}\bigg) \nonumber\\
    &+
    \underbrace{\sum_{s\in\S}\sum_{a\in\A}d^{\star}(s)\bar f^{(t)}(a|s)\phi^\top(s,a)(\hat w^{(t)}-\hat w^{(t)}_\star)}_{(A)} 
    + \underbrace{\sum_{s\in\S}\sum_{a\in\A}d^{\star}(s)\bar f^{(t)}(a|s)\left(\phi^\top(s,a)\hat w_\star^{(t)}-\bar Q^{(t)}(s,a)\right)}_{(B)} \nonumber\\
    &+ \underbrace{\sum_{s\in\S}\sum_{a\in\A}d^{\star}(s)\pi^\star(a|s)\phi^\top(s,a)(\hat w_\star^{(t)}-\hat w^{(t)})}_{(C)}
    + \underbrace{\sum_{s\in\S}\sum_{a\in\A}d^{\star}(s)\pi^\star(a|s)\left(\bar Q^{(t)}(s,a)-\phi^\top(s,a)\hat w_\star^{(t)}\right)}_{(D)} \nonumber \\
    &\geq \frac{1}{\alpha}(D_\star^{(t+1)}-D_\star^{(t)}).
\end{align}

Below we upper bound the terms $|(\mathrm{I})|$-$|(\mathrm{IV})|$ and $|(A)|$-$|(D)|$. 

For any $t\in\NN$ and $n\in[N]$, we define matrix $\Sigma_{\tilde d_n^{(t)}}\in\R^{p\times p}$ as 
\begin{equation}\label{eq:Sigma_d}
    \Sigma_{\tilde d_n^{(t)}}\coloneqq\E_{(s,a)\sim \tilde d_n^{(t)}}\left[\phi(s,a)\phi^\top(s,a)\right]\,,
\end{equation}
and  
\begin{align}
    \varepsilon_{\text{stat},n}^{(t)}&\coloneqq \ell\left(w_n^{(t)},Q_n^{(t)},\tilde d_n^{(t)}\right)-\ell\left(w^{(t)}_{\star,n},Q_n^{(t)},\tilde d_n^{(t)}\right)\,,\label{eq:eps_stat_t_n}\\
    \varepsilon_{\text{approx},n}^{(t)}&\coloneqq\ell\left(w_{\star,n}^{(t)}, Q_n^{(t)},\tilde{d}_n^{(t)}\right)\,,\label{eq:eps_approx_t_n}
\end{align}
then for all $n\in[N]$, by Assumption~\ref{asmp:eps_stat} and Assumption~\ref{asmp:eps_approx} we have
\begin{equation}\label{eq:exp_eps}
    \E\left[\varepsilon_{\text{stat},n}^{(t)}\right]\leq \staterror^n\,,\quad \text{and} \quad \E\left[\varepsilon_{\text{approx},n}^{(t)}\right]\leq \approxerror^n\,.
\end{equation}
We let $\bar \varepsilon_{\text{stat}}^{(t)}\coloneqq\frac{1}{N}\sum_{n=1}^N \varepsilon_{\text{stat},n}^{(t)}$ and $\bar \varepsilon_{\text{approx}}^{(t)}\coloneqq\frac{1}{N}\sum_{n=1}^N \varepsilon_{\text{approx},n}^{(t)}$. By Cauchy-Schwartz's inequality we have
\begin{align}
    |(\mathrm{I})|&\leq\frac{1}{N}\sum_{n=1}^N\sum_{(s,a)\in\S\times\A}d^{\bar f^{(t+1)}}(s)\bar f^{(t+1)}(a|s)|\phi^\top(s,a)(w_n^{(t)}- w_{\star,n}^{(t)})|\notag\\
    &\leq\frac{1}{N}\sum_{n=1}^N\sqrt{\sum_{(s,a)\in\S\times\A}\frac{\left(d^{\bar f^{(t+1)}}(s)\right)^2\left(\bar f^{(t+1)}(a|s)\right)^2}{\tilde d^{(t)}_n(s,a)}\cdot\sum_{(s,a)\in\S\times\A}\tilde d^{(t)}_n(s,a)\left(\phi^\top(s,a)(w^{(t)}_n-w_{\star,n}^{(t)})\right)^2}\notag\\
    &=\frac{1}{N}\sum_{n=1}^N\sqrt{\E_{(s,a)\sim\tilde d_n^{(t)}}\left[\left(\frac{\left(d^{\bar f^{(t+1)}}(s)\right)\left(\bar f^{(t+1)}(a|s)\right)}{\tilde d_n^{(t)}(s,a)}\right)^2\right]\norm{w_n^{(t)}-w^{(t)}_{\star,n}}_{\Sigma_{\tilde d_n^{(t)}}}^2}\notag\\
    &\leq\frac{1}{N}\sum_{n=1}^N\sqrt{C_\nu\norm{w_n^{(t)}-w^{(t)}_{\star,n}}_{\Sigma_{\tilde d_n^{(t)}}}^2}\notag\\
    &\leq\frac{1}{N}\sum_{n=1}^N\sqrt{C_\nu\varepsilon_{\text{stat},n}^{(t)}}\leq \sqrt{C_\nu\meanstaterror^{(t)}}\,,\label{eq:ub_I}
\end{align}
where the third inequality follows from Assumption~\ref{asmp:bound_transfer}, the last inequality uses Jensen's inequality, and the penultimate inequality by Assumption~\ref{asmp:eps_stat} and by noticing that for all $w\in\R^p$, we have
\begin{align}
    &\ell\left(w,Q_n^{(t)},\tilde d_n^{(t)}\right)-\ell\left(w^{(t)}_{\star,n},Q_n^{(t)},\tilde d_n^{(t)}\right)\notag\\
    &=\E_{(s,a)\sim\tilde d_n^{(t)}}\left[\left(\phi^\top(s,a)w-\phi^\top(s,a)w^{(t)}_{\star,n}+\phi^\top(s,a)w^{(t)}_{\star,n}-Q^{(t)}_n(s,a)\right)^2\right]-\ell\left(w^{(t)}_{\star,n},Q_n^{(t)},\tilde d_n^{(t)}\right)\notag\\
    &=\E_{(s,a)\sim\tilde d_n^{(t)}}\left[\left(\phi^\top(s,a)w-\phi^\top(s,a)w^{(t)}_{\star,n}\right)^2\right]+2\left(w-w^{(t)}_{\star,n}\right)^\top\E_{(s,a)\sim\tilde d_n^{(t)}}\left[\left(\phi^\top(s,a)w^{(t)}_{\star,n}-Q^{(t)}_n(s,a)\right)\phi(s,a)\right]\notag\\
    &=\norm{w-w^{(t)}_{\star,n}}_{\Sigma_{\tilde d_n^{(t)}}}+\left(w-w^{(t)}_{\star,n}\right)^\top\nabla_w\ell\left(w^{(t)}_{\star,n},Q_n^{(t)},\tilde d_n^{(t)}\right)\notag\\
    &\geq\norm{w-w^{(t)}_{\star,n}}_{\Sigma_{\tilde d_n^{(t)}}}\,,
\end{align}
where the last line follows from the first-order optimality condition   $w^{(t)}_{\star,n}\in\arg\min_w \ell\left(w,Q_n^{(t)},\tilde d_n^{(t)}\right)$:
$$\forall w\in\R^p:\quad \left(w-w^{(t)}_{\star,n}\right)^\top\nabla_w\ell\left(w^{(t)}_{\star,n},Q_n^{(t)},\tilde d_n^{(t)}\right)\geq 0.$$

Analogous to bounding $|(\mathrm{I})|$, by simply substituting $\bar f^{(t+1)}$ with $\bar f^{(t)}$ or $\pi^\star$ or  substituting $d^{\bar f^{(t+1)}}$ into $d^\star$, we obtain the same upper bound for $|(\mathrm{III})|$, $|(A)|$ and $|(C)|$, i.e.,
\begin{equation}\label{eq:ub_III/A/C}
    |(\mathrm{III})|, |(A)|, |(C)|\leq \sqrt{C_\nu\meanstaterror^{(t)}}\,. 
\end{equation}

Now we upper bound $|(\mathrm{II})|$ as follows:
\begin{align}
    |(\mathrm{II})|&\leq\frac{1}{N}\sum_{n=1}^N\sum_{(s,a)\in\S\times\A}d^{\bar f^{(t+1)}}(s)\bar f^{(t+1)}(a|s)\left(|\phi^\top(s,a)w_{\star,n}^{(t)}-Q_n^{(t)}(s,a)|+|Q_n^{(t)}(s,a)-\bar Q^{(t)}(s,a)|\right)\notag\\
    &\leq\frac{1}{N}\sum_{n=1}^N\sqrt{\sum_{(s,a)\in\S\times\A}\frac{\left(d^{\bar f^{(t+1)}}(s)\right)^2\left(\bar f^{(t+1)}(a|s)\right)^2}{\tilde d^{(t)}_n(s,a)}}\cdot\notag\\
    &\qquad \cdot \sqrt{2\sum_{(s,a)\in\S\times\A}\tilde d^{(t)}_n(s,a)\left(\left(\phi^\top(s,a)w_{\star,n}^{(t)}-Q_n^{(t)}(s,a)\right)^2+\left(Q_n^{(t)}(s,a)-\bar Q^{(t)}(s,a)\right)^2\right)}\notag\\
    &=\frac{1}{N}\sum_{n=1}^N\sqrt{\E_{(s,a)\sim\tilde d_n^{(t)}}\left[\left(\frac{\left(d^{\bar f^{(t+1)}}(s)\right)\left(\bar f^{(t+1)}(a|s)\right)}{\tilde d_n^{(t)}(s,a)}\right)^2\right]\cdot 2\left(\varepsilon_{\text{approx},n}^{(t)}+L_Q^2\norm{\xi_n^{(t)}-\bar\xi^{(t)}}_2^2\right)}\notag\\
    &\leq\sqrt{2C_\nu\left(_\mathrm{F}\bar \varepsilon_{\text{approx}}^{(t)}+\frac{L_Q^2}{N}\norm{\vxi^{(t)}-\vone\bar\xi^{(t)}}_\mathrm{F}^2\right)}\,,\label{eq:ub_II}
\end{align}
where $L_Q$ is defined in Lemma~\ref{lm:L-Q}, the second line uses Cauchy-Schwartz's inequality and Young's inequality~\eqref{eq:young_mul} and the last inequality uses Assumption~\ref{asmp:bound_transfer} and Jensen's inequality.

Analogous to bounding $|(\mathrm{II})|$, by simply substituting $\bar f^{(t+1)}$ with $\bar f^{(t)}$ or $\pi^\star$ or  substituting $d^{\bar f^{(t+1)}}$ into $d^\star$, we obtain the same upper bound for $|(\mathrm{IV})|$, $|(B)|$ and $|(D)|$, i.e.,
\begin{equation}\label{eq:ub_IV/B/D}
    |(\mathrm{IV})|, |(B)|, |(D)|\leq \sqrt{2C_\nu\left(\bar \varepsilon_{\text{approx}}^{(t)}+\frac{L_Q^2}{N}\norm{\vxi^{(t)}-\vone\bar\xi^{(t)}}_\mathrm{F}^2\right)}\,. 
\end{equation}

Plugging \eqref{eq:ub_I},\eqref{eq:ub_III/A/C},\eqref{eq:ub_II},\eqref{eq:ub_IV/B/D} into \eqref{eq:8_terms} and dividing both sides by $(1-\gamma)$ yield
\begin{equation*}
    \vartheta_\rho\left(\delta^{(t+1)}-\delta^{(t)}\right)+\delta^{(t)}\leq \frac{D^{(t)}_\star}{(1-\gamma)\alpha}-\frac{D^{(t+1)}_\star}{(1-\gamma)\alpha}+\frac{2\sqrt{C_\nu}(\vartheta+1)}{1-\gamma}\left(\sqrt{\bar \varepsilon_{\text{stat}}^{(t)}}+\sqrt{2\left(\bar \varepsilon_{\text{approx}}^{(t)}+\frac{L_Q^2}{N}\norm{\vxi^{(t)}-\vone\bar\xi^{(t)}}_\mathrm{F}^2\right)}\right)\,.
\end{equation*}
Taking expectation on both sides of the above expression and making use of the simple fact that
$$\E\left[\sqrt{x}\right]\leq \sqrt{\E[x]}\,,$$
we reach the conclusion \eqref{eq:improve_fnac}.

\subsection{Proof of Lemma~\ref{lm:linear_sys_nac}}\label{sec_app:proof_lm_linear_sys}
 
For any $\zeta>0$, by the actor update rule~\eqref{eq:actor_update} and \eqref{eq:update_theta_mean} we have that
  \begin{align}
      \norm{\vxi^{(t+1)}-\vone_N\bar\xi^{(t+1)\top}}_\mathrm{F}^ 2&=\norm{\mW(\vxi^{(t)}+\alpha\vh^{(t)})-\vone_N(\bar\xi^{(t)}+\alpha\hat w^{(t)})^\top}_\mathrm{F}^2\notag\\
      &\leq(1+\zeta)\sigma^2\norm{\vxi^{(t)}-\vone_N\bar\xi^{(t)\top}}_\mathrm{F}^2+\alpha^2(1+1/\zeta)\sigma^2\norm{\vh^{(t)}-\vone_N\hat w^{(t)\top}}_\mathrm{F}^2,\label{eq:Omega_1_pre}
  \end{align}
  where the last line follows from Young's inequality \eqref{eq:young_2} and \eqref{eq:property_W}.
  By the gradient tracking step \eqref{eq:GT} , Young's inequality~\eqref{eq:young_2} and \eqref{eq:property_W}, we have
  \begin{align}
      \norm{\vh^{(t+1)}-\vone\hat w^{(t+1)\top}}_\mathrm{F}^2&=\norm{\mW(\vh^{(t)}+\vw^{(t+1)}-\vw^{(t)})-\vone\hat w^{(t)\top} +\vone (\hat w^{(t)\top}-\hat w^{(t+1)\top})}_\mathrm{F}^2\notag\\
      &=\norm{\mW\vh^{(t)}-\vone\hat w^{(t)\top}+\mW(\vw^{(t+1)}-\vw^{(t)})-\vone(\hat w^{(t+1)\top}-\hat w^{(t)\top})}_\mathrm{F}^2\notag\\
      &\leq(1+\zeta)\sigma^2\norm{\vh^{(t)}-\vone_N\hat w^{(t)\top}}+ (1+1/\zeta)\sigma^2\norm{\vw^{(t+1)}-\vw^{(t)}-\vone(\hat w^{(t+1)\top}-\hat w^{(t)\top})}_\mathrm{F}^2\notag\\
      &\leq (1+\zeta)\sigma^2\norm{\vh^{(t)}-\vone_N\hat w^{(t)\top}}+ (1+1/\zeta)\sigma^2\norm{\vw^{(t+1)}-\vw^{(t)}}_\mathrm{F}^2,\label{eq:Omega2_1}
  \end{align}
  where the last inequality follows from the fact
  \begin{align}
    &  \norm{\vw^{(t+1)}-\vw^{(t)}-\vone(\hat w^{(t+1)\top}-\hat w^{(t)\top})}_\mathrm{F}^2\nonumber \\
    &=\norm{\vw^{(t+1)}-\vw^{(t)}}_\mathrm{F}^2+N\norm{\hat w^{(t+1)}-\hat w^{(t)}}_2^2-2\sum_{n=1}^N\langle w^{(t+1)}_n-w^{(t)}_n, \hat w^{(t+1)}-\hat w^{(t)}\rangle\notag\\
      &=\norm{\vw^{(t+1)}-\vw^{(t)}}_\mathrm{F}^2-N\norm{\hat w^{(t+1)}-\hat w^{(t)}}_2^2\notag\\
      &\leq\norm{\vw^{(t+1)}-\vw^{(t)}}_\mathrm{F}^2.\label{eq:fact_w-w}
  \end{align}

  Then for any $n\in[N]$, $t\in\NN$ and $w\in\R^p$, we have
  \begin{align}
      &\ell(w,Q^{(t)}_n,\tilde d^{(t)}_n)-\ell(w^{(t)}_{\star,n},Q^{(t)}_n,\tilde d^{(t)}_n)\notag\\
      &=\E_{(s,a)\sim\tilde d^{(t)}_n}\left[\left(\phi^\top(s,a)w-\phi^\top(s,a)w^{(t)}_{\star,n}+\phi^\top(s,a)w^{(t)}_{\star,n}-Q^{(t)}_n(s,a)\right)^2\right]-\ell(w^{(t)}_{\star,n},Q_n^{(t)},\tilde d^{(t)}_n)\notag\\
      &=\E_{(s,a)\sim\tilde d^{(t)}_n}\left[\left(\phi^\top(s,a)w-\phi^\top(s,a)w^{(t)}_{\star,n}\right)^2\right]+2(w-w^{(t)}_{\star,n})^\top\E_{(s,a)\sim\tilde d^{(t)}_n}\left[\left(\phi^\top(s,a)w^{(t)}_{\star,n}-Q^{(t)}_n(s,a)\right)\phi(s,a)\right]\notag\\
      &=\norm{w-w^{(t)}_{\star,n}}_{\Sigma_{\tilde d^{(t)}_n}}^2+(w-w^{(t)}_{\star,n})^\top\nabla_w\ell(w^{(t)}_{\star,n},Q_n^{(t)},\tilde d^{(t)}_n)\notag\\
      &\geq \norm{w-w^{(t)}_{\star,n}}_{\Sigma_{\tilde d^{(t)}_n}}^2\notag\\
      &\geq (1-\gamma)\mu \norm{w-w^{(t)}_{\star,n}}_2^2,\label{eq:bound_w-w_star}
  \end{align}
  where the penultimate line follows from the first-order optimality conditions for the optima $w^{(t)}_{\star,n}$:
  \begin{equation}
      \forall w\in\R^p:\quad(w-w^{(t)}_{\star,n})^\top\nabla_w\ell(w^{(t)}_{\star,n},Q_n^{(t)},\tilde d^{(t)}_n)\geq 0
  \end{equation}
  and the last line is by Assumption~\ref{asmp:psd} and \eqref{eq:d_facts}.

  Note that
  \begin{align}
      &\ell(w^{(t)}_{\star,n},Q^{(t+1)}_n,\tilde d^{(t+1)}_n)\notag\\
      &=\E_{(s,a)\sim \tilde d^{(t+1)}_n}\left[(\phi^\top(s,a)w^{(t)}_{\star,n}-Q^{(t+1)}_n(s,a))^2\right]\notag\\
      &\leq 2\sum_{(s,a)\in\S\times\A} \tilde d^{(t)}_n(s,a)\frac{\tilde d^{(t+1)}_n(s,a)}{\tilde d^{(t)}_n(s,a)}(\phi^\top(s,a)w^{(t)}_{\star,n}-Q^{(t)}_n(s,a))^2+2\E_{(s,a)\sim\tilde d^{(t+1)}_n}(Q^{(t+1)}_n(s,a)-Q^{(t)}_n(s,a))^2\notag\\
      &\leq 2C_\nu \E_{(s,a)\sim\tilde d^{(t)}_n}(\phi^\top(s,a)w^{(t)}_{\star,n}-Q^{(t)}_n(s,a))^2
      +2L_Q\norm{\xi^{(t+1)}_n-\xi^{(t)}_n}_2^2\notag\\
      &\leq 2C_\nu \approxerror^n+2L_Q^2\norm{\xi^{(t+1)}_n-\xi^{(t)}_n}_2^2,\label{eq:l_t_n_star}
  \end{align}
  where the second inequality uses Assumption~\ref{asmp:bound_transfer} and Lemma~\ref{lm:L-Q}, and the last line uses Assumption~\ref{asmp:eps_approx}.

  The above equation~\eqref{eq:l_t_n_star} together with \eqref{eq:bound_w-w_star} gives
  \begin{align}
      \norm{\vw^{(t+1)}_\star-\vw^{(t)}_\star}_\mathrm{F}^2=\sum_{n=1}^N\norm{w^{(t+1)}_{\star,n}-w^{(t)}_{\star,n}}_2^2\leq &\frac{1}{(1-\gamma)\mu}\sum_{n=1}^N\left(\ell(w^{(t)}_{\star,n},Q^{(t+1)}_n,\tilde d^{(t+1)}_n)-\ell(w^{(t+1)}_{\star,n},Q^{(t+1)}_n,\tilde d^{(t+1)}_n)\right)\notag\\
      &\leq \frac{2}{(1-\gamma)\mu}\left(C_\nu \sum_{n=1}^N\approxerror^n+L_Q^2\norm{\vxi^{(t+1)}-\vxi^{(t)}}_\mathrm{F}^2\right).\label{eq:diff_w_star}
  \end{align}
  where $\vw^{(t)}_\star\coloneqq(w^{(t)}_1,\cdots,w^{(t)}_N)^\top$, $\forall t$.

Also note that by Assumption~\ref{asmp:eps_stat} and \eqref{eq:bound_w-w_star} we have
\begin{equation}\label{eq:w-w_star}
    \forall t\in\NN:\quad \norm{\vw^{(t)}-\vw^{(t)}_\star}_\mathrm{F}^2\leq\frac{\sum_{n=1}^N\staterror^n}{(1-\gamma)\mu}.
\end{equation}
  Therefore, by \eqref{eq:l_t_n_star} and \eqref{eq:w-w_star} we have
  \begin{align}
      \norm{\vw^{(t+1)}-\vw^{(t)}}_\mathrm{F}^2&\leq 3\left(\norm{\vw^{(t+1)}-\vw^{(t+1)}_\star}_\mathrm{F}^2+\norm{\vw^{(t+1)}_\star-\vw^{(t)}_\star}_\mathrm{F}^2+\norm{\vw^{(t)}-\vw^{(t)}_\star}_\mathrm{F}^2\right)\notag\\
      &\leq \frac{6}{(1-\gamma)\mu}\left(N(C_\nu\bar\varepsilon_{\text{approx}}+\bar\varepsilon_{\text{stat}})+L_Q^2\norm{\vxi^{(t+1)}-\vxi^{(t)}}_\mathrm{F}^2\right).\label{eq:diff_w}
  \end{align}
  where the first inequality uses Young's inequality~\eqref{eq:young_mul}.

  Note that by the update rule \eqref{eq:actor_update}, the double stochasticity of the mixing matrix $\mW$ and the consensus property \eqref{eq:property_W} we have
  \begin{align}
      \norm{\vxi^{(t+1)}-\vxi^{(t)}}_\mathrm{F}^2&=\norm{\mW(\vxi^{(t)}+\alpha\vh^{(t)})-\vxi^{(t)}}_\mathrm{F}^2\notag\\
      &=\norm{(\mW-\mI)(\vxi^{(t)}-\vone_N\bar\xi^{(t)\top})+\alpha(\mW\vh^{(t)}-\vone_N\hat w^{(t)\top})+\alpha\vone(\hat w^{(t)}-\hat w^{(t)}_\star)^\top+\vone(\hat w^{(t)}_\star)^\top}_\mathrm{F}^2\notag\\
      &\leq 16\norm{\vxi^{(t)}-\vone_N\bar\xi^{(t)\top}}_\mathrm{F}^2+4\alpha^2\sigma^2\norm{\vh^{(t)}-\vone_N\hat w^{(t)\top}}_\mathrm{F}^2+4\alpha^2N\norm{\hat w^{(t)}-\hat w^{(t)}_\star}_2^2+4\alpha^2 N\norm{\hat w^{(t)}_\star}_2^2\notag\\
      &\leq 16\norm{\vxi^{(t)}-\vone_N\bar\xi^{(t)\top}}_\mathrm{F}^2+4\alpha^2\sigma^2\norm{\vh^{(t)}-\vone_N\hat w^{(t)\top}}_\mathrm{F}^2+4\alpha^2\sum_{n=1}^N\norm{w^{(t)}_n- w^{(t)}_{\star,n}}_2^2+4\alpha^2 \sum_{n=1}^N\norm{ w^{(t)}_{\star,n}}_2^2\notag\\
      \leq & 16\norm{\vxi^{(t)}-\vone_N\bar\xi^{(t)\top}}_\mathrm{F}^2+4\alpha^2\sigma^2\norm{\vh^{(t)}-\vone_N\hat w^{(t)\top}}_\mathrm{F}^2 + \frac{4\alpha^2 N\bar\varepsilon_{\text{stat}}}{(1-\gamma)\mu}+\frac{4\alpha^2 N C_\phi^2}{\mu^2(1-\gamma)^4},\label{eq:bound_diff_theta}
  \end{align}
  where the penultimate line uses Jensen's inequality and the last line follows from \eqref{eq:bound_w-w_star}, Assumption~\ref{asmp:eps_stat} and \eqref{eq:bound_l_minimizer}. 

  Combining \eqref{eq:bound_diff_theta} and \eqref{eq:diff_w} with \eqref{eq:Omega2_1}, we deduce
  \begin{align}
      \norm{\vh^{(t+1)}-\vone\hat w^{(t+1)\top}}_\mathrm{F}^2&\leq (1+1/\zeta)\frac{96\sigma^2 L_Q^2}{(1-\gamma)\mu}\norm{\vxi^{(t)}-\vone\bar \xi^{(t)\top}}_\mathrm{F}^2+\sigma^2\left(1+\zeta+(1+1/\zeta)\frac{24L_Q^2\alpha^2}{(1-\gamma)\mu}\right)\norm{\vh^{(t)}-\vone\hat w^{(t)\top}}_\mathrm{F}^2\notag\\
      &+(1+1/\zeta)\frac{6\sigma^2}{(1-\gamma)\mu}\left(N(\bar\varepsilon_{\text{stat}}+C_\nu \bar\varepsilon_{\text{approx}})+4L_Q^2\left(\frac{\alpha^2N\bar\varepsilon_{\text{stat}}}{(1-\gamma)\mu}+\frac{\alpha^2NC_\phi^2}{\mu^2(1-\gamma)^2}\right)\right).\label{eq:Omega_2_pre}
  \end{align}

  Finally, \eqref{eq:matrix_fnac} follows from taking expectations on both sides of \eqref{eq:Omega_1_pre} and \eqref{eq:Omega_2_pre}.

\end{document}